\documentclass{article}

\PassOptionsToPackage{numbers,compress,sort}{natbib}

 \usepackage[preprint]{neurips_2026}

\usepackage[utf8]{inputenc}
\usepackage[T1]{fontenc}

\usepackage{amsmath}
\usepackage{amssymb}
\usepackage{amsfonts}
\usepackage{mathtools}
\usepackage{bm}

\usepackage[dvipsnames]{xcolor}
\usepackage{graphicx}
\usepackage{pifont}
\usepackage{xspace}
\usepackage[normalem]{ulem}

\usepackage{hyperref}
\usepackage{url}
\usepackage{booktabs}
\usepackage{nicefrac}
\usepackage{microtype}

\usepackage{algorithm}
\usepackage{algorithmic}
\usepackage{enumitem}
\usepackage{comment}
\usepackage{commath}

\usepackage{booktabs}
\usepackage[table]{xcolor}

\newcommand{\yinan}[1]{{\color{blue}[Yinan: #1]}}

\newcommand{\kj}[1]{{\color{RedOrange}[#1]}}

\newcommand{\blue}[1]{{\color[rgb]{.3,.5,1}#1}}

\newcommand{\grn}[1]{{\color{ForestGreen}#1}}

\newcommand{\guide}[1]{{\color{Violet}[#1]}}

\definecolor{kjgray}{rgb}{.7,.7,.7}

\def\ddefloop#1{\ifx\ddefloop#1\else\ddef{#1}\expandafter\ddefloop\fi}
\def\ddef#1{\expandafter\def\csname c#1\endcsname{\ensuremath{\mathcal{#1}}}}
\ddefloop ABCDEFGHIJKLMNOPQRSTUVWXYZ\ddefloop
\def\ddef#1{\expandafter\def\csname b#1\endcsname{\ensuremath{{\boldsymbol{#1}}}}}
\ddefloop ABCDEFGHIJKLMNOPQRSUVWXYZabcdeghijklmnopqrtsuvwxyz\ddefloop
\def\ddef#1{\expandafter\def\csname h#1\endcsname{\ensuremath{\hat{#1}}}}
\ddefloop ABCDEFGHIJKLMNOPQRSTUVWXYZabcdefghijklmnopqrsuvwxyz\ddefloop % except for 't' because it is used by others
\def\ddef#1{\expandafter\def\csname hc#1\endcsname{\ensuremath{\widehat{\mathcal{#1}}}}}
\ddefloop ABCDEFGHIJKLMNOPQRSTUVWXYZ\ddefloop
\def\ddef#1{\expandafter\def\csname t#1\endcsname{\ensuremath{\tilde{#1}}}}
\ddefloop ABCDEFGHIJKLMNOPQRSTUVWXYZabcdefghijklmnopqrtsuvwxyz\ddefloop
\def\ddef#1{\expandafter\def\csname r#1\endcsname{\ensuremath{\mathring{#1}}}}
\ddefloop ABCDEFGHIJKLMNOPQRSTUVWXYZabcdefghijklmnopqrtsuvwxyz\ddefloop
\def\ddef#1{\expandafter\def\csname bar#1\endcsname{\ensuremath{\bar{#1}}}}
\ddefloop ABCDEFGHIJKLMNOPQRSTUVWXYZabcdefghijklmnopqrtsuvwxyz\ddefloop
\def\ddef#1{\expandafter\def\csname wbar#1\endcsname{\ensuremath{\overline{#1}}}}
\ddefloop ABCDEFGHIJKLMNOPQRSTUVWXYZabcdefghijklmnopqrtsuvwxyz\ddefloop
\def\ddef#1{\expandafter\def\csname tc#1\endcsname{\ensuremath{\widetilde{\mathcal{#1}}}}}
\ddefloop ABCDEFGHIJKLMNOPQRSTUVWXYZ\ddefloop

\newcommand{\Acal}{\mathcal{A}}

\def\hmu{{\ensuremath{\hat{\mu}}} }

\DeclareMathOperator{\EE}{\mathbb{E}}

\DeclareMathOperator{\PP}{\mathbb{P}}

\def\RR{{\mathbb{R}}}

\newcommand{\fr}[2]{ { \frac{#1}{#2} }}

\def\eps{\ensuremath{\varepsilon}\xspace}

\def\sm{{\ensuremath{\setminus}\xspace} }

\def\tTh{\ensuremath{\tilde{\Theta}}}

\newcommand{\what}[1]{ {\ensuremath{\widehat{#1}}} }

\def\hDelta{\ensuremath{\what{\Delta}}\xspace}

\makeatletter
\newcommand{\vast}{\bBigg@{3}}
\newcommand{\Vast}{\bBigg@{4}}
\makeatother

\def\cS{{\ensuremath{\mathcal{S}}}}

\def\Dt{\Delta}

\usepackage{amsthm}

\newtheorem{theorem}{Theorem}
\newtheorem{lemma}[theorem]{Lemma}
\newtheorem{proposition}[theorem]{Proposition}
\newtheorem{corollary}[theorem]{Corollary}

\theoremstyle{definition}
\newtheorem{definition}[theorem]{Definition}

\newtheorem{remark}[theorem]{Remark}

\usepackage{caption}
\usepackage{subcaption}
\usepackage[toc,page,header]{appendix}
\usepackage{minitoc}
\doparttoc
\usepackage{silence}
\makeatletter
\def\mtcgapbeforeheads{10\p@}
\def\mtcgapafterheads{-15\p@}  
\makeatother

\newcommand{\edit}[2]{{\xspace\textcolor{blue}{\sout{#1}}}{ \textcolor{red}{#2}}}
{\color{kjsavedrevised}}%
\newcommand{\revisedi}[1]{{\color{MidnightBlue}#1}}

\def\Hlb{H_{\mathrm{lb}}}
\def\Alt{\mathsf{Alt}}

\usepackage{enumitem}
\setlist{nolistsep} % removes the space before and after itemize, and between items(??)
\setlist{itemsep=.1em} % space between items

\setlist[itemize]{topsep=.5pt,itemsep=0pt,parsep=2pt}
\setlist[enumerate]{topsep=.5pt,itemsep=0pt,parsep=2pt}

\newif\ifFINAL

\FINALtrue % or
\ifFINAL
  \def\blue#1{#1}
  \renewcommand{\guide}[1]{}
  \renewcommand{\kj}[1]{}
  \renewcommand{\yinan}[1]{}
  \renewcommand{\edit}[2]{{}{#2}}
  \renewcommand{\grn}[1]{{}}
  \renewcommand{\revisedi}[1]{{#1}}
\else
  \usepackage{transparent}
  \usepackage[inline]{showlabels}
  \usepackage{rotating}
  \renewcommand{\showlabelfont}%
  {\transparent{0.8}\scriptsize\bfseries\slshape\color{Lavender}}
\fi

\title{$\eps$-Good Action Identification in Fixed-Budget Monte Carlo Tree Search}

\author{%
Yinan Li \\
Department of Computer Science\\
  University of Arizona\\
  \texttt{yinanli@arizona.edu} \\
  \And
  Tuan Nguyen \\
Department of Computer Science\\
  University of Arizona\\
  \texttt{tnguyen9210@arizona.edu} \\
  \AND
  Kwang-Sung Jun \\
  CSE/GSAI \\
  POSTECH  \\
  \texttt{kwangsungjun@postech.ac.kr} \\
}

\begin{document}

\maketitle

\faketableofcontents

\begin{abstract}
We study the fixed-budget max–min action
identification problem in depth-2 max–min trees, which is an important special case of Monte Carlo Tree Search (MCTS).
In this problem, a learner sequentially and adaptively selects $T$ reward samples from leaves (depth-2 nodes) and then must recommend a subtree (depth-1 node) with the largest min-value among its children nodes.
Motivated by approximate planning, we focus on \emph{$\varepsilon$-good subtree identification}, where any subtree whose min value is within $\varepsilon$ of optimal maximin value is acceptable.
Our main contribution is an \emph{$\varepsilon$-agnostic} algorithm—requiring no knowledge of $\varepsilon$—that nevertheless achieves error bounds with explicit instance-dependent dependence on $\varepsilon$.
We show that, for every meaningful $\varepsilon$, the misidentification probability decays as
$\exp\!\big(-\tTh(\frac{T}{H_2(\varepsilon)} )\big)$, where $H_2(\varepsilon)$ captures both cross-subtree and within-subtree gaps.
In the special case where each subtree has a single leaf, the model reduces to standard multi-armed bandit identification, and our bounds recover (up to accelerating factors) the best-known $\varepsilon$-good guarantees associated with halving-style methods, while providing a new $\eps$-good analysis and guarantee for the Successive Rejects algorithm in fixed budget best arm identification.
On the lower-bound side, we present complementary positive and negative results. 
% that highlight fundamental difficulties in max--min trees and show that the intrinsic hardness is driven by at least a subset of critical leaves that determine subtree minima and certify the optimal subtree.
While there is a gap between the upper and lower bounds, we discuss the main technical challenges in obtaining a tighter lower bound, which tells us that the maximin action identification problem is quite different from the standard $K$-armed bandits. 
To our knowledge, this is the first provable algorithmic guarantee for fixed-budget maximin action identification.
\end{abstract}

\textfloatsep=.5em

\kj{
our key contribution is the UB. LB is a bit weak and needs defense on why it is not easy.

UB: we need to claim novelty on our algorithm and the UB. UB itself does not seem nontrivial/surprising because it matches the bound from the FC setting -- except that the $\eps$-good. 

\yinan{done} novelty/nontriviality on analysis? -- for this, $\eps$-good guarantee plays a role. because existing $\eps$-good guarantee is on SH, not SR (point out the difference that SH discards samples, and SR does not).

\yinan{done} novelty/nontriviality on algorithm? -- try to have an argument like 'if you don't have this particular component of our algorithm, then, we wouldn't obtain our guarantee'

try to compare it with multiple bandits (bubeck'13). explain why this algorithm would not solve our problem (or any other trivial extensions of it). put this into the section where you describe your algorithm. \yinan{I mentioned it in the introduction. }

Complexity of the tree structure, our reject algorithm is based on our novel definition of gaps (??)

}
\grn{Todo:
\begin{itemize}
    \item Done. Checklist
    \item Done. Shrink the main body to 9 pages
    \item Done. Show the sum of per-phase budget is bounded by total budget
    \item Done. Revise the statement of Prop 8 and its proof
    \item Done. Move the negative lower bound to appendix?
    \item Done. Discuss Komiyama et al. 2022 lightly, mention the strong oracle requirement
    \item Done. Explain the difference from those in FC literature, and/or acknowledge the similarities
    \item Done. Revise the introduction to motivate from generic MCTS to depth-2 trees
    \item Done. (From R1) some posterior work to (Garivier et al.) confirmed that the FC lower bound can be attained. the paper
https://jmlr.org/papers/volume22/18-798/18-798.pdf
provides a generic analysis of Track and Stop and discuss this particular application, see Section 5.2.2 page 24
    \item (Tuan) Add experiments in the appendix - I have done the comparison with four baselines (see the plot in the figures folder), I am thinking that you could do something related to showing the effect of $\eps$, as well as some allocation heatmap, hopefully more samples are allocated to those ``critical leaves''
    \item (Kwang) Related works, even those related-but-not-really-related, it improves the overall impression
\end{itemize}
}

\section{Introduction}
\label{sec:intro}

\kj{you are using the terminology of actions, arms, move, and subtrees. ask yourself: are these terminologies clear to the readers? 

\begin{itemize}
    \item action = subtree; action: more appropriate from the application perspective; subtree: more appropriate for the specific MCTS formulation.
    \item arm = leave; my recommendation is that everall avoid using the word 'arm'.
    \item move = subtree (suggestion: change move to action)
    \item subtree = subtree
\end{itemize}

ideally, we need to talk about generic MCTS (arbitrary depth) and then say we focus on depth 2 for simplicity. (probably you can find a similar exposition from garivier'16)
}

\revisedi{
Monte Carlo Tree Search (MCTS) is a widely used paradigm for planning in large
sequential decision problems, especially in game trees where exhaustive search is
computationally infeasible \citep{coulom2006efficient,kocsis2006bandit,
browne2012survey}. Starting from a root state, MCTS repeatedly performs simulated
rollouts, selectively expands or samples parts of the tree, and uses the observed
noisy payoffs to decide which root action to play. This viewpoint has been highly
influential in game playing and planning, including modern large-scale systems
combining tree search with learned value or policy functions
\citep{silver2016mastering}. From a statistical perspective, the key question is
a pure-exploration one: given a limited simulation budget, how should samples be
allocated across the tree so as to identify a good root action?

In a two-player zero-sum game tree, the value of a root action is determined by
the opponent's subsequent responses. Thus the value propagated to the root is
naturally described by alternating max and min operations. In an arbitrary-depth
tree, this max--min value is obtained recursively from the leaves to the root.
The depth-2 case is the simplest nontrivial instance of this structure. Each root
action $i$ corresponds to a subtree, each leaf $(i,j)$ corresponds to a possible
opponent response\kj{[x] what is continuation outcome? not clear to me.}\yinan{removed}, and the value of action $i$ is
$
    v_i = \min_{j \in [L]} \mu_{i,j}
$. 
The goal is then to identify an action with largest max--min value,
$\max_i v_i$. Although simple, this model already captures the central difficulty
absent from standard best-arm identification: the learner must allocate samples
both across root actions and within each action to determine its worst descendant.

This depth-2 abstraction is closely connected to prior theoretical work on MCTS
and maximin action identification. \citet{garivier2016maximin} introduced a
strategic bandit model motivated by MCTS, in which the learner samples noisy
outcomes of action pairs and aims to identify a maximin action in the
fixed-confidence setting. \citet{kaufmann17monte} studied best-action
identification in game trees of arbitrary depth, using confidence intervals to
summarize deeper levels of the tree and applying best-arm identification at the
root. \citet{huang2017structured} developed a more general structured
fixed-confidence framework motivated by minimax game search. 
More recently, the generic GLR/Tracking analysis of~\citet[][Section 5.2.2]{kaufmann2021mixture} implies asymptotic fixed-confidence optimality for depth-2 maxmin game trees under the usual regularity assumptions on the oracle weights. 
These works provide
fixed-confidence guarantees, where the algorithm stops once it can return a
correct action with probability at least $1-\delta$.
We discuss fixed-confidence literature in detail in Appendix~\ref{sec:related}.
% Table~\ref{tab:comparison-fc} summarizes the comparison. 

In this paper, we study the complementary fixed-budget problem. The learner is
given a prescribed simulation budget $T$ and must output a recommendation after
exactly $T$ samples. This criterion is natural in MCTS applications where a
decision must be made after a fixed planning window. Unlike regret minimization,
where performance is evaluated during the sampling process, fixed-budget pure
exploration separates simulation from execution: the learner is judged only by
the quality of the final recommendation
\citep{even-dar06action,bubeck09pure, audibert10best}. Our objective is to
control the probability of recommending a non-optimal, or more generally a
non-$\varepsilon$-good, root action after the fixed budget has been exhausted.
}

\guide{explain why mcts is challenging}
Max-min identification differs fundamentally from standard best-arm identification (BAI).
In BAI, the goal is to identify the arm with the largest value (often its mean), whereas in max-min identification, the learner must compare \emph{groups} of arms through a nested max-min functional: each top-level action is evaluated by the worst outcome among its descendants. 
% (or, more generally, its worst-case summary).
This induces unique statistical challenges.
First, information must be allocated not only across actions but also \emph{within} each action to locate its minimizing leaf.
Second, naive bottom-up approaches that fully solve all within-action subproblems (e.g., using Successive Accepts and Rejects for multi-bandit best arm identification in~\citet{bubeck13multiple}) are wasteful (also see empirical evidence in Appendix~\ref{sec:experiments}), because many within-action comparisons are irrelevant to deciding which action maximizes the min-value; the difficulty is instance-dependent and hinges on a structured collection of cross-action and within-action gaps.

\guide{motivate the $\varepsilon$-good version}
Moreover, in practical planning, identifying the \emph{exact} optimal subtree is often unnecessary: when multiple subtrees are nearly optimal, it suffices to output any subtree whose value is within a target accuracy $\varepsilon$ of the optimum.
This motivates the \emph{$\varepsilon$-good subtree identification} objective studied in this work.
Concretely, rather than requiring an exactly optimal subtree, we ask the algorithm to return an \emph{$\varepsilon$-good} subtree, meaning a subtree whose min value (worst-case outcome) is at least $v^\star-\varepsilon$, where $v^\star$ is the optimal max-min value.
This relaxation is algorithmically and statistically meaningful: it enlarges the set of acceptable outputs, often yields substantially smaller sample requirements, and better matches MCTS practice where planning is approximate.
At the same time, the correct dependence on $\varepsilon$ is subtle in max-min trees because the learner must resolve both \emph{cross-subtree} comparisons (which subtree is better) and \emph{within-subtree} structure (which leaf determines the min-value), and the relevant gaps depend on $\varepsilon$.

\guide{...}
A key design goal of our algorithm is \emph{$\varepsilon$-agnosticity}.
In many deployments, $\varepsilon$ is not known in advance, may be chosen downstream, or may vary across instances and operating conditions.
We therefore develop an algorithm that \emph{does not take $\varepsilon$ as input} yet still enjoys guarantees that scale with the (unknown) target \edit{accuracy}{error}:
for every (meaningful) $\varepsilon$, the same algorithm achieves a fixed-budget error guarantee (or equivalently a sample-complexity guarantee) governed by an instance-dependent complexity term $H_2(\varepsilon)$ that captures the difficulty of identifying an $\varepsilon$-good subtree.
In other words, our method automatically adapts to the effective resolution required by the instance, without prior knowledge of $\varepsilon$.

\begin{comment}
\revisedi{
\textbf{Relation to fixed-confidence maximin identification.}
The closest line of work studies maximin or game-tree action identification in the fixed-confidence setting. \cite{garivier2016maximin} introduced the depth-2 maximin action-identification model and proposed confidence-bound and racing/elimination algorithms with asymptotic sample-complexity guarantees. \cite{kaufmann17monte} studied best-action identification in game trees of arbitrary depth using BAI subroutines and confidence intervals propagated through the tree, and~\cite{huang2017structured} developed a more general structured fixed-confidence framework motivated by minimax search. More recently, the generic GLR/Tracking analysis of~\citep[][Section 5.2.2]{kaufmann2021mixture} implies asymptotic fixed-confidence optimality for depth-2 maxmin game trees under the usual regularity assumptions on the oracle weights. Table~\ref{tab:comparison-fc} summarizes the comparison. 
}
\end{comment}

% \yinan{Our contributions..also mention the lower bound}
To summarize this paper's contribution, 
we provide (i) an $\varepsilon$-agnostic sampling and recommendation procedure, (ii) an upper bound on the probability of recommending a non-$\varepsilon$-good subtree under a fixed sampling budget, and (iii) an instance-dependent complexity characterization $H_2(\eps)$ that explicitly reflects both cross-subtree and within-subtree gaps.
These guarantees extend classical fixed-budget best-arm identification analyses~\citep[e.g.][]{audibert10best} to max-min trees, and complement fixed-confidence maximin action identification work~\citep[e.g.][]{garivier2016maximin, kaufmann17monte, huang2017structured} with the fixed-budget counterpart, and address approximate identification with automatic adaptation to $\varepsilon$.
To our knowledge, this is the first provable algorithmic guarantee for fixed-budget MCTS. 
% , complementing a line of fixed-confidence results in structured/tree bandits and MCTS~\citep{garivier2016maximin, kaufmann17monte, huang2017structured}. 
\revisedi{Our guarantees are not directly comparable to fixed-confidence guarantees, since the objectives differ. Nevertheless, the gap quantities appearing in our upper bound have a similar structural form as the inverse-gap terms in prior fixed-confidence analyses.}
% In the depth-$2$ max-min action identification setting, our sample complexity is never worse than those in fixed-confidence settings. 
See Appendix~\ref{sec:related} for detailed discussions of the sample complexity. 

% On the lower bound side, we provide one positive and one negative result, both of which showcase the unique difficulty of characterizing the optimal sample complexity in max-min trees. Specifically, we show that a certain type of lower bound aligning with the idea of the lower bound in~\citet{audibert10best} for unstructured best-arm identification cannot be shown. At the same time, we show that the difficulty is driven by at least a subset of ``critical'' leaves (those determining subtree minima and
% certifying the optimal subtree), thereby isolating the structural source of hardness in max-min identification. 
On the lower-bound side, we prove an information-theoretic lower bound for exact
maximin subtree identification. The bound is governed by a critical-leaf
complexity $H_{\mathrm{lb}}(\nu)$, which involves the leaves that determine
competitor subtree minima and the leaves needed to certify the optimal subtree, thereby isolating the structural source of hardness in max-min identification. 
% This shows that the statistical difficulty of max--min identification is driven
% by the internal tree structure, and not only by root-level gaps. 
We also defer to
Appendix~\ref{app:negative-permutation-lb} a complementary negative result
showing that a direct permutation-style lower-bound route, aligning with the idea of the lower bound in~\citet{audibert10best} for unstructured BAI, cannot yield an $H_2$-type lower
bound for max--min trees even in the smallest nontrivial $(K=2,L=2)$ case. 

\revisedi{Our upper and lower bounds should be interpreted as order-wise sample-complexity results: they identify part of the relevant instance dependence up to logarithmic
and universal constant factors. This is different from the sharper
fixed-budget optimality notions studied by
\citet{komiyama2022minimax,degenne2023existence}, which concern precise
asymptotic error exponents and their attainability by adaptive algorithms. These
results show that sharp fixed-budget optimality is subtle even for unstructured
BAI. To our knowledge, they do not directly extend to the structured max--min
tree setting in a way that yields a matching characterization for our problem.}
We discuss technical challenges in obtaining a tighter upper or lower bound in Appendix~\ref{sec:appendix-perspectives}. 
\vspace{-1em}
\section{Preliminaries}
\label{sec:preliminaries}
\vspace{-.5em}
\paragraph{Model and indexing.}
There are $K$ subtrees (root actions) indexed by $i\in[K]:=\{1,\dots,K\}$.
Each subtree contains $L$ leaves indexed by $j\in[L]:=\{1,\dots,L\}$, and we denote a leaf by $(i,j)\in[K]\times[L]$.
(Allowing subtree-dependent leaf counts $L_i$ requires only notational changes; our upper and lower bounds extend verbatim.)

Each leaf $(i,j)$ is associated with an unknown reward distribution $\nu_{i,j}$ on $\RR$ with mean
$\mu_{i,j}:=\EE_{X\sim\nu_{i,j}}[X]$.
We assume $\nu_{i,j}$ is \emph{sub-Gaussian with variance proxy $1$}: for all $\lambda\in\RR$,
\begin{equation}
\label{eq:sg1}
\EE_{X\sim\nu_{i,j}}\!\big[\exp(\lambda(X-\mu_{i,j}))\big]\le \exp\!\left(\frac{\lambda^2}{2}\right).
\end{equation}
\vspace{-.1em}
\paragraph{Max-min values and WLOG ordering.}
Define the min-value of subtree $i$ by $v_i := \min_{j\in[L]}\mu_{i,j}$ and the optimal max--min value by
$v^\star := \max_{i\in[K]} v_i$.
Without loss of generality, re-index leaves within each subtree and re-index subtrees so that
\begin{equation}
\label{eq:wlog-order}
\mu_{1,1}\ge \mu_{2,1}\ge \cdots \ge \mu_{K,1},
\qquad
\mu_{i,1}\le \mu_{i,2}\le \cdots \le \mu_{i,L}\ \ \forall i\in[K].
\end{equation}
Then $v_i=\mu_{i,1}$ for all $i$, and subtree $1$ is optimal with $v^\star=v_1=\mu_{1,1}$.
Throughout, we assume the optimal subtree is unique for simplicity. 
\vspace{-.5em}

\paragraph{Fixed-budget interaction.}
As illustrated in Algorithm~\ref{alg:fixed-budget-interaction}, 
a (possibly randomized) learning algorithm interacts with the instance for a fixed budget of $T$ rounds.
At each round, it adaptively selects a leaf $(i,j)$ and observes one sample from $\nu_{i,j}$.
After $T$ samples, it outputs a recommended subtree index $\widehat i_T\in[K]$. 

\begin{algorithm}[t]
\caption{Fixed-budget interaction}
\label{alg:fixed-budget-interaction}
\begin{algorithmic}[1]
\REQUIRE Budget $T$
\FOR{$t=1,2,\ldots,T$}
    \STATE Select a leaf $(i_t,j_t)$ adaptively (possibly at random)
    \STATE Observe a sample $X_t \sim \nu_{i_t,j_t}$
\ENDFOR
\STATE \textbf{Output}: A recommended subtree index $\hat{i}_T \in [K]$
\end{algorithmic}
\end{algorithm}
\vspace{-.5em}
\paragraph{$\varepsilon$-good subtree identification.}
For $\varepsilon\ge 0$, a subtree $i$ is \emph{$\varepsilon$-good} if its min-value is within $\varepsilon$ of optimal:
\begin{align*}
    v_i \ge v^\star - \varepsilon
    \qquad\Longleftrightarrow\qquad
    \mu_{i,1} \ge \mu_{1,1} - \varepsilon .
\end{align*}
 
Let $\cG_\varepsilon := \{i\in[K]: v_i\ge v^\star-\varepsilon\}$ denote the set of $\varepsilon$-good subtrees, and
$
g(\varepsilon) := |\cG_\varepsilon|
$
be the number of $\varepsilon$-good subtrees, and
$\cB_{\eps} := [K]\setminus \cG_{\eps}$ the set of $\eps$-bad subtrees.
The learning objective is to output $\widehat i_T\in\cG_\varepsilon$; the error event is
$\{\widehat i_T\notin\cG_\varepsilon\}$. 
When $G_\varepsilon=[K]$, the problem is trivial; Throughout the nontrivial
analysis we consider $\varepsilon$ such that $G_\varepsilon\neq [K]$.
\vspace{-1em}
\paragraph{Gap quantities.}
We define leaf-wise gaps used in our instance-dependent bounds.
For the optimal subtree $i=1$, for each $j\in[L]$ define
\begin{equation}
\label{eq:gap-opt}
\Delta_{1,j} := \mu_{1,j}-\mu_{2,1}.
\end{equation}
For any $i\neq 1$ and $j\in[L]$, define
\begin{equation}
\label{eq:gap-subopt}
\Delta_{i,j} := \max\Big\{\mu_{1,1}-\mu_{i,1},\ \mu_{i,j}-\mu_{i,1}\Big\}.
\end{equation}
Under~\eqref{eq:wlog-order} and the unique optimal subtree assumption, all these gaps are positive. Intuitively,
$\mu_{1,1}-\mu_{i,1}$ captures the cross-subtree gap in its min value, while
$\mu_{i,j}-\mu_{i,1}$ captures the within-subtree difficulty of certifying the minimizing leaf.
These quantities will parameterize our upper and lower bounds for $\varepsilon$-good identification.

\paragraph{Sorted gaps, critical gap, and complexity.}
Let $\{\Delta_{i,j} : (i,j)\in[K]\times[L]\}$ be the multiset of the $KL$ gap values defined above.
Let $(\ell)$ denote the index of the $\ell$-th smallest element in this multiset,
\[
\Delta_{(1)} \le \Delta_{(2)} \le \cdots \le \Delta_{(KL)},
\]
breaking ties arbitrarily. 
With gaps and sorted gaps, we define the complexity measure as follows: 

\begin{definition}
\label{def:eps-complexity}
    Define the \emph{critical gap} (the gap level that separates $\varepsilon$-good from non-$\varepsilon$-good subtrees)
\begin{align}
    \blue{\Delta^\star} := \Delta_{g(\varepsilon)+1,\,1}.
\end{align}

\noindent
Let $\blue{m}$ be the largest integer such that $\Delta_{(m+1)}=\Delta^\star$.
Finally, define the (sorted-gap) complexities
\begin{align}
\label{eq:H-defs}
\blue{H_2(\varepsilon)} := \max_{r \ge m+1} r\,\Delta_{(r)}^{-2},
\qquad
H_1(\eps) := \sum_{i = 1}^{KL} \big( \Delta_{(i)} \vee \eps \big) ^{-2}.
\end{align}
% \begin{align}
% \label{eq:H-2-epsilon-def}
% \blue{H_2(\varepsilon)} := \max_{r \ge m+1} r\,\Delta_{(r)}^{-2}, 
% \end{align}
% and 
% \begin{align}
% \label{eq:H-1-epsilon-def}
%     H_1(\eps) := \sum_{i = 2}^{KL} ( \Delta_{(i)} \vee \eps) ^{-2}
% \end{align}

\end{definition}

% \paragraph{Informal preview of results.}
% Our Algorithm~\ref{alg:sr-for-mcts} is $\varepsilon$-agnostic (it does not take $\varepsilon$ as input) and, for every $\varepsilon$ (such that not all subtrees are $\eps$-good),
% achieves a fixed-budget misidentification probability 
% % on the order of
% $
% \PP(\widehat i_T\notin \cG_\varepsilon)
% \ \lesssim\
% \exp\!\left(
% -\frac{T}{\log(KL)\, \blue{H_2(\varepsilon)}}
% \right),
% $
% up to universal constants on the exponent.

% Considering the special case $L=1$, where each subtree has a single leaf, and the model degenerates to the standard $K$-armed best-arm identification problem.
% In this case, $\cG_\varepsilon=\{i\in[K]:\mu_i\ge \mu^\star-\varepsilon\}$ and 
% our Algorithm~\ref{alg:fixed-budget-interaction} gracefully reduces to the celebrated \emph{Successive Rejects}(SR) algorithm~\citep{audibert10best}. Hence, Theorem~\ref{thm:main-upper-bound} provides a failure probability upper bound for SR to output an $\eps$-good arm, with sample complexity
% $H_2(\varepsilon)$ specializing to the usual sorted-gap complexity built from the arm gaps. 

\vspace{-1em}

\section{Successive Rejects for MCTS}
\label{sec:upper-bound}
\vspace{-1em}

\subsection{From Successive Rejects to subtree identification.}
\label{sec:sr-to-subtree}
Our algorithm design starts from the celebrated \emph{Successive Rejects} (SR) algorithm for fixed-budget best-arm identification.
SR splits the budget into phases; in each phase it samples the remaining arms uniformly and then removes one arm based on empirical evidence.
A precise description appears in~\cite{audibert10best} (Figure~3 therein): at the end of phase $k$, SR removes the arm with the lowest empirical mean among the active (surviving) ones. 
% Their analysis yields an error probability of order $\exp\!\big(-\frac{n}{\log(K)\,H_2 }\big)$.

In our max-min subtree setting, the relevant sub-optimality of a leaf $(i,j)$ is not captured by a single mean comparison: a leaf can matter because it separates min values of subtrees (via $\mu_{1,1}-\mu_{i,1}$) or because it separates the minimum within a subtree (via $\mu_{i,j}-\mu_{i,1}$).
Accordingly, we run an SR-style phase-based procedure, but with (i) empirical gaps corresponding to our sub-optimality gaps $\Delta_{i,j}$ from Section~\ref{sec:preliminaries}, and (ii) a \emph{subtree elimination} rule that removes whole subtrees when sub-optimality evidence is strong.
\vspace{-1em}

\paragraph{Empirical minimizers, best subtree, and empirical gaps.}
Fix a phase $k$ and let $\hat{\mu}_{i,j}$ denote the empirical mean of leaf $(i,j)$ computed from all samples collected so far.
For each active subtree $i$, define its empirical minimizer index
\[
\hat{1}(i) \in \arg\min_{j\in \cL_i(\cA)} \ \hat{\mu}_{i,j},
\]
and the corresponding empirical min-value $\hat{v}_i(t) := \hat{\mu}_{i,\hat{1}(i)}$.
Let the empirical best subtree be
\[
\hat{a} \in \arg\max_{i \in \cS(\cA)} \ \hat{v}_i,
\qquad
\hat{\mu}^{\star} := \hat{v}_{\hat{a}} = \hat{\mu}_{\hat{a},\hat{1}(\hat{a})}.
\]
We now define empirical gaps $\hat{\Delta}_{i,j}$ by literally replacing population means in the true gaps
(Section~\ref{sec:preliminaries}) with their empirical counterparts.
For an active leaf $(i,j)$:
\[
\text{if } i=\hat{a},\qquad
\hat{\Delta}_{i,j}
~:=~
\hat{\mu}_{i,j}\;-\;\max_{s\neq i}\hat{\mu}_{s,\hat{1}(s)},
\]
\vspace{-1em}
\[
\text{if } i\neq \hat{a},\qquad
\hat{\Delta}_{i,j}
~:=~
\max\!\Big(
\hat{\mu}^{\star}-\hat{\mu}_{i,\hat{1}(i)},\;
\hat{\mu}_{i,j}-\hat{\mu}_{i,\hat{1}(i)}
\Big).
\]
These are empirical counterparts of $\Delta_{i,j}$ that either certify the optimal subtree, or simultaneously encode (a) separation from the current best subtree and
(b) within-subtree separation from the current empirical minimum.
\vspace{-1em}

\begin{algorithm}[t]
\caption{Successive Rejects for $(K,L)$-MCTS}
\begin{algorithmic}[1]
\label{alg:sr-for-mcts}
\REQUIRE Total pull budget $T$
\STATE $\Acal \gets [K]\times[L]$ ~~~~// the set of active leaves
\STATE $\forall \cA' \subseteq [K] \times [L], \cL_i(\cA') := \{j\in[L]: (i,j) \in \cA'\}$ ~~~~// extracts active leaves of subtree $i$
\STATE \revisedi{$\forall \cA' \subseteq [K] \times [L], \cS(\cA') := \{i\in [K]: \cL_i(\cA') \neq \emptyset\}$} ~~~~// extracts active subtrees
% \STATE $\forall \cA' \subseteq [K] \times [L], \cL_i(\cA') := \{j\in[L]: (i,j) \in \cA'\}$ ~~~~// extracts active leaves of subtree $i$
%\STATE \kj{define functions $\cS(\cA) := \{i\in [K]: \exists j\in[L]: (i,j) \in \cA\}$ and $\cL_i(\cA) := \{j\in[L]: (i,j) \in \cA\}$.}
%\STATE $\forall i \in [K]$, $\cA_i \gets \{i\}\times[L]$ ~~~~// the set of active leaves for each subtree. % \kj{$\cA_i \gets [L]$}
% \STATE \kj{$\cA(i) := \{(i',j') \in \cA: i' = i\}$ or, use the symbol of $\cI_i(\cA)$}
\STATE $k \gets 1$ \revisedi{~~~~ // phase index}
\STATE Note: Throughout, ties for any choices are broken with an arbitrary yet fixed rule.
% \STATE $n_0 \gets 0$
\WHILE{$k \le KL - 1$}
    \STATE $n_k \gets \left\lceil
    \dfrac{(T-KL)/\overline{\log}(KL)}{KL+1-k}
    \right\rceil$,
    where $\overline{\log}(KL):=\frac12+\sum_{r=2}^{KL}\frac1r$
    \label{line:sample-schedule}
  % \STATE $n_k \gets \left\lfloor \dfrac{T / \log (KL)}{KL+1-k} \right\rfloor$ \yinan{Is this correct? Audibert10 has $\left\lceil \dfrac{(T-K) / \bar\log (K)}{K+1-k} \right\rceil$} 
  % \kj{you know better than me; you analyzed this algorithm..}
  \STATE $\forall (i,j) \in \cA$, pull arm $(i,j)$ until the pull count equals $n_k$ and compute (over a total $n_k$ pulls) the empirical mean $\hat \mu_{i,j}$
  \STATE $\forall i \in \cS(\cA),~~ \hat{1}(i) \gets \arg\min_{j\in \cL_i(\cA)}\ \hat{\mu}_{i,j} $  ~~~~// subtree $i$'s empirical worst leaf (i.e., with the smallest value)
  \STATE $\ha \gets \arg\max_{i \in \cS(\cA)} \hat \mu_{i,\hat1(i)}$ ~~~~// the empirical maximin subtree
  \STATE $\hmu^* \gets \revisedi{\hmu_{\ha, \hat 1(\ha)}}$ ~~~~// the empirical maximin value
  \label{line:hat-a}
  \STATE  
      $\forall (i,j) \in \cA,
      \hat\Delta_{i,j} \gets
      \begin{cases}
         \hat{\mu}_{i,j} - \max_{i' \neq \hat a}\hmu_{i', \hat1(i')}
         %\revisedi{\hmu_{\hat a, \hat1(i)}}
         & 
         \text{if } i = \ha
      \\ \max\left(\hmu^* - \hat{\mu}_{i,\hat{1}(i)},\, \hat{\mu}_{i,j} - \hat{\mu}_{i,\hat{1}(i)}\right) 
         & 
         \text{otherwise}
      \end{cases}
      $
      ~~// empirical gaps
  \STATE $\hat \Delta_{\max} \gets \max_{(i,j) \in \Acal} \hat{\Delta}_{i,j}$
  \label{line:hat-Delta-max}
  \IF{ $\exists$ a subtree $x \in \cS(\cA) \sm \{\ha\}$: $\forall j \in \cL_x(\cA),~ \hat{\Delta}_{x,j} = \hat \Delta_{\max}$ }
  \label{line:subtree-elimination-condition}
    \STATE Let $x$ be the subtree satisfying the condition in the if statement.
    \STATE \revisedi{$\cE \gets \{(x,j): j \in \cL_x(\cA)\}$}
    % \STATE $\cE \gets \{(i,j) \in \cA: i = x\}$ 
  \ELSE
    \STATE Choose $(x,y)$ among active leaves $\cA$ s.t.\ $\hat{\Delta}_{x,y} = \hat \Delta_{\max}$  
%    \kj{[ ] wait... we need to make sure $x \neq \ha$, right?}
%    \kj{[ ] surviving vs active? choose one terminology and stick with it.}
    \STATE $\cE \gets \{(x,y)\}$ \kj{[ ] feel free to change the notation}
  \ENDIF
  \STATE $\cA \gets \cA ~\sm~ \cE$ 
  % ~~~~// the number of leaves being eliminated
  \IF{$|\cS(\cA)| = 1$}
    \STATE \revisedi{Let $\hi_T$ be the only element in $\cS(\cA)$} ~~// the only remaining subtree.
    \STATE \textbf{break}
  \ENDIF
  \STATE $k \gets k + |\cE|$
\ENDWHILE
\STATE \textbf{Output: } $\hi_T$  
\end{algorithmic}
\end{algorithm}

\subsection{Tree-aware fixed-budget elimination: subtree vs.\ leaf}
\vspace{-1em}

% \kj{note: it is important to emphasize our contribution of deciding when to eliminate the entire subtree (and why alternative strategies would fail)} \yinan{done}

\revisedi{Our sampling rule belongs to the broad family of elimination/racing algorithms used in fixed-budget BAI and fixed-confidence maximin identification. The novelty is not the use of elimination~\citep[e.g.][]{audibert10best} per se. Rather, the issue in a max–min tree is that deleting a single leaf can change the surviving empirical minimum of a subtree and thereby make a suboptimal subtree appear artificially strong. Our subtree-removal condition is designed to avoid this failure mode in a fixed-budget SR-style schedule.} The elimination step is where we depart from SR. 
Recall that $\cA$ is the current set of active leaves. 
% and, for each subtree $i$, let $\cA_i := \cA \cap (\{i\}\times[L])$ be its active leaves.
At the end of each phase, compute $\hat{\Delta}_{i,j}$ for all $(i,j)\in\cA$. Let $\hat \Delta_{\max} := \max_{(i,j) \in \Acal} \hat{\Delta}_{i,j}$ be the empirical largest gap value among active leaves. 

\noindent\textbf{Single leaf elimination.}
In most phases, the algorithm behaves like Successive Rejects: it computes empirical gaps $\hat\Delta_{i,j}$ for all active leaves and removes one leaf with the largest empirical gap.
Heuristically, a large gap indicates that the corresponding leaf is already well separated at the current resolution and is therefore unlikely to be the bottleneck for identifying an $\varepsilon$-good subtree; eliminating it concentrates future samples on the harder, smaller-gap comparisons.

\noindent\textbf{Subtree elimination.}
There is, however, an important exception in max--min trees.
In contrast to unstructured best-arm identification, eliminating a leaf while keeping its subtree active can change the subtree's surviving minimum.
If the eliminated leaf is the true minimizer of a suboptimal subtree, then the minimum among the remaining leaves increases, so the algorithm would subsequently compare other subtrees against a different objective value than the true max--min value of that subtree.
This creates a failure mode: a suboptimal subtree can appear artificially stronger after its minimizer is removed, which may corrupt later elimination decisions and ultimately the final recommendation.

Subtree-elimination ideas also appear in prior fixed-confidence MCTS and
maximin-identification work. We do not claim that eliminating an entire subtree is
new by itself. In \citet{garivier2016maximin} and \citet{teraoka2014efficient},
a subtree can be eliminated once it contains a leaf that is confidently smaller
than all leaves in another subtree. This condition is natural in the
fixed-confidence setting and is closely aligned with the intuition behind
$\alpha$--$\beta$ pruning \citep{knuth1975analysis}: once a subtree's worst-case value is certified to be
below that of a competitor, the subtree can be safely discarded.
% \kj{citation for $\alpha$--$\beta$ pruning?}

However, this type of elimination requires a high-confidence certificate of
suboptimality, which in turn may require identifying the subtree's true
minimizing leaf. In our fixed-budget setting, spending samples to certify the
internal minimum of every potentially suboptimal subtree can be wasteful, and is
not compatible with the goal of an $\varepsilon$-agnostic $\varepsilon$-good identification. Moreover, removing only a single leaf can be unsafe in a max--min
tree: if the removed leaf is the true minimizer of a suboptimal subtree, then the
minimum among the surviving leaves is artificially increased, so later phases
compare against a distorted value for that subtree.

Our rule addresses this issue differently. We eliminate a subtree only when all
of its active leaves attain the largest empirical gap $\widehat{\Delta}_{\max}$
(line~\ref{line:subtree-elimination-condition} of
Algorithm~\ref{alg:sr-for-mcts}). In that case, rather than removing one leaf and risking distortion of the subtree's surviving min-value, we remove the entire subtree.
Thus, our rule is not an $\alpha$--$\beta$-style certificate of
suboptimality; it is a tree-safe fixed-budget elimination mechanism designed to
preserve the max--min structure during a Successive-Rejects-style procedure. In fact, our particular subtree elimination rule is essential for obtaining our guarantee, and we were not able to obtain the same guarantee with other straightforward rules. 

In $K$-armed bandits, SR removes exactly one arm per phase, hence it runs a fixed number of phases. In contrast, our subtree elimination can remove \emph{many} leaves in a single phase.
Therefore, we implement the procedure with a \texttt{while} loop: the next phase is triggered only after updating the active set. 
The total number of phases is data-dependent, and it adapts to how many leaves were eliminated in the previous phase.
Operationally, this places the method between SR and more aggressive halving-style schemes: eliminations can be as mild as SR (one leaf at a time)
or much stronger when evidence supports discarding an entire subtree.

% The remainder of this section formalizes the algorithm and presents the instance-dependent performance guarantee. 
% \vspace{-1em}

\subsection{Instance-dependent performance guarantee}
\paragraph{An $\varepsilon$-agnostic guarantee for $\varepsilon$-good identification.}
A salient feature of Theorem~\ref{thm:main-upper-bound} is that the algorithm is \emph{$\varepsilon$-agnostic}: it does not take $\varepsilon$ as an input and its sampling/elimination decisions are identical for all target accuracies.
Nevertheless, for every meaningful $\varepsilon \geq 0$ (such that not all subtrees are $\eps$-good), the same output $\widehat i_T$ enjoys an \emph{$\varepsilon$-good} fixed-budget guarantee:
the probability of returning a subtree whose min value is more than $\varepsilon$ below that of the optimal subtree decays exponentially in the budget $T$, with an exponent governed by the $\varepsilon$-dependent complexity $H_2(\varepsilon)$.
In other words, without knowing the target resolution in advance, the algorithm automatically adapts to the effective difficulty of identifying an $\varepsilon$-good subtree.
% \vspace{-.5em}
\begin{theorem}[$\varepsilon$-agnostic $\varepsilon$-good error bound]
\label{thm:main-upper-bound}
Suppose $KL\ge 2$ and $T>KL$. 
Recall the definition of $H_2(\varepsilon)$ in
Definition~\ref{def:eps-complexity} and $\overline{\log}(KL)$ defined in
line~\ref{line:sample-schedule} of Algorithm~\ref{alg:sr-for-mcts}. Then the failure probability of Successive
Rejects for $(K,L)$-MCTS (Algorithm~\ref{alg:sr-for-mcts}) in returning an $\varepsilon$-good subtree satisfies
\[
    \PP\!\left(\widehat i_T\notin \cG_\varepsilon\right)
    \le
    2K^2L^2
    \exp\!\left(
        - \frac{T-KL}
        {128\,\overline{\log}(KL)\,H_2(\varepsilon)}
    \right).
\]
\end{theorem}
% \begin{theorem}
%     % \label{thm:main-upper-bound}
%     Recall the definition of $H_2(\eps)$ in Definition~\ref{def:eps-complexity}.
%     The failure probability of the algorithm Successive Rejects for $(K,L)$-MCTS in finding an $\eps$-good subtree satisfies: 
%     \[
%     \PP(\widehat i_T\notin \cG_\varepsilon) \leq
%     2K^2 L^2 \exp\left(- \frac{T}{128 \ln (KL) \cdot H_2(\eps)} \right). 
%     \]
% \end{theorem}
\vspace{-1em}
For comparison with prior fixed-confidence bounds, we relate the sorted-gap
complexity $H_2(\varepsilon)$ to the sum-type quantity $H_1(\varepsilon)$.
The following lemma shows that $H_2(\varepsilon)$ is always dominated by
$H_1(\varepsilon)$; detailed comparisons with fixed-confidence complexities are
given in Appendix~\ref{sec:related}.

\begin{lemma}
\label{lem:H2-H1}
Recall the definition of $H_2(\eps), H_1(\eps)$ in Definition~\ref{def:eps-complexity}. We have, 
\[
H_2(\eps) \leq H_1(\eps). 
\]
% and \[
% \sum_{i = m+1}^{KL} ( \Delta_{(i)} \vee \eps) ^{-2} \leq \ln (\fr {KL}m) \cd H_2(\eps)
% \]
\end{lemma}

\paragraph{Special case: $K$-armed bandits ($L=1$).}
When each subtree has a single leaf ($L=1$), the model reduces to the classical $K$-armed bandit identification problem:
there are arms indexed by $i\in[K]$, each arm $i$ has an unknown $1$-sub-Gaussian reward distribution with mean $\mu_i$,
and the learner allocates a fixed budget $T$ of samples before outputting a recommendation $\widehat i_T$.
Without loss of generality assume $\mu_1\ge \mu_2\ge \cdots \ge \mu_K$, so that arm $1$ is optimal and $\mu^\star=\mu_1$.
For $\varepsilon\ge 0$, an arm $i$ is called $\varepsilon$-good if $\mu_i\ge \mu^\star-\varepsilon$, and let
$\cG_\varepsilon := \{i\in[K]: \mu_i\ge \mu^\star-\varepsilon\}$.
% \vspace{-1em}
\begin{corollary}[$K$-armed bandits: fixed-budget $\varepsilon$-good identification]
\label{cor:bandit-upper}
Specialize Theorem~\ref{thm:main-upper-bound} to $L=1$.
Then the failure probability of Successive Rejects in outputting an $\varepsilon$-good arm satisfies
\[
\PP(\widehat i_T\notin \cG_\varepsilon)
\;\le\;
2K^2
    \exp\!\left(
        - \frac{T-K}
        {128\,\overline{\log}(K)\,H_2(\varepsilon)}
    \right), 
\]
where $H_2(\varepsilon)$ is the $\varepsilon$-dependent sorted-gap complexity from Definition~\ref{def:eps-complexity}
(with the gaps reducing to the standard arm-gap quantities when $L=1$).
\end{corollary}
% \vspace{-1em}

\begin{remark}[novelty already for $K$-armed bandits]
% \paragraph{Remark (novelty already for $K$-armed bandits).}
Even in the classical $K$-armed bandit setting ($L=1$), Corollary~\ref{cor:bandit-upper} appears to be new for the \emph{Successive Rejects} (SR) analysis:
it provides an $\varepsilon$-agnostic $\varepsilon$-good guarantee, with an explicit $\varepsilon$-dependent complexity $H_2(\varepsilon)$ governing the fixed-budget error exponent.
While SR is well studied for exact identification ($\varepsilon=0$), extending SR-style analyses to approximate identification is nontrivial because the set of acceptable outputs depends on $\varepsilon$ and the proof must control the elimination dynamics uniformly over all meaningful accuracy levels.
\end{remark}
\vspace{-1em}

Related $\varepsilon$-agnostic results are known for more aggressive elimination schemes such as \emph{Successive Halving} (SH), but the proof techniques for SH exploit halving-specific structure (e.g., synchronized elimination of a fixed fraction of arms each round) and do not directly transfer to SR, which eliminates arms sequentially and uses a different budget allocation across phases.
Our analysis develops an analysis for SR that yields $\varepsilon$-good identification without knowledge of $\varepsilon$.

Moreover, in the $K$-armed case, the resulting $\varepsilon$-dependent complexity $H_2(\varepsilon)$ matches the best-known $\varepsilon$-good identification guarantees for SH~\citep{zhao2023revisiting}, up to accelerating factors.
In this sense, our results are consistent with the sharp $\varepsilon$-good picture in unstructured bandits, while providing a new route that is generalized from SR and extends to the much more challenging max--min subtree structure.
% \vspace{-1em}

\subsection{Proof roadmap}
The proof of Theorem~\ref{thm:main-upper-bound} follows an inductive ``invariant + concentration'' strategy.
The central invariant is the notion of \emph{$\varepsilon$-soundness} (Definition~\ref{def:soundness}), which formalizes what it means for the active set at the beginning of a phase to remain faithful to the true max--min structure.
Recall that $\cA$ is the current active leaf set, and
for each subtree $i\in[K]$, write
$
\cA_i = \cA \cap (\{i\}\times[L])
$ as its active leaves. 
We call subtree $i$ \emph{nonempty} if $\cA_i\neq\emptyset$ and \emph{empty} otherwise.
We call leaf $(i,j)$ \emph{surviving} if $(i,j)\in\cA$.
% \vspace{-.5em}

\begin{definition}
\label{def:soundness}
    For any phase $k$, we say ``phase $k$ is $\eps$-sound'' if at the beginning of phase $k$, 
    \begin{enumerate}[label=(\arabic*), ref=\arabic*]

    \item[] If $k \le KL-m$:
    \item The optimal subtree is nonempty: $\cA_{1}\neq\emptyset$.
    \label{item:soundness-1}
    \item $\forall x\neq 1$ and $\eps/2$-good, $(x,1)$ is surviving:
    $\forall x\in \cG_{\varepsilon/2}\setminus\{1\},\ (x,1)\in\cA$.
    \label{item:soundness-2}
    \item $\forall x$ such that $x$ is $\eps/2$-bad, either $(x,1)$ is surviving, or $x$ is empty.
    \label{item:soundness-3}
    
    \item[] If $k > KL-m$:
    \item Either all $\eps/2$-good subtrees are nonempty, or all $\eps$-bad subtrees are empty.
    \label{item:soundness-4}
    \item $\forall x$ such that $x$ is $\eps$-bad, either $(x,1)$ is surviving, or $x$ is empty.
    \label{item:soundness-5}
    
    \end{enumerate}
    % \begin{itemize}
    %     \item If $k \leq KL-m$, 
    %         \begin{enumerate}
    %             \item The optimal subtree is nonempty: $\cA_{1}\neq \emptyset$. 
    %             \label{item:soundness-early-1}
    %             \item $\forall x \neq 1$ and $\eps/2$-good, $(x,1)$ is surviving: $\forall x\in \cG_{\varepsilon/2}\setminus\{1\}, \ (x,1)\in \cA$. 
    %             \label{item:soundness-early-2}
    %             \item $\forall x$ such that $x$ is $\eps/2$-bad, either $(x,1)$ is surviving, or $x$ is empty. 
    %             \label{item:soundness-early-3}
    %             % \yinan{Maybe come up with a name for ``either $(x,1)$ is surviving, or $x$ is empty''}
    %         \end{enumerate}
    %     \item If $k > KL-m$, 
    %         \begin{enumerate}
    %             \item Either all $\eps/2$-good subtrees are nonempty, or all $\eps$-bad subtrees are empty
    %             \item $\forall x$ such that $x$ is $\eps$-bad, either $(x,1)$ is surviving, or $x$ is empty
    %         \end{enumerate}
    % \end{itemize}
    For the ease of exposition, suppose the algorithm terminates after exactly $k'$ phases, we still consider the soundness of $k'+1$, which implies the termination condition. 

\end{definition}
\vspace{-1em}
% Assumptions:
% \begin{itemize}
%     \item WLOG, $\Dt^*$ is well-defined, i.e., not all subtrees are $\eps$-good. 
% \end{itemize}

% \textbf{Key idea:} to protect $\eps/2$-good subtrees against $\eps$-bad subtrees. 

Roughly, soundness requires two things: (i) the optimal subtree (and, more generally, near-optimal subtrees) are not prematurely eliminated, and (ii) any not-near-optimal subtree either still contains its true minimizing leaf or has already been completely removed.
This invariant is essential because eliminating the true minimizer of a sub-optimal subtree would ``inflate'' its surviving minimum and thereby distort subsequent max--min comparisons.

Throughout the analysis, we assume without loss of generality that the critical gap $\Delta^\star$ is well-defined, i.e., not all subtrees are $\varepsilon$-good.
Indeed, if every subtree is $\varepsilon$-good, then every output is acceptable and the misidentification probability is identically zero, so the theorem holds trivially.
Under this nontrivial regime, $\Delta^\star$ and the associated index $m$ (Definition~\ref{def:eps-complexity}) meaningfully separate ``acceptable'' subtrees from those that must be rejected.
\vspace{-1em}

\paragraph{Key idea: protect $\varepsilon/2$-good subtrees from $\varepsilon$-bad ones.}
At a high level, the algorithm is designed to ensure that any truly $\varepsilon$-bad subtree is eliminated \emph{before} any $\varepsilon/2$-good subtree can be discarded.
This is the role of Definition~\ref{def:soundness}: conditions~\ref{item:soundness-3} and \ref{item:soundness-5} rule out the possibility that a non-near-optimal subtree appears artificially strong.
In the early regime ($k\le KL-m$), conditions~\ref{item:soundness-1}--\ref{item:soundness-3} guarantee that the optimal subtree remains active and that every $\varepsilon/2$-good subtree stays safely anchored by its true minimizer $(x,1)$. See Lemma~\ref{lem:basic-invariant-early-phase} to Lemma~\ref{lem:very-good-subtree-early-phases}, in particular Lemma~\ref{lem:Delta-hat-max} to Lemma~\ref{lem:very-good-subtree-early-phases} in the Appendix for formal statements. Specifically, we include a warmup section in appendix~\ref{sec:appendix-phase-1-analysis} that analyzes phase 1, to assist understanding of the proof ideas. 
In the later regime ($k>KL-m$), conditions~\ref{item:soundness-4}--\ref{item:soundness-5} capture the endgame: either all $\varepsilon/2$-good subtrees persist, or all $\varepsilon$-bad subtrees have been eliminated, which in particular implies the termination condition. See Lemma~\ref{lem:bad-subtree-late-phases} to Lemma~\ref{lem:very-good-subtree-late-phases} in the Appendix for formal statements. 

The formal induction argument (Lemma~\ref{lem:sufficient-condition-eps-subtree} in the Appendix) proceeds by showing that on a suitable high-probability deviation event (uniform concentration of empirical means),
every phase is $\varepsilon$-sound.
The induction step analyzes the elimination rule at the end of a phase and proves that, assuming phase $k$ is $\varepsilon$-sound (at the beginning),
the algorithm will not break any condition in $\varepsilon$-soundness with the elimination happening in phase $k$.   
% eliminate an $\varepsilon/2$-good subtree while an $\varepsilon$-bad subtree remains active.
% Equivalently, under proper deviation control, $\varepsilon$-bad subtrees are rejected earlier because their critical leaves exhibit larger empirical gaps than those of $\varepsilon/2$-good subtrees.
This maintains soundness across phases and yields the desired fixed-budget error bound after union-bounding over phases.
\vspace{-1em}

\section{Lower bound for exact maximin subtree identification ($\eps=0$)}
\label{sec:lower-bound}
\vspace{-1em}
This section investigates lower bounds for \emph{exact} maximin subtree identification ($\eps=0$) in the
fixed-budget setting. 
% We present one negative result and one positive result,
% both indicating that characterizing the optimal sample complexity in maximin trees is qualitatively
% more subtle than in unstructured best-arm identification.
We introduce a complexity term tailored to the maximin structure, which will govern the instance-dependent (over instances with bounded complexity) fixed-budget lower bound.
\vspace{-1em}

\paragraph{Lower bound complexity.}
Let $\nu$ denote an instance, and recall that the (subtree-level) gaps $\Delta_{i,1}$ for $i\neq 1$ and the (within-optimal-subtree) gaps $\Delta_{1,j}$ for $j\neq 1$ are defined in Section~\ref{sec:preliminaries}.
We define the \emph{max-min lower bound complexity}
\begin{equation}\label{eq:H-lb}
H_{\mathrm{lb}}(\nu)
\;:=\;
\sum_{i \neq 1}\frac{1}{\Delta_{i,1}^2}
\;+\;
\sum_{j \neq 1}\frac{1}{\Delta_{1,j}^2}.
\end{equation}
The first term $\sum_{i\neq 1}\Delta_{i,1}^{-2}$ captures the cost (in terms of number of samples needed) of separating the competitor's minimum leaf
mean $\{\mu_{i,1}: i \neq 1\}$ (even if the minimizing leaf is known for each non-optimal subtree) from the optimal minimum leaf mean $\mu_{1,1}$. 
The second term $\sum_{j\neq 1}\Delta_{1,j}^{-2}$ reflects the cost of verifying the subtree 1 is indeed the optimal subtree: even if the min value of the strongest competitor subtree ($\mu_{2,1}$) is known, the learner must gather enough evidence to at least validate that all non-minimizer leaves inside the optimal subtree 1, $\{\mu_{1,j}: j \neq 1\}$, are above $\mu_{2,1}$.

Recall the definition of $H_1(\eps)$ in Definition~\ref{def:eps-complexity}. Note that with $\eps = 0$, $H_2(0)$ and $H_1(0)$ are well known to be equal up to logarithmic factors, i.e., $H_2(0) \le H_1(0) \le \log(2KL)\,H_2(0)$~\citep[e.g.,][]{audibert10best}. 
For any max-min tree instance $\nu$, to introduce the explicit dependence on $\nu$, we write 
\[
H(\nu)\;\equiv\;H_1(0)
\;=\;
\sum_{(i,j) \in [K] \times [L]}\frac{1}{\Delta_{i,j}^2}. 
\]

By construction, the summands defining $H_{\mathrm{lb}}(\nu)$ form a subset of those in $H(\nu)$,
hence $H_{\mathrm{lb}}(\nu)\le H(\nu)$ for every instance $\nu$.
For $\mathcal{H}>0$, define the class of instances with bounded complexity: 
\[
\mathcal{P}_{\mathcal{H}}
\;:=\;
\bigl\{\nu:\ H(\nu)\le \mathcal{H}\bigr\}.
\]

\begin{theorem}[Lower bound for maximin trees]\label{thm:main-lower-bound}
Let $\nu$ be a maximin-tree instance with a unique optimal subtree, and assume its reward distributions on all leaves are Gaussian with variance 1. Consider the instance class $ \mathcal{P}_{4 H(\nu)}$.
For any fixed-budget algorithm $\pi$ that collects at most $T$ samples in total, there exists an instance
$\nu' \in \mathcal{P}_{4 H(\nu)}$ such that
\[
\mathbb{P}_{\nu'}\!\left(\mathrm{error}\right)
\;\ge\;
c_1 \exp\!\left(-c_2 \frac{T}{\Hlb(\nu)}\right),
\]
where $\mathbb{P}_{\nu'}\!\left(\mathrm{error}\right)$ is the event that $\pi$ outputs an incorrect optimal subtree on instance $\nu'$, and $c_1,c_2>0$
are universal constants (independent of $K,L,T$ and $\nu$).
\end{theorem}
\vspace{-.5em}

% \noindent
% Theorem~\ref{thm:main-lower-bound} implies that, for any instance $\nu$,  uniformly over instances $\nu'$ with $H(\nu')\le O(H(\nu))$, \yinan{change this notation}
% no algorithm can guarantee error smaller than $\exp\!\bigl(-\Theta(T/\Hlb(\nu))\bigr)$.
% In particular, compared to the ``flat'' upper-bound complexity $\sum_{i,j} \Delta_{i,j}^{-2}$ (up to log factors) \yinan{clarify this},
% the intrinsic max-min hardness decomposes into two unavoidable components in~\eqref{eq:H-lb}:
% (i) distinguishing the optimal subtree from competitors via $\Delta_{i,1}$, and
% (ii) verifying the subtree 1's optimality against the strongest competitor (subtree 2) via $\Delta_{1,j}$.

Theorem~\ref{thm:main-lower-bound} is a local worst-case statement around a reference instance $\nu$:
even when restricting attention to instances whose leafwise complexity $H(\cdot)$ is within a constant factor of
$H(\nu)$, there exists a hard instance $\nu'$ on which the error cannot decay faster than
$\exp(-\Theta(T/H_{\mathrm{lb}}(\nu)))$. The appearance of $H_{\mathrm{lb}}(\nu)$ reflects that, in maximin trees,
the difficulty is driven by at least a subset of ``critical'' leaves (those determining subtree minima and
certifying the optimal subtree). 
% , which is why $H_{\mathrm{lb}}(\nu)\le H(\nu)$.
This ``non-uniform importance'' phenomenon is also visible in the fixed-confidence literature, e.g.,
in the optimal allocation behavior discussed by~\citet{garivier2016maximin} in their Section 6, certain arms receive asymptotically vanishing sampling effort depending on the instance, 
% This ``non-uniform importance'' of leaves is also witnessed in the fixed-confidence
% literature on tree-structured identification~\citep[e.g.][]{garivier2016maximin} : in the optimal allocation, certain arms may receive asymptotically vanishing sampling effort, 
indicating that they become effectively irrelevant as the confidence level increases. 
We discuss this phenomenon further in Appendix~\ref{sec:appendix-perspectives}. 
\vspace{-1em}

\paragraph{Connection to classical BAI lower bounds ($L=1$).}
When $L=1$, the maximin tree reduces to the standard $K$-armed best-arm identification (BAI) problem.
In this degenerate case, Theorem~\ref{thm:main-lower-bound} recovers the familiar fixed-budget $H_1$ lower bound established in~\citet{kaufmann2016complexity,carpentier2016tight}.
Indeed, $H(\nu)=H_1(0)$ becomes the usual leafwise complexity $\sum_{i\neq 1}\Delta_i^{-2}$, while
$H_{\mathrm{lb}}(\nu)$ collapses to the same quantity (since the within-optimal-subtree term vanishes),
and the statement reduces to an $\exp\!\bigl(-\Theta(T/H(\nu))\bigr)$ lower bound on the misidentification
probability for some instance of comparable complexity.

% \paragraph{Warm-up: a Gaussian flipping proof for unstructured BAI.}
Before proving Theorem~\ref{thm:main-lower-bound}, we provide a warm-up result for the $K$-armed case,
stated as Theorem~\ref{thm:lb-bai-gaussian}. There we present a clean proof under Gaussian arms with unit
variance, using a flipping-style construction as in~\citet{carpentier2016tight} (see also~\citet{kaufmann2016complexity}).
Our proof does not rely on the classical fixed-budget change-of-measure route based on the empirical KL divergence~\citep{audibert10best,carpentier2016tight}.
Instead, it follows a proof recipe common in regret lower bounds and also in~\citet{kaufmann2016complexity}: construct a collection of nearby
alternatives, decompose the KL divergence of the induced transcript distributions via the expected number
of pulls, and convert KL divergence into an unavoidable testing error via standard inequalities (e.g.,
Bretagnolle--Huber / Le Cam). This perspective aligns with the information-theoretic tools developed for
regret minimization and is in the same spirit as the techniques surveyed in standard bandit references~\citep[e.g.][]{lattimore18bandit}.
\vspace{-1em}

\paragraph{Proof strategy for maximin trees.}
The proof of Theorem~\ref{thm:main-lower-bound} extends the warm-up recipe to the tree-structured setting.
At a high level, we again build a finite family of alternative instances by perturbing as few leaves as possible at a time. We then argue, via a contradiction-style argument over
this family, that any fixed-budget algorithm must either allocate substantial samples to each leaf (or a group of leaves),
or incur a non-negligible probability of misidentification on at least one alternative.
The resulting lower bound is governed by $H_{\mathrm{lb}}(\nu)$, which isolates the structural source of
hardness in maximin identification: separating subtree minima across competitors and certifying the minimum
within the optimal subtree. The full proof is deferred to Appendix~\ref{sec:appendix-lower-bound-proof}.
% \vspace{-1em}

\textbf{A limitation of permutation-style lower bounds.}
A complementary negative result is given
in Appendix~\ref{app:negative-permutation-lb}, where we show that a permutation-style lower-bound route, in the spirit of~\citet{audibert10best} and based on the idea that the knowledge (just except for the indices) of the instance is available to the algorithm, cannot by itself establish an $H_2$-type
lower bound for max--min trees. Together with Theorem~\ref{thm:main-lower-bound}, this suggests
that fixed-budget max--min lower bounds are qualitatively subtler than their
unstructured BAI counterparts.

\vspace{-1em}

% \section{Conclusions and Perspectives}
\section{Conclusions}
\label{sec:perspectives}
\vspace{-1em}
% \yinan{From our chat:
% I think the gap between the upper bound and lower bound is rooted at a much deeper level of reason. I even think this mismatch itself (at least in FB) is worth a paper in the future. 
 
% When we prove the upper bound, we are playing the role of an algorithm - the nature selects an instance for us to solve, and the algorithm cannot know which leaf is (2,1) and which is (2,2). 
 
% On the other hand, when we prove the lower bound, we are playing the role of selecting an alternative instance. Given any base instance, the adversary proposes an algorithm, and the base instance and the algorithm jointly induce a probability measure of number of pulls on each leaf. Now our goal is to construct a ``closest confusing alternative'' (the one that flips which subtree is optimal with the least KL cost). But this closest confusing alternative does not need to modify leaf (2,2). And that's why the lower bound seems to omit leaf (2,2).  
% }

We studied the fixed-budget depth-$2$ max--min action identification motivated by Monte Carlo Tree Search (MCTS), focusing on $\varepsilon$-good subtree identification, where any subtree whose min value is within $\varepsilon$ of optimal is acceptable.
Our main contribution is on the algorithmic side; we present a Successive-Rejects-style \emph{$\varepsilon$-agnostic} method that requires no knowledge of $\varepsilon$ yet achieves instance-dependent guarantees with explicit $\varepsilon$-dependence governed by
$H_2(\varepsilon)$. 
% In the special case of one leaf per subtree, the model reduces to fixed-budget best-arm identification; our analysis recovers (up to accelerating factors) the best-known $\varepsilon$-good guarantees for halving-style methods and provides a new $\varepsilon$-good guarantee for Successive Reject. 
To the best of our knowledge, this work gives the first provable fixed-budget guarantees for an MCTS max–min action identification problem. 
On the lower-bound side, we established complementary results indicating difficulties unique to max--min trees. In particular, the hardness is driven by at least a subset of leaves that determine subtree minima and certify the optimal subtree.

% We studied fixed-budget $\varepsilon$-good action identification in depth-2
% max--min trees. We proposed an $\varepsilon$-agnostic Successive-Rejects-style
% algorithm and proved an instance-dependent error bound governed by
% $H_2(\varepsilon)$, which captures both cross-subtree and within-subtree gaps.
% Our lower-bound results show that the statistical hardness is driven by a subset
% of structurally critical leaves, while also indicating that sharp fixed-budget
% lower bounds for max--min trees are more subtle than in standard best-arm
% identification.
Important open problems include closing the upper--lower gap, extending the
analysis to deeper trees, and developing fixed-budget guarantees for richer
tree-search models, such as stochastic game trees with chance nodes
\citep{lanctot2013monte}, non-max--min backup rules
\citep{coulom2006efficient,lanctot2014monte}, and general structured
pure-exploration objectives \citep{huang2017structured,degenne2019pure}.

% Important open problems include closing the remaining upper--lower gap, extending
% the analysis to deeper game trees, and developing fixed-budget guarantees for
% richer tree-search models. Examples include stochastic game trees with chance
% nodes, where values are propagated by expectation at nature nodes
% \citep{lanctot2013monte}, and MCTS variants with backup rules beyond the
% max--min operator, such as the Monte Carlo backup operators studied by
% \citet{coulom2006efficient}. More broadly, it would be interesting to extend the
% analysis to structured pure-exploration objectives beyond max--min trees
% \citep{huang2017structured}.

% An important open problem is to close the remaining gap between the upper and lower bounds on the sample complexity for max--min identification. 
% This requires a sharper understanding of the alignment between algorithmic sampling rules and an information-theoretic characterization of structured instances, 
% as noted in Appendix~\ref{sec:appendix-perspectives}. 

% Garivier 16 (mcts paper) in their Section 6, they provide some `(still speculative) perspective of an important improvement'', including a (non explicit) lower bound on the sample complexity'' on 2 by 2 trees. Also, the optimal strategy is going to depend a lot on the position of $\mu_4$ relatively to $\mu_1$ and $\mu_2$.''

\section*{Acknowledgments}
Yinan Li, Tuan Ngo Nguyen, Kwang-Sung Jun were supported in part by the National Science Foundation under grant CCF-2327013 and Meta Platforms, Inc.
This work was partly supported by Institute of Information \& communications Technology Planning \& Evaluation (IITP) grant funded by the Korea government(MSIT) (No.RS-2019-II191906, Artificial Intelligence Graduate School Program(POSTECH)).

% \begin{ack}

% Use unnumbered first level headings for the acknowledgments. All acknowledgments

% Do {\bf not} include this section in the anonymized submission, only in the final paper. You can use the \texttt{ack} environment provided in the style file to automatically hide this section in the anonymized submission.
% \end{ack}

\clearpage
\bibliographystyle{abbrvnat}
\bibliography{library-shared}

\clearpage
\appendix

% \crefalias{section}{appendix} % uncomment if you are using cleveref

\addcontentsline{toc}{section}{Appendix} % Add the appendix text to the document TOC
\part{Appendix}
\parttoc
\section{Related Work}
\label{sec:related}

% \yinan{
% Perhaps something we need to cite in related work (which is related but not really related):

% - huang et al., "Structured Best Arm Identification with Fixed Confidence"

% - Learning Probably Approximately Correct Maximin Strategies
% in Simulation-Based Games with Infinite Strategy Spaces

% - https://sgo-workshop.github.io/CameraReady2019/4.pdf (I think there is an AAMAS conference paper that seems to be the final version of this work..)
% }

% \kj{also, kaufmann'17 and garivier'16}

% \kj{due to interest of time, I suggest writing the following points, ignoring related but not related work

% \begin{itemize}
%     \item garivier'16: they consider the fixed confidence setting with $\eps=0$ guarantee. we consider fixed budget and allow $\eps$ slack in the $\eps$-agnostic way. for the special case of $\eps=0$, our sample complexity matches theirs up to logarithmic factors.
%     \item kaufmann'17 : fixed confidence setting with an arbitrary depth. however, their sample complexity is not as tight as ours (we need to explain this in mathematical terms) for the special case of depth 2. 
%     \item huang'17: my understanding is that they generalized the MCTS structure into an arbitrary structure, but their guarantee for depth-2 maximin trees is the same as garivier'16, and thus the same comparison as garivier'16 holds.
% \end{itemize}
% }

\paragraph{MCTS, bandit planning, and fixed-budget tree search. }
Monte Carlo Tree Search (MCTS) was developed as a way to combine selective tree
search with Monte Carlo evaluation in large game trees
\citep{coulom2006efficient,browne2012survey}. A central algorithmic development
is UCT, which applies upper-confidence bandit ideas to Monte Carlo planning
\citep{kocsis2006bandit}. This connection between bandit allocation and planning
has also been studied more broadly through the optimistic principle for
large-scale optimization and planning \citep{munos2014bandits}.

Our work follows the pure-exploration interpretation of MCTS: during planning,
the learner spends a finite simulation budget and is evaluated only through the
quality of the final root action. This is distinct from regret minimization during
sampling. Several practical MCTS algorithms exploit this fixed-budget viewpoint.
For example, \citet{cazenave2014sequential} proposed SHOT, which applies
Sequential Halving to trees, and \citet{teraoka2014efficient} studied subtree
deletion rules for Monte Carlo tree search. These works are algorithmically
motivating, but they do not provide the same type of instance-dependent
fixed-budget error guarantee for noisy depth-2 max--min identification that we
study here.

\paragraph{Fixed-budget BAI, elimination algorithms, and bottom-up baselines. }
Our algorithm is inspired by the fixed-budget best-arm identification literature.
In ordinary BAI, the learner receives a budget $T$ and must recommend an arm
after the budget is exhausted. This setting is closely related to simple-regret
minimization in pure exploration \citep{bubeck09pure}. The Successive Rejects
algorithm of \citet{audibert10best} is a classical fixed-budget elimination
algorithm: it samples surviving arms in phases and removes one empirically
suboptimal arm at the end of each phase. \citet{karnin2013almost} later proposed
Sequential Halving, another elimination-based fixed-budget method with strong
guarantees.

A natural bottom-up baseline for our problem would be to first solve the
within-subtree identification problem for every subtree, and then compare the
estimated subtree minima. For example, one could apply a best-arm identification
algorithm to the negated rewards within each subtree in order to identify its
minimizing leaf. This idea is closely related to the multi-bandit BAI setting of
\citet{gabillon2011multi}, where the goal is to identify the best arm in each of
several bandit problems, and to the Successive Accepts and Rejects framework of
\citet{bubeck13multiple}, which also applies to multi-bandit best-arm identification.
Such a bottom-up approach, however, is potentially wasteful for max--min action
identification. The final decision only requires identifying a subtree with large
minimum value; it does not require solving all within-subtree minimization
problems to high accuracy. In particular, for a subtree that is clearly
suboptimal, many non-minimizing leaves may be irrelevant to the final
recommendation.

Our setting reduces to standard fixed-budget BAI when $L=1$, but the max--min
case $L>1$ introduces an additional structural difficulty. A leaf may matter
because it separates the min-values of two subtrees, or because it identifies the
minimizing leaf within a subtree. Therefore, directly applying ordinary SR to all
$KL$ leaves would ignore the max--min semantics, while a fully bottom-up
multi-bandit approach may oversolve irrelevant within-subtree problems. Our
algorithm keeps the phase-based spirit of SR, but modifies the empirical gaps
and elimination rule so that the active set remains faithful to the max--min
structure while avoiding the need to certify every subtree minimum.

\paragraph{Approximate identification and multiple correct answers.}
The $\varepsilon$-good objective is also related to pure-exploration problems
with multiple acceptable answers. In ordinary bandits, \citet{katzsamuels20true}
study the problem of identifying an arm whose mean is within $\varepsilon$ of the
best mean, and characterize how the number of good arms can change the sample
complexity. \citet{zhao2023revisiting} give an $\varepsilon$-good analysis of
Sequential Halving that holds for every $\varepsilon$ although $\varepsilon$ is
not an input to the algorithm. In the fixed-confidence setting,
\citet{degenne2019pure} study pure exploration with multiple correct answers and
show that extending Track-and-Stop-style methods to this setting requires care
because the answer map can be discontinuous.

Our problem has a related but different multiple-answer structure. The acceptable
set is
\[
    G_\varepsilon = \{ i : v_i \ge v^\star - \varepsilon \},
\]
where each value $v_i$ is itself a minimum over leaves. Thus, $\varepsilon$-good
max--min identification combines approximate root-action identification with
within-subtree minimization. Our contribution is an $\varepsilon$-agnostic
fixed-budget guarantee for this structured objective.

\paragraph{Asymptotically optimal fixed-confidence methods.}
The generic GLR/Tracking framework of \citet{kaufmann2021mixture} also applies
to best-action identification in maxmin game trees. In their
Section~5.2.2, the goal is to identify the root action with largest value in a
maxmin game tree by querying noisy samples from the leaves. For depth-two trees,
they note that the problem has rank $L+1$, where $L$ is the maximum number of
second-player actions, and that a GLR stopping rule with the corresponding
rank-based threshold is asymptotically optimal when combined with Tracking,
under the required regularity assumptions on the oracle weights.

Their optimal allocation
is characterized through the variational quantity
\[
    T^\star(\mu)^{-1}
    =
    \sup_{w \in \Sigma}
    \inf_{\lambda \in \Alt(\mu)}
    \sum_a w_a d(\mu_a,\lambda_a),
\]
and, for maxmin game trees, this characterization does not generally yield a
simple closed-form sampling rule. Even in the depth-two case, the oracle weights
are described as computable through disciplined convex optimization tools and as
only numerically computable. In contrast, our result studies the fixed-budget
criterion and gives an explicit finite-budget error-probability bound for
$\varepsilon$-good identification using an $\varepsilon$-agnostic
Successive-Rejects-style procedure.

\begin{table}[t]
\centering
\small
\caption{Comparison with closely related fixed-confidence work.}
\label{tab:comparison-fc}
\renewcommand{\arraystretch}{1.15}
\begin{tabular}{p{0.17\linewidth} p{0.17\linewidth} p{0.25\linewidth} p{0.25\linewidth}}
\toprule
\textbf{Work} & \textbf{Setting} & \textbf{Tree/model} & \textbf{Algorithmic style} \\
\midrule
\citet{garivier2016maximin}
&
Fixed confidence
&
Depth-2 maximin trees
&
M-LUCB; Maximin-Racing
\\

\citet{kaufmann17monte}
&
Fixed confidence
&
Arbitrary-depth game trees
&
BAI-MCTS with propagated confidence intervals
\\

\citet{huang2017structured}
&
Fixed confidence
&
General structured BAI / minimax search
&
LUCB-micro
\\

\citet[][Section 5.2.2]{kaufmann2021mixture}
&
Fixed confidence
&
Depth-2 maxmin game trees as a generic identification problem
&
GLR stopping + Tracking
\\

\rowcolor{gray!12}
\textbf{This paper}
&
\textbf{Fixed budget}
&
\textbf{Depth-2 max--min trees}
&
\textbf{SR-style tree-aware elimination}
\\
\bottomrule
\end{tabular}
\end{table}

\paragraph{Complexity comparison with fixed-confidence guarantees. }
Fixed-confidence and fixed-budget criteria optimize different quantities, so the following
comparison should be interpreted at the level of instance-dependent gap structure rather
than as a formal dominance statement. Fixed-confidence bounds control quantities such as
$\mathbb{E}_{\mu}[\tau_{\delta}]$ or high-probability stopping times, typically in the
regime $\delta \rightarrow 0$, whereas our result controls
\[
    \mathbb{P}_{\mu}\!\left( \widehat{i}_{T} \notin G_{\varepsilon} \right)
\]
for a deterministic budget $T$. Thus, translating a fixed-budget error exponent of the
form
\[
    \mathbb{P}_{\mu}\!\left( \widehat{i}_{T} \notin G_{\varepsilon} \right)
    \lesssim \exp\!\left(- \frac{T}{H}\right)
\]
into the heuristic relation $\delta \approx \exp(-T/H)$ is useful only for comparing
the leading inverse-gap terms. It should not be interpreted as saying that one setting
strictly dominates the other. In what follows, we compare the complexity measures only
to clarify how the same max--min gap structure appears across fixed-budget and
fixed-confidence analyses.

To the best of our knowledge, our work provides the first \emph{fixed-budget} guarantees for maximin action identification. Prior theoretical results for MCTS identification are primarily in the
fixed-confidence setting, where the goal is to return the optimal subtree with probability at least
$1-\delta$ and to minimize the (random) stopping time. The closest fixed-confidence references are~\citet{garivier2016maximin, kaufmann17monte, huang2017structured}.

\citet{garivier2016maximin} study depth-$2$ maximin trees under $\eps=0$ (exact identification) in the
fixed-confidence setting. They propose two algorithms whose sample complexities are characterized by
instance-dependent terms $H_1^\star(\mu)$ and $H_2^\star(\mu)$.
In our notation (with indices defined in Equation~\ref{eq:wlog-order}), their complexities can be written as
\[
H_1^{\star}(\mu)
= \sum_{j\in [L]} \frac{1}{\Bigl(\mu_{1,j}-\tfrac{\mu_{1,1}+\mu_{2,1}}{2}\Bigr)^{2}}
\;+\;
\sum_{(i,j): i \neq 1}
\frac{1}{\Bigl(\tfrac{\mu_{1,1}+\mu_{2,1}}{2}-\mu_{i,1}\Bigr)^{2}\,\vee\,\Bigl(\mu_{i,j}-\mu_{i,1}\Bigr)^{2}},
\]
and
\[
H_2^{\star}(\mu)
:= \sum_{j\in[L]} \frac{1}{\left(\mu_{1,j}-\mu_{1,1}\right)^{2}\,\vee\,\left(\mu_{1,1}-\mu_{2,1}\right)^{2}}
\;+\;
\sum_{(i,j): i \neq 1}
\frac{1}{\left(\mu_{1,1}-\mu_{i,1}\right)^{2}\,\vee\,\left(\mu_{i,j}-\mu_{i,1}\right)^{2}}.
\]
Our fixed-budget sample complexity (up to logarithmic factors) is
\[
H_2(\varepsilon):=\max_{r\ge m+1} r\,\Delta_{(r)}^{-2}, 
\qquad\text{with}\qquad
H_2(\varepsilon)\le H_1(\varepsilon)
\]
by Lemma~\ref{lem:H2-H1}.
At $\eps=0$, 
% our algorithm achieves a sample complexity  at most
\begin{align}
\label{eqn:H1-exact-identification}
    H_1(0)
= \sum_{j} \frac{1}{\left(\mu_{1,j}-\mu_{2,1}\right)^2}
\;+\;
\sum_{i\neq 1}\sum_{j} \frac{1}{\left(\left(\mu_{i,j}-\mu_{i,1}\right)\vee\left(\mu_{1,1}-\mu_{i,1}\right)\right)^2}.
\end{align}
A direct comparison shows that both complexities in~\citet{garivier2016maximin} match ours up to constant factors:
\[
H_1(0)\ \le\ H_1^\star(\mu)\ \le\ 4H_1(0),
\qquad
H_1(0)\ \le\ H_2^\star(\mu)\ \le\ 4H_1(0).
\]
In the special case $\eps=0$, our fixed-budget sample complexity thus matches their fixed-confidence sample complexity  up to
logarithmic factors.

\citet{kaufmann17monte} study MCTS as best-arm identification in trees of arbitrary depth in the fixed-confidence setting, with an $\eps$-relaxation. 
% Their sample complexity is controlled by
% \[
% H_\varepsilon^{\star}(\mu)
% := \sum_{\ell \in \mathcal{L}}
% \frac{1}{\Delta_\ell^{2} \,\vee\, \Delta_\star^{2} \,\vee\, \varepsilon^{2}},
% \qquad \text{where}\qquad
% \Delta_\star := V(s^{\star}) - V(s_2^{\star}),
% \qquad
% \Delta_\ell := \max_{s \in \operatorname{Anc}(\ell)\setminus\{s_0\}}
% \bigl|V(s) - V(P(s))\bigr|.
% \]
Specializing to depth-$2$ maximin trees, their sample complexity is 
% $\Delta_\star=\mu_{1,1}-\mu_{2,1}$ and, for any leaf
% $\ell=(i,j)$, the path-based gap reduces to $\Delta_\ell=\mu_{1,1}-\mu_{i,1}$.
% Therefore,
\[
H_\varepsilon^{\star}(\mu)
=
\sum_{j}
\frac{1}{\left((\mu_{1,1}-\mu_{2,1})\vee \varepsilon\right)^2}
\;+\;
\sum_{i\neq 1}\sum_j
\frac{1}{\left((\mu_{1,1}-\mu_{i,1})\vee \varepsilon\right)^2}.
\]
In contrast, our leafwise sample complexity (for $\eps$-agnostic identification) is
\[
H_1(\varepsilon)
=
\sum_{j}
\frac{1}{\left((\mu_{1,j}-\mu_{2,1})\vee \varepsilon\right)^2}
\;+\;
\sum_{i\neq 1}\sum_j
\frac{1}{\left(\left((\mu_{i,j}-\mu_{i,1})\vee(\mu_{1,1}-\mu_{i,1})\right)\vee \varepsilon\right)^2},
\]
The key difference is that $H_\varepsilon^{\star}(\mu)$ uses the same tree-level separation
$(\mu_{1,1}-\mu_{i,1})$ for every leaf in subtree $i$, and thus repeatedly ``pays'' for leaves that are not determining the min value of a sub-optimal subtree. Their $H_\varepsilon^{\star}(\mu)$ can be much larger than our
$H_1(\varepsilon)$ (and hence than our $H_2(\varepsilon)$), because $H_1(\varepsilon)$ depends on the
leaf-specific gap quantity
\[
(\mu_{i,j}-\mu_{i,1})\vee(\mu_{1,1}-\mu_{i,1}),
\]
which can be arbitrarily large for most leaves even when $(\mu_{1,1}-\mu_{i,1})$ is small.

\citet{huang2017structured} generalizes the MCTS identification problem to arbitrary structure in the fixed-confidence setting. When specialized to depth-$2$ maximin trees, their instance-dependent
complexity reduces to the same form as in~\citet{garivier2016maximin}, and therefore the same order with our $H_1(0)$. 
% In particular, in the exact-identification case $\eps=0$, their guarantees
% match ours up to logarithmic factors.

\revisedi{

\paragraph{Fixed-budget complexity and optimality.}
A recent line of work shows that optimal sample-complexity characterizations in
the fixed-budget setting are substantially more delicate than in the
fixed-confidence setting, even for ordinary best-arm identification. 
\citet{komiyama2022minimax} study minimax optimality for fixed-budget BAI through
normalized error exponents. They emphasize that, unlike in fixed confidence, a
fixed-budget algorithm cannot in general fully adapt to every instance: improving
the probability of error on one instance may worsen it on another. Their optimal
rate $R^{\mathrm{go}}$ is defined through an oracle allocation that depends on
the final empirical distribution after $T$ rounds, and the corresponding
tracking problem is nontrivial. They also introduce a second rate
$R^{\mathrm{go}}_\infty$ and a conceptual delayed optimal tracking algorithm
achieving it, but the resulting procedure is computationally almost infeasible.

This perspective is reinforced by \citet{degenne2023existence}, who studies
whether fixed-budget identification tasks admit a ``complexity'' analogous to the
fixed-confidence characteristic time: a lower bound on the error exponent that
is attained by a single algorithm on all instances. He shows that such a
complexity need not exist, even in simple best-arm identification problems
including Bernoulli BAI with two arms and Gaussian BAI for sufficiently large
$K$. These results suggest that, in fixed-budget problems, there may be no
single universally optimal instance-dependent complexity attained uniformly by a
practical adaptive algorithm.

Our work is consistent with this viewpoint. We do not claim a minimax-optimal
fixed-budget exponent for max--min trees. Instead, we give an explicit
finite-budget, instance-dependent upper bound for an $\varepsilon$-agnostic
Successive-Rejects-style algorithm. The remaining gap between our upper and lower
bounds should therefore be interpreted in light of the broader difficulty of
fixed-budget optimality: even in unstructured BAI, the existence and attainability
of a sharp universal complexity is subtle.
}

\section{Proof of Theorem~\ref{thm:main-upper-bound}}
\label{sec:appendix-upper-bound-proof}
Throughout the analysis, we assume without loss of generality that the critical gap $\Delta^\star$ is well-defined, i.e., not all subtrees are $\varepsilon$-good.
Indeed, if every subtree is $\varepsilon$-good, then every output is acceptable and the misidentification probability is identically zero, so the theorem holds trivially.
\subsection{Warmup: Analysis of phase 1}
\label{sec:appendix-phase-1-analysis}
\paragraph{Recap on the empirical minimizers and empirical gaps.}
Recall from Section~\ref{sec:sr-to-subtree}, 
for a fixed phase $k$, let $\hat{\mu}_{i,j}$ denote the empirical mean of leaf $(i,j)$ computed from all samples collected so far. 
Recall that $\cA$ is the current set of active leaves. 
and, for each subtree $i$, let $\cA_i := \cA \cap (\{i\}\times[L])$ be its active leaves.
For each active subtree $i$ (i.e., $\cA_i\neq\emptyset$), define its empirical minimizer index
\[
\hat{1}(i) \in \arg\min_{(i,j)\in \cA_i} \ \hat{\mu}_{i,j},
\]
and the corresponding empirical min-value $\hat{v}_i := \hat{\mu}_{i,\hat{1}(i)}$.
Define the empirically best subtree
\[
\hat{a} \in \arg\max_{i:\ \cA_i\neq\emptyset} \ \hat{v}_i. 
\]
We further define the empirically second-best subtree
\[
\hat{b} \in \arg\max_{i:\ \cA_i\neq\emptyset,\ i\neq \hat{a}} \ \hat{v}_i,
\]
breaking ties arbitrarily. Recall as well the definition of empirical gaps $\hat{\Delta}_{i,j}$ from Section~\ref{sec:sr-to-subtree} and the definion of the empirically largest gap $\hat \Delta_{\max} = \max_{(i,j) \in \Acal} \hat{\Delta}_{i,j}$ from line~\ref{line:hat-Delta-max} in Algorithm~\ref{alg:sr-for-mcts}. Note that in phase 1, $\Acal = [K]\times[L]$.

\begin{lemma}
    At phase 1, 
      if all leaves satisfy that $\abs{\hmu - \mu} < \frac{\Delta_{(KL)}}{8}$, then
    $\hat \Delta_{\max} > \frac{3}{4} \Delta_{(KL)}$. 
    \label{lem:Delta-hat-max-phase-1}
\end{lemma}
\begin{proof}
    \begin{enumerate}
        \item If $\Delta_{(KL)}$ is given by $\mu_{x,L} - \mu_{x,1}$ for some $x \neq 1$. 

        If $x$ is not the empirically best subtree, 
        \[
        \hat \Delta_{x,L} = \max(\hat \mu_{x,L} - \hmu_{x, \hat{1}(x)}, \hmu_{\hat{a}, \hat{1}(\hat{a})} - \hmu_{x, \hat{1}(x)})
        \geq \hat \mu_{x,L} - \hmu_{x, \hat{1}(x)} \geq \hmu_{x,L} - \hmu_{x,1} > \frac{3}{4} \Delta_{(KL)} 
        \] 

        If $x$ is the empirically best subtree, \[
        \hat \Delta_{x,L} = \hat \mu_{x,L} - \hmu_{\hat{b}, \hat{1}(\hat{b})} \geq \hat \mu_{x,L} - \hmu_{x, \hat{1}(x)} \geq \hmu_{x,L} - \hmu_{x,1} > \frac{3}{4} \Delta_{(KL)}
        \]

        Hence $\hat \Delta_{\max} \geq \hat \Delta_{x,L} > \frac{3}{4} \Delta_{(KL)}$.

        \item If $\Delta_{(KL)}$ is given by $\mu_{1,L} - \mu_{2,1}$. In this case the condition that $\abs{\hmu - \mu} < \frac{\Delta_{(KL)}}{8}$ for all leaves ensures that $\hmu_{1,L} - \hmu_{x,1} > \frac{3}{4} \Delta_{(KL)}, \forall x\neq 1$. 
        
        If $1$ is not the empirically best subtree, then there $\exists x\neq 1$ such that $x$ is the empirically best subtree,
        \[
        \hat \Delta_{1,L} = \max(\hat \mu_{1,L} - \hmu_{1, \hat{1}(1)}, \hmu_{\hat{a}, \hat{1}(\hat{a})} - \hmu_{1, \hat{1}(1)})
        \geq \hat \mu_{1,L} - \hmu_{1, \hat{1}(1)} 
        \geq \hat \mu_{1,L} - \hmu_{x, \hat{1}(x)}
        \geq \hmu_{1,L} - \hmu_{x,1} > \frac{3}{4} \Delta_{(KL)} 
        \] 

        If $1$ is the empirically best subtree, \[
        \hat \Delta_{1,L} = \hat \mu_{1,L} - \hmu_{\hat{b}, \hat{1}(\hat{b})} \geq \hat \mu_{1,L} - \hmu_{\hat{b},1}  > \frac{3}{4} \Delta_{(KL)}
        \]

        Hence $\hat \Delta_{\max} \geq \hat \Delta_{1,L} > \frac{3}{4} \Delta_{(KL)}$.
        \item If $\Delta_{(KL)}$ is given by $\mu_{1,1} - \mu_{K,1}$. In this case the condition that $\abs{\hmu - \mu} < \frac{\Delta_{(KL)}}{8}$ for all leaves ensures that $\hmu_{1,j} - \hmu_{K,1} > \frac{3}{4} \Delta_{(KL)}, \forall j \in [L]$. Hence, $K$ is not the empirically best subtree, and 
        \begin{align*}
            \hat \Delta_{K,1} 
            &= \max(\hat \mu_{K,1} - \hmu_{K, \hat{1}(K)}, \hmu_{\hat{a}, \hat{1}(\hat{a})} - \hmu_{K, \hat{1}(K)})
        \\&\geq \hmu_{\hat{a}, \hat{1}(\hat{a})} - \hmu_{K, \hat{1}(K)} 
        \geq \hmu_{1, \hat{1}(1)} - \hmu_{K, \hat{1}(K)} 
        \geq \hmu_{1, \hat{1}(1)} - \hmu_{K, 1} > \frac{3}{4} \Delta_{(KL)} 
        \end{align*}
    \end{enumerate}
\end{proof}

\begin{lemma}
  At phase 1, 
  if all leaves satisfy that $\abs{\hmu - \mu} < \frac{\Delta_{(KL)}}{8}$, then the optimal leaves in suboptimal subtrees, i.e., $\cbr{(y,1): y \neq 1 }$, will not be eliminated by the single leaf elimination rule. 
  \label{lem:optimal-leaf-in-suboptimal-subtree-phase-1-improved}
\end{lemma}
\begin{proof}
Let $y \neq 1$, we will show that if $\abs{\hmu - \mu} < \frac{\Delta_{(KL)}}{8}$ for all leaves, then $(y,1)$ will not be eliminated by the single-leaf elimination rule. 

    \begin{enumerate}
        \item If $\hat \Delta_{y,1}$ is given by $\hmu_{\hat{a}, \hat{1}(\hat{a})} - \hmu_{y, \hat{1}(y)}$, 
        \begin{enumerate}
            \item If there exists a leaf $(y,j)$ such that $\hmu_{y,j} > \hmu_{\hat{a}, \hat{1}(\hat{a})}$, then \[
            \hat \Delta_{y,j} 
            = \max(\hat \mu_{y,j} - \hmu_{y, \hat{1}(y)}, \hmu_{\hat{a}, \hat{1}(\hat{a})} - \hmu_{y, \hat{1}(y)})
            = \hat \mu_{y,j} - \hmu_{y, \hat{1}(y)} 
            > \hmu_{\hat{a}, \hat{1}(\hat{a})} - \hmu_{y, \hat{1}(y)}
            = \hat \Delta_{y,1}
            \]
            hence $(y,1)$ will not be eliminated by the single-leaf elimination rule.
            \item If all leaves $(y,j)$'s are such that $\hmu_{y,j} \leq \hmu_{\hat{a}, \hat{1}(\hat{a})}$, then $\forall j$,  \[
            \hat \Delta_{y,j} 
            = \max(\hat \mu_{y,j} - \hmu_{y, \hat{1}(y)}, \hmu_{\hat{a}, \hat{1}(\hat{a})} - \hmu_{y, \hat{1}(y)})
            = \hmu_{\hat{a}, \hat{1}(\hat{a})} - \hmu_{y, \hat{1}(y)}
            = \hat \Delta_{y,1}
            \]
            hence either subtree $y$ will be eliminated entirely, or one leaf not in $y$ will be eliminated, or leaf/leaves in another subtree not $y$ will be eliminated. 
        \end{enumerate}
        
        \item If $\hat \Delta_{y,1}$ is given by $\hmu_{y,1} - \hmu_{y, \hat{1}(y)}$, then \[
        \hmu_{y,1} - \hmu_{y, \hat{1}(y)}
        = \hmu_{y,1} - \mu_{y,1} + \mu_{y,1} - \mu_{y, \hat{1}(y)} + \mu_{y, \hat{1}(y)} - \hmu_{y, \hat{1}(y)} 
        < \frac{\Delta_{(KL)}}{4}
        \]
        By Lemma~\ref{lem:Delta-hat-max-phase-1}, $\hat \Delta_{\max} > \frac{3}{4} \Delta_{(KL)}$, hence $\hat \Delta_{y,1} < \hat \Delta_{\max}$, and $(y,1)$ will not be eliminated by the single-leaf elimination rule.

        \item If $y$ is the empirically best subtree and $\hat \Delta_{y,1}$ is given by $\hmu_{y,1} - \hmu_{\hat{b}, \hat{1}(\hat{b})}$, then 
        \[
        \hmu_{y,1} - \hmu_{\hat{b}, \hat{1}(\hat{b})}
        \leq \hmu_{y,1} - \hmu_{1, \hat{1}(1)}
        = \hmu_{y,1} - \mu_{y,1} + \mu_{y,1} - \mu_{1,\hat{1}(1)} + \mu_{1,\hat{1}(1)} - \hmu_{1, \hat{1}(1)}
        < \frac{\Delta_{(KL)}}{4}
        \]
        where the first inequality is because $y$ is the empirically best subtree, and hence $\hmu_{\hat{b}, \hat{1}(\hat{b})} \geq \hmu_{1, \hat{1}(1)}$. 
        % \[
        % \hmu_{y,1} - \hmu_{\hat{b}, \hat{1}(\hat{b})}
        % \leq \hmu_{y,1} - \hmu_{\hat{a}, \hat{1}(\hat{a})}
        % = \hmu_{y,1} - \mu_{y,1} + \mu_{y,1} - \mu_{a,1} + \mu_{a,1} - \hmu_{a,1} + \hmu_{a,1} - \hmu_{\hat{a}, \hat{1}(\hat{a})}
        % < \frac{\Delta_{(KL)}}{4}
        % \]
        
        By Lemma~\ref{lem:Delta-hat-max-phase-1}, $\hat \Delta_{\max} > \frac{3}{4} \Delta_{(KL)}$, hence $\hat \Delta_{y,1} < \hat \Delta_{\max}$, and $(y,1)$ will not be eliminated by the single-leaf elimination rule.
        
    \end{enumerate}
\end{proof}

\begin{lemma}
    At phase 1, if all leaves satisfy that $\abs{\hmu - \mu} < \frac{\Delta_{(KL)}}{8}$, then subtree $1$ will not become empty. 
    \label{lem:subtree-a-not-empty-phase-1-improved}
\end{lemma}
\begin{proof}
    We note that it suffices to show that subtree $1$ will not be eliminated by the subtree elimination rule since the single-leaf elimination rule for a single-leaf subtree coincides with the subtree elimination rule. 

    We prove this lemma by contradiction in each of the following cases. 
    \begin{enumerate}
        \item If $\Delta_{(KL)}$ is given by $\mu_{x,L} - \mu_{x,1}$ for some $x \neq 1$. 

        Since $\abs{\hmu - \mu} < \frac{\Delta_{(KL)}}{8}$ for all leaves, we have that $\hmu_{x,L} - \hmu_{x,1} > \frac{3}{4} \Delta_{(KL)}$. If it were the case that the subtree $x$ is the empirically best subtree, and subtree $1$ gets eliminated, then $\hat \Delta_{x,L} \leq \hat \Delta_{1,1}$. However, $\hat \Delta_{x,L} = \hat \mu_{x,L} - \hmu_{\hat{b}, \hat{1}(\hat{b})} \geq \hat \mu_{x,L} - \hmu_{x, \hat{1}(x)} \geq \hmu_{x,L} - \hmu_{x,1} > \frac{3}{4} \Delta_{(KL)}$, but $\hat \Delta_{1,1} = \hmu_{x, \hat{1}(x)} - \hat \mu_{1,1} \leq \hmu_{x, 1} - \hat \mu_{1,1} = 
        \hmu_{x, 1} - \mu_{x, 1} + \mu_{x,1} - \mu_{1,1}+ \mu_{1,1}- \hat \mu_{1,1} 
        < \frac{1}{4} \Delta_{(KL)}
        $. Contradiction.

        If it were the case that the subtree $y$, with $y \neq 1, ~y \neq x$ is the empirically best subtree, and subtree $1$ gets eliminated, then $\hat \Delta_{x,L} \leq \hat \Delta_{1,1}$. However, $\hat \Delta_{x,L} \geq \hat \mu_{x,L} - \hmu_{x, \hat{1}(x)} \geq \hmu_{x,L} - \hmu_{x,1} > \frac{3}{4} \Delta_{(KL)}$, but $\hat \Delta_{1,1} = \hmu_{y, \hat{1}(y)} - \hat \mu_{1,1} \leq \hmu_{y, 1} - \hat \mu_{1,1} = 
        \hmu_{y, 1} - \mu_{y, 1} + \mu_{y,1} - \mu_{1,1}+ \mu_{1,1}- \hat \mu_{1,1} 
        < \frac{1}{4} \Delta_{(KL)}
        $. Contradiction.

        \item If $\Delta_{(KL)}$ is given by $\mu_{1,L} - \mu_{2,1}$. 

        In this case the condition that $\abs{\hmu - \mu} < \frac{\Delta_{(KL)}}{8}$ for all leaves ensures that $\hmu_{1,L} - \hmu_{x,1} > \frac{3}{4} \Delta_{(KL)}, \forall x\neq a$. If subtree $1$ were eliminated, then there $\exists x \neq 1$ such that $x$ is the empirically best subtree and $\hmu_{x,1} \geq \hmu_{x,\hat{1}(x)} \geq \hmu_{1,L}$, which leads to a contradiction. 

        \item If $\Delta_{(KL)}$ is given by $\mu_{1,1} - \mu_{K,1}$. 

        In this case the condition that $\abs{\hmu - \mu} < \frac{\Delta_{(KL)}}{8}$ for all leaves ensures that $\hmu_{1,j} - \hmu_{K,1} > \frac{3}{4} \Delta_{(KL)}, \forall j \in [K]$. If subtree $1$ were eliminated, then there $\exists x \neq 1$ such that $x$ is the empirically best subtree and 
        $\forall j, \hat \Delta_{1,j} = \hmu_{x,\hat{1}(x)} - \hmu_{1,j}  = \hmu_{x,\hat{1}(x)} - \mu_{x,\hat{1}(x)} + \mu_{x,\hat{1}(x)} - \mu_{1,j} + \mu_{1,j} - \hmu_{1,j} \leq \frac{1}{4} \Delta_{(KL)}$, 
        which leads to a contradiction with that, $\forall j, \hat \Delta_{1,j} = \hat \Delta_{\max}$, since $\hat \Delta_{\max} > \frac{3}{4} \Delta_{(KL)}$ by Lemma~\ref{lem:Delta-hat-max-phase-1}. 
    \end{enumerate}
\end{proof}

\begin{lemma}
\label{lem:half-epsilon-subtree-phase-1}
Let subtree $s \neq 1$ be $\eps/2$-good. 
    At phase 1, 
  if all leaves satisfy that $\abs{\hmu - \mu} < \frac{\Delta_{(KL)}}{8}$, then $(s,1)$ will not be eliminated. 
  % then the optimal leaves in $\eps/2$-good subtrees, i.e., $\cbr{(s,1): s \text{ is } \eps/2 \text{-optimal}}$, will not be eliminated. 
\end{lemma}

\begin{proof}
    In the proof of Lemma~\ref{lem:optimal-leaf-in-suboptimal-subtree-phase-1-improved}, in cases 2 and 3, we've shown that $\hDelta_{s,1} < \frac{\Delta_{(KL)}}{4}$. We will now upper bound $\hDelta_{s,1}$ in case 1 where $\hat \Delta_{s,1}$ is given by $\hmu_{\hat{a}, \hat{1}(\hat{a})} - \hmu_{s, \hat{1}(s)}$. 

    \begin{align*}
        \hmu_{\hat{a}, \hat{1}(\hat{a})} - \hmu_{s, \hat{1}(s)}
        &\leq
        \hmu_{\hat{a},1} - \hmu_{s, \hat{1}(s)}
        \\&< 
        \mu_{\hat{a},1} + \frac{\Delta_{(KL)}}{8} - \mu_{s, \hat{1}(s)} + \frac{\Delta_{(KL)}}{8}
        \\&\leq 
        \mu_{1,1} - \mu_{s, \hat{1}(s)} + \frac{\Delta_{(KL)}}{4}
        \\&\leq 
        \mu_{1,1} - \mu_{s, 1} + \frac{\Delta_{(KL)}}{4}
        \\&\leq 
        \fr \eps 2 + \frac{\Delta_{(KL)}}{4}
        \tag{$s$ is $\eps/2$-good}
        \\&< 
        \frac{3\Delta_{(KL)}}{4}
        \tag{$\eps < \Dt^* \leq \Dt_{(KL)}$}
    \end{align*}
    By Lemma~\ref{lem:Delta-hat-max-phase-1}, $\hat \Delta_{\max} > \frac{3}{4} \Delta_{(KL)}$, hence $\hat \Delta_{s,1} < \hat \Delta_{\max}$, and $(s,1)$ will not be eliminated
\end{proof}

\subsection{Analysis of a generic phase $k$ in early regime ($k \leq KL-m$)}
% \yinan{To prove these lemmas, one also needs to note that for any leaf, its $\Dt$ can remain the same or increase after performing any phase $k \leq KL-m$ that maintains the soundness. TODO: show this. As a result, at the beginning of any phase $k \leq KL-m$, the true max gap $\Dt_{\max} \geq \Delta_{(KL+1-k)}$. }

\paragraph{Phase-$k$ (restricted) gaps.}
Recall that
$\nu=\{\nu_{i,j}\}_{(i,j)\in[K]\times[L]}$ is the original instance.
At the beginning of phase $k$, the algorithm maintains an active leaf set $\cA^{(k)}\subseteq[K]\times[L]$.
For each subtree $i$, let $\cA_i^{(k)}:=\cA^{(k)}\cap(\{i\}\times[L])$ denote its active leaves; we call subtree $i$ \emph{active} if $\cA_i^{(k)}\neq\emptyset$.
The eliminations performed in earlier phases effectively restrict the instance to the remaining leaves.
Accordingly, we define \emph{phase-$k$ gaps} by applying the same max-min gap construction as in
\eqref{eq:gap-opt}--\eqref{eq:gap-subopt}, but with all minima and competitor comparisons taken \emph{within the remaining tree}.

Formally, define the (mean) min-value of subtree $i$ restricted to active leaves as
\[
v_i^{(k)} := \min\{\mu_{i,j} : (i,j)\in\cA^{(k)}\},
\qquad i\in[K]\ \text{with}\ \cA_i^{(k)}\neq\emptyset,
\]
and let $i_*^{(k)}\in\arg\max_{i:\cA_i^{(k)}\neq\emptyset} v_i^{(k)}$ be a best active subtree under this restriction.
Let $v_{(2)}^{(k)} := \max\{v_i^{(k)} : i\neq i_*^{(k)},\ \cA_i^{(k)}\neq\emptyset\}$ denote the second-best active min-value (breaking ties arbitrarily).

% For each active leaf $(i,j)\in\cA^{(k)}$, the phase-$k$ gap $\Delta_{i,j}^{(k)}$ measures how easy it is to dismiss this leaf (or its subtree)
% \emph{given only the remaining leaves}:
% \begin{itemize}
% \item If $i=i_*^{(k)}$ (leaf in the currently best active subtree), then
% $\Delta_{i,j}^{(k)}$ compares the leaf mean $\mu_{i,j}$ to the best competitor subtree among the remaining ones;
% \item If $i\neq i_*^{(k)}$ (leaf in a non-best active subtree), then
% $\Delta_{i,j}^{(k)}$ is the maximum of a \emph{cross-subtree} term (how far subtree $i$'s restricted min-value is from the best restricted min-value)
% and a \emph{within-subtree} term (how far leaf $(i,j)$ is above subtree $i$'s restricted min-value).
% \end{itemize}
Equivalently, $\Delta_{i,j}^{(k)}$ is the ``gap'' that would be assigned to leaf $(i,j)$ if one considered the induced instance consisting only of active leaves at phase $k$.

With this interpretation, we define
\[
\Delta_{i,j}^{(k)} :=
\begin{cases}
\mu_{i,j} - v_{(2)}^{(k)}, & \text{if } i = i_*^{(k)},\\[3pt]
\max\big\{v_{i_*^{(k)}}^{(k)} - v_i^{(k)},\ \mu_{i,j} - v_i^{(k)}\big\}, & \text{if } i\neq i_*^{(k)}.
\end{cases}
\]
Furthermore, define
\begin{equation}
\label{eq:Delta-max}
\Delta_{\max}^{(k)} \;:=\; \max_{(i,j)\in\cA^{(k)} }\ \Delta_{i,j}^{(k)}.
\end{equation}

\paragraph{Phase-$k$ residual instance.}
We define the \emph{phase-$k$ residual instance} as the restriction of $\nu$ to active leaves,
\[
\nu^{(k)} := \{\nu_{i,j} : (i,j)\in \cA^{(k)}\}.
\]

In the rest of this section, for the simplicity of the notation, and to focus on the logic through phases of elimination, we denote by:
\begin{itemize}
    \item $a$, the optimal subtree within $\nu^{(k)}$, i.e., $a := i_*^{(k)}$. 
    \item $b$, the second optimal subtree within $\nu^{(k)}$, i.e., $b \in \arg\max_{i:i\neq i_*^{(k)}, \cA_i^{(k)}\neq\emptyset} v_i^{(k)}$. 
    \item $z$, the least optimal subtree within $\nu^{(k)}$, i.e., $z \in \arg\min_{i:\cA_i^{(k)}\neq\emptyset} v_i^{(k)}$. 
\end{itemize}
Recall that in Appendix~\ref{sec:appendix-phase-1-analysis}, with the same spirit, we denoted by $\ha$ the empirically optimal subtree within the reduced instance, and $\hb$ the empirically second optimal subtree within the reduced instance. 
% \yinan{Add more explanation, maybe refer to the algorithm.}
Define the active-maximum index
\begin{align}
\label{eq:active-max-index}
    h_i \in \arg\max_{(i,j)\in \cA_x^{(k)}} \mu_{i,j}.
\end{align}

Note that we omit the phase index $k$ on $a,b,z,\ha,\hb,h_i$, as the phase index is clear from the context in all of the following lemmas. 

The next immediate lemma establishes the key invariants maintained by ``$\eps$-soundness'': 
\begin{enumerate}
    \item Subtree 1 remains to be the optimal subtree in $\nu^{(k)}$; 
    \item The min leaf (which defines the sub-optimality of the subtree) is surviving, if the subtree is still active. 
\end{enumerate}
These are the key structures of a tree that our algorithm maintains with high probability with $\eps$-soundness, so that the residual instance $\nu^{(k)}$ does not become ``confusing'' in terms of the max-min subtree compared to the original tree instance.  

\begin{lemma}
    Let $k \leq KL-m$. 
    Suppose after $k-1$ phases, the algorithm has not terminated, and phase $k$ is $\eps$-sound. Then, 
    \begin{enumerate}
        \item $a = 1$.
        \item For any $i \geq 2, j \in [L]$, if $(i,j) \in \cA^{(k)}$, then $(i,1) \in \cA^{(k)}$. 
    \end{enumerate}
    \label{lem:basic-invariant-early-phase}
\end{lemma}

\begin{proof}
    The first statement follows directly from items~\ref{item:soundness-1} and \ref{item:soundness-3} in Definition~\ref{def:soundness}. 

    The second statement follows directly from item~\ref{item:soundness-3} in Definition~\ref{def:soundness}. 
\end{proof}

\begin{lemma}
    Let $k \leq KL-m$. 
    Suppose after $k-1$ phases, the algorithm has not terminated, and phase $k$ is $\eps$-sound. Then, for each $(i,j) \in \cA^{(k)}$, 
    \[
    \Dt_{i,j}^{(k)} \geq \Dt_{i,j}
    \]
    \label{lem:gap-monotone}
\end{lemma}

\begin{proof}
    Fix any $(i,j)\in\cA$.
    \begin{itemize}
        \item Case 1: $i=1$. By Lemma~\ref{lem:basic-invariant-early-phase}, $\Dt_{i,j}^{(k)}$ is given by $\mu_{i,j} - \mu_{x,1}$, for some subtree $x \neq 1$. Thus, 
        \[
        \Dt_{i,j}^{(k)} = \mu_{i,j} - \mu_{x,1} \geq \mu_{i,j} - \mu_{2,1} = \Dt_{i,j}. 
        \]
        \item Case 2: $i \neq 1$. By the first item in Lemma~\ref{lem:basic-invariant-early-phase}, subtree $i$ is a subtimal subtree; moreover, 
         
        \begin{align*}
            \Dt_{i,j}^{(k)} 
            &= \max\big\{v_1^{(k)} - v_i^{(k)},\ \mu_{i,j} - v_i^{(k)}\big\}
            \\&= \max\big\{v_1^{(k)} - \mu_{i,1},\ \mu_{i,j} - \mu_{i,1} \big\}
            \tag{Second item in Lemma~\ref{lem:basic-invariant-early-phase}}
            \\&\geq \max\big\{\mu_{1,1} - \mu_{i,1},\ \mu_{i,j} - \mu_{i,1} \big\}
            \tag{$v_1^{(k)} \geq \mu_{1,1}$}
            \\&= \Dt_{i,j}
        \end{align*}
    \end{itemize}
\end{proof}

\begin{lemma}
    Let $k \leq KL-m$. 
    Suppose after $k-1$ phases, the algorithm has not terminated, and phase $k$ is $\eps$-sound. Let $\Dt_{\max}^{(k)}$ denote the max phase-$k$ gap at the beginning of phase $k$ (see Equation~\ref{eq:Delta-max} for the definition), then $\Dt_{\max}^{(k)} \geq \Delta_{(KL+1-k)}$.
    \label{lem:max-gap-non-decreasing}
\end{lemma}

\begin{proof}
    At the beginning of phase $k$, there are $k-1$ leaves already eliminated and $KL-k+1$ leaves surviving, that is, $|\cA^{(k)}| = KL+1-k$. We have, 
    \begin{align*}
        \Dt_{\max}^{(k)} 
        &= \max_{(i,j)\in\cA^{(k)} }\ \Delta_{i,j}^{(k)} 
        \\&\geq \max_{(i,j)\in\cA^{(k)} }\ \Delta_{i,j}
        \tag{Lemma~\ref{lem:gap-monotone}}
        \\&\geq
        \Delta_{(KL+1-k)}
    \end{align*}
    where the last inequality is because, for any $\cS \subseteq [K]\times[L]$ with $|\cS| = KL+1-k$, we have 
    $\max_{(i,j)\in \cS }\ \Delta_{i,j} \geq
        \Delta_{(KL+1-k)}$. 
\end{proof}

The following lemmas are direct extensions of Lemma~\ref{lem:Delta-hat-max-phase-1} to \ref{lem:subtree-a-not-empty-phase-1-improved} and Lemma~\ref{lem:half-epsilon-subtree-phase-1} to any phase $k \leq KL-m$. A similar proof idea follows with replacing the original instance $\nu$ (in Lemma~\ref{lem:Delta-hat-max-phase-1} to \ref{lem:half-epsilon-subtree-phase-1}) with phase-$k$ residual instance $\nu^{(k)}$. 
Thanks to Lemma~\ref{lem:max-gap-non-decreasing}, in order to control the deviation $(\abs{\hmu - \mu})$ with respect to $\Dt_{\max}^{(k)}$, it suffices to control it w.r.t. $\Delta_{(KL+1-k)}$. 

Despite the similarity of the proofs for the following lemmas in this subsection compared to Lemmas~\ref{lem:Delta-hat-max-phase-1} to \ref{lem:half-epsilon-subtree-phase-1}, we include the proofs for completeness. 

\begin{lemma}
    Let $k \leq KL-m$. 
    Suppose after $k-1$ phases, the algorithm has not terminated, and phase $k$ is $\eps$-sound. If all leaves satisfy that $\abs{\hmu - \mu} < \frac{\Delta_{(KL+1-k)}}{8}$, then
    $\hat \Delta_{\max} > \frac{3}{4} \Delta_{(KL+1-k)}$. 
    \label{lem:Delta-hat-max}
\end{lemma}

\begin{proof}
Recall $\Delta_{\max}^{(k)}$ defined in Equation~\eqref{eq:Delta-max}. 
Recall the active-maximum index $h_i$ defined in Equation~\eqref{eq:active-max-index}. By Lemma~\ref{lem:max-gap-non-decreasing}, all leaves satisfy that $\abs{\hmu - \mu} < \frac{\Delta_{(KL+1-k)}}{8} \leq \frac{\Delta_{\max}^{(k)}}{8}$. 

By the $\eps$-soundness of phase $k$ and Lemma~\ref{lem:basic-invariant-early-phase}, there are following three cases: 
    \begin{enumerate}
        \item If $\Delta_{\max}^{(k)}$ is given by $\mu_{x,h_x} - \mu_{x,1}$ for some $x \neq a$. 

        If $x$ is not the empirically best subtree, 
        \[
        \hat \Delta_{x,h_x} = \max(\hat \mu_{x,h_x} - \hmu_{x, \hat{1}(x)}, \hmu_{\hat{a}, \hat{1}(\hat{a})} - \hmu_{x, \hat{1}(x)})
        \geq \hat \mu_{x,h_x} - \hmu_{x, \hat{1}(x)} \geq \hmu_{x,h_x} - \hmu_{x,1} > \frac{3}{4} \Delta_{\max}^{(k)} 
        \] 

        If $x$ is the empirically best subtree, \[
        \hat \Delta_{x,h_x} = \hat \mu_{x,h_x} - \hmu_{\hat{b}, \hat{1}(\hat{b})} \geq \hat \mu_{x,h_x} - \hmu_{x, \hat{1}(x)} \geq \hmu_{x,h_x} - \hmu_{x,1} > \frac{3}{4} \Delta_{\max}^{(k)}
        \]

        Hence $\hat \Delta_{\max} \geq \hat \Delta_{x,L} > \frac{3}{4} \Delta_{\max}^{(k)}$.

        \item If $\Delta_{\max}^{(k)}$ is given by $\mu_{a,h_a} - \mu_{b,1}$. In this case the condition that $\abs{\hmu - \mu} < \frac{\Delta_{\max}^{(k)}}{8}$ for all leaves ensures that $\hmu_{a,h_a} - \hmu_{x,1} > \frac{3}{4} \Delta_{\max}^{(k)}, \forall x\neq a$. 
        
        If $a$ is not the empirically best subtree, then there $\exists x\neq a$ such that $x$ is the empirically best subtree,
        \[
        \hat \Delta_{a,h_a} = \max(\hat \mu_{a,h_a} - \hmu_{a, \hat{1}(a)}, \hmu_{\hat{a}, \hat{1}(\hat{a})} - \hmu_{a, \hat{1}(a)})
        \geq \hat \mu_{a,h_a} - \hmu_{a, \hat{1}(a)} 
        \geq \hat \mu_{a,h_a} - \hmu_{x, \hat{1}(x)}
        \geq \hmu_{a,h_a} - \hmu_{x,1} > \frac{3}{4} \Delta_{\max}^{(k)} 
        \] 

        If $a$ is the empirically best subtree, \[
        \hat \Delta_{a,h_a} = \hat \mu_{a,h_a} - \hmu_{\hat{b}, \hat{1}(\hat{b})} \geq \hat \mu_{a,h_a} - \hmu_{\hat{b},1}  > \frac{3}{4} \Delta_{\max}^{(k)}
        \]

        Hence $\hat \Delta_{\max} \geq \hat \Delta_{a,h_a} > \frac{3}{4} \Delta_{\max}^{(k)}$.
        \item If $\Delta_{\max}^{(k)}$ is given by $\mu_{a,1} - \mu_{z,1}$. In this case the condition that $\abs{\hmu - \mu} < \frac{\Delta_{\max}^{(k)}}{8}$ for all leaves ensures that $\hmu_{a,j} - \hmu_{z,1} > \frac{3}{4} \Delta_{\max}^{(k)}, \forall j$ such that $(a,j) \in \cA^{(k)}$. Hence, $z$ is not the empirically best subtree, and 
        \begin{align*}
          \hat \Delta_{z,1} = \max(\hat \mu_{z,1} - \hmu_{z, \hat{1}(z)}, \hmu_{\hat{a}, \hat{1}(\hat{a})} - \hmu_{z, \hat{1}(z)})
          \geq \hmu_{\hat{a}, \hat{1}(\hat{a})} - \hmu_{z, \hat{1}(z)} 
          &\geq \hmu_{a, \hat{1}(a)} - \hmu_{z, \hat{1}(z)} 
          \\&\geq \hmu_{a, \hat{1}(a)} - \hmu_{z, 1} 
          \\&> \frac{3}{4} \Delta_{\max}^{(k)} 
        \end{align*}
    \end{enumerate}
    Therefore, in all the 3 cases, $\hat \Delta_{\max} > \frac{3}{4} \Delta_{\max}^{(k)} \geq \frac{3}{4} \Delta_{(KL+1-k)}$, where the last inequality is by Lemma~\ref{lem:max-gap-non-decreasing}.
\end{proof}

\begin{lemma}
Let $k \leq KL-m$. 
    Suppose after $k-1$ phases, the algorithm has not terminated, and phase $k$ is $\eps$-sound. If all leaves satisfy that $\abs{\hmu - \mu} < \frac{\Delta_{(KL+1-k)}}{8}$, then the optimal leaves in suboptimal subtrees, i.e., $\cbr{(y,1): y \neq a }$, will not be eliminated by the single leaf elimination rule. 
  \label{lem:optimal-leaf-in-suboptimal-subtree-epsilon}
\end{lemma}

\begin{proof}
By Lemma~\ref{lem:max-gap-non-decreasing}, all leaves satisfy that $\abs{\hmu - \mu} < \frac{\Delta_{(KL+1-k)}}{8} \leq \frac{\Delta_{\max}^{(k)}}{8}$. 

Let $y \neq a$, we will show that if $\abs{\hmu - \mu} < \frac{\Delta_{\max}^{(k)}}{8}$ for all leaves, then $(y,1)$ will not be eliminated by the single-leaf elimination rule. 

    \begin{enumerate}
        \item If $\hat \Delta_{y,1}$ is given by $\hmu_{\hat{a}, \hat{1}(\hat{a})} - \hmu_{y, \hat{1}(y)}$, 
        \begin{enumerate}
            \item If there exists a leaf $(y,j)$ such that $\hmu_{y,j} > \hmu_{\hat{a}, \hat{1}(\hat{a})}$, then \[
            \hat \Delta_{y,j} 
            = \max(\hat \mu_{y,j} - \hmu_{y, \hat{1}(y)}, \hmu_{\hat{a}, \hat{1}(\hat{a})} - \hmu_{y, \hat{1}(y)})
            = \hat \mu_{y,j} - \hmu_{y, \hat{1}(y)} 
            > \hmu_{\hat{a}, \hat{1}(\hat{a})} - \hmu_{y, \hat{1}(y)}
            = \hat \Delta_{y,1}
            \]
            hence $(y,1)$ will not be eliminated by the single-leaf elimination rule.
            \item If all leaves $(y,j)$'s are such that $\hmu_{y,j} \leq \hmu_{\hat{a}, \hat{1}(\hat{a})}$, then $\forall j$,  \[
            \hat \Delta_{y,j} 
            = \max(\hat \mu_{y,j} - \hmu_{y, \hat{1}(y)}, \hmu_{\hat{a}, \hat{1}(\hat{a})} - \hmu_{y, \hat{1}(y)})
            = \hmu_{\hat{a}, \hat{1}(\hat{a})} - \hmu_{y, \hat{1}(y)}
            = \hat \Delta_{y,1}
            \]
            hence either subtree $y$ will be eliminated entirely, or one leaf not in $y$ will be eliminated, or another subtree not $y$ will be eliminated. 
        \end{enumerate}
        
        \item If $\hat \Delta_{y,1}$ is given by $\hmu_{y,1} - \hmu_{y, \hat{1}(y)}$, then \[
        \hmu_{y,1} - \hmu_{y, \hat{1}(y)}
        = \hmu_{y,1} - \mu_{y,1} + \mu_{y,1} - \mu_{y, \hat{1}(y)} + \mu_{y, \hat{1}(y)} - \hmu_{y, \hat{1}(y)} 
        < \frac{\Delta_{\max}^{(k)}}{4}
        \]
        By Lemma~\ref{lem:Delta-hat-max}, $\hat \Delta_{\max} > \frac{3}{4} \Delta_{\max}^{(k)}$, hence $\hat \Delta_{y,1} < \hat \Delta_{\max}$, and $(y,1)$ will not be eliminated by the single-leaf elimination rule.

        \item If $y$ is the empirically best subtree and $\hat \Delta_{y,1}$ is given by $\hmu_{y,1} - \hmu_{\hat{b}, \hat{1}(\hat{b})}$, then 
        \[
        \hmu_{y,1} - \hmu_{\hat{b}, \hat{1}(\hat{b})}
        \leq \hmu_{y,1} - \hmu_{a, \hat{1}(a)}
        = \hmu_{y,1} - \mu_{y,1} + \mu_{y,1} - \mu_{a,\hat{1}(a)} + \mu_{a,\hat{1}(a)} - \hmu_{a, \hat{1}(a)}
        < \frac{\Delta_{\max}^{(k)}}{4}
        \]
        where the first inequality is because $y$ is the empirically best subtree, and hence $\hmu_{\hat{b}, \hat{1}(\hat{b})} \geq \hmu_{a, \hat{1}(a)}$. 
        % \[
        % \hmu_{y,1} - \hmu_{\hat{b}, \hat{1}(\hat{b})}
        % \leq \hmu_{y,1} - \hmu_{\hat{a}, \hat{1}(\hat{a})}
        % = \hmu_{y,1} - \mu_{y,1} + \mu_{y,1} - \mu_{a,1} + \mu_{a,1} - \hmu_{a,1} + \hmu_{a,1} - \hmu_{\hat{a}, \hat{1}(\hat{a})}
        % < \frac{\Delta_{(KL)}}{4}
        % \]
        
        By Lemma~\ref{lem:Delta-hat-max}, $\hat \Delta_{\max} > \frac{3}{4} \Delta_{\max}^{(k)}$, hence $\hat \Delta_{y,1} < \hat \Delta_{\max}$, and $(y,1)$ will not be eliminated by the single-leaf elimination rule.
        
    \end{enumerate}
\end{proof}

\begin{lemma}
Let $k \leq KL-m$. 
    Suppose after $k-1$ phases, the algorithm has not terminated, and phase $k$ is $\eps$-sound. If all leaves satisfy that $\abs{\hmu - \mu} < \frac{\Delta_{(KL+1-k)}}{8}$, then
    subtree $a$ will not become empty. 
    \label{lem:subtree-a-not-empty-epsilon}
\end{lemma}

\begin{proof}
By Lemma~\ref{lem:max-gap-non-decreasing}, all leaves satisfy that $\abs{\hmu - \mu} < \frac{\Delta_{(KL+1-k)}}{8} \leq \frac{\Delta_{\max}^{(k)}}{8}$. 

We note that it suffices to show that subtree $a$ will not be eliminated by the subtree elimination rule since the single-leaf elimination rule for a single-leaf subtree coincides with the subtree elimination rule. 

    We prove this lemma by contradiction in each of the following cases. 
    \begin{enumerate}
        \item If $\Delta_{\max}^{(k)}$ is given by $\mu_{x,h_x} - \mu_{x,1}$ for some $x \neq a$. 

        Since $\abs{\hmu - \mu} < \frac{\Delta_{\max}^{(k)}}{8}$ for all leaves, we have that $\hmu_{x,h_x} - \hmu_{x,1} > \frac{3}{4} \Delta_{\max}^{(k)}$. If it were the case that the subtree $x$ is the empirically best subtree, and subtree $a$ gets eliminated, then, for any $(a,j)$ such that $(a,j) \in \cA_a^{(k)}$, $\hat \Delta_{x,h_x} \leq \hat \Delta_{a,j}$. 
        % \yinan{$a,1$ can be any active leaf in $a$ in more generic phases}. 
        However, $\hat \Delta_{x,h_x} = \hat \mu_{x,h_x} - \hmu_{\hat{b}, \hat{1}(\hat{b})} \geq \hat \mu_{x,h_x} - \hmu_{x, \hat{1}(x)} \geq \hmu_{x,h_x} - \hmu_{x,1} > \frac{3}{4} \Delta_{\max}^{(k)}$, but $\hat \Delta_{a,j} = \hmu_{x, \hat{1}(x)} - \hat \mu_{a,j} \leq \hmu_{x, 1} - \hat \mu_{a,j} = 
        \hmu_{x, 1} - \mu_{x, 1} + \mu_{x,1} - \mu_{a,j}+ \mu_{a,j}- \hat \mu_{a,j} 
        < \frac{1}{4} \Delta_{\max}^{(k)}
        $. Contradiction.

        If it were the case that the subtree $y$ with $y \neq a, x$ is the empirically best subtree, and subtree $a$ gets eliminated, then, for any $(a,j)$ such that $(a,j) \in \cA_a^{(k)}$, $\hat \Delta_{x,h_x} \leq \hat \Delta_{a,j}$. However, $\hat \Delta_{x,h_x} \geq \hat \mu_{x,h_x} - \hmu_{x, \hat{1}(x)} \geq \hmu_{x,h_x} - \hmu_{x,1} > \frac{3}{4} \Delta_{\max}^{(k)}$, but $\hat \Delta_{a,j} = \hmu_{y, \hat{1}(y)} - \hat \mu_{a,j} \leq \hmu_{y, 1} - \hat \mu_{a,j} = 
        \hmu_{y, 1} - \mu_{y, 1} + \mu_{y,1} - \mu_{a,j}+ \mu_{a,j}- \hat \mu_{a,j} 
        < \frac{1}{4} \Delta_{\max}^{(k)}
        $. Contradiction.

        \item If $\Delta_{\max}^{(k)}$ is given by $\mu_{a,h_a} - \mu_{b,1}$. 

        In this case the condition that $\abs{\hmu - \mu} < \frac{\Delta_{\max}^{(k)}}{8}$ for all leaves ensures that $\hmu_{a,h_a} - \hmu_{x,1} > \frac{3}{4} \Delta_{\max}^{(k)}, \forall \text{ active } x\neq a$ . If subtree $a$ were eliminated, then there $\exists  \text{ active } x \neq a$ such that $x$ is the empirically best subtree and $\hmu_{x,1} \geq \hmu_{x,\hat{1}(x)} \geq \hmu_{a,h_a}$, which leads to a contradiction. 

        \item If $\Delta_{\max}^{(k)}$ is given by $\mu_{a,1} - \mu_{z,1}$. 

        In this case the condition that $\abs{\hmu - \mu} < \frac{\Delta_{\max}^{(k)}}{8}$ for all leaves ensures that $\hmu_{a,j} - \hmu_{z,1} > \frac{3}{4} \Delta_{\max}^{(k)}, \forall j$ such that $(a,j) \in \cA_a^{(k)}$. If subtree $a$ were eliminated, then there $\exists x \neq a$ such that $x$ is the empirically best subtree and 
        $\forall j \text{ such that $(a,j) \in \cA_a^{(k)}$ }, \hat \Delta_{a,j} = \hmu_{x,\hat{1}(x)} - \hmu_{a,j}  = \hmu_{x,\hat{1}(x)} - \mu_{x,\hat{1}(x)} + \mu_{x,\hat{1}(x)} - \mu_{a,j} + \mu_{a,j} - \hmu_{a,j} \leq \frac{1}{4} \Delta_{\max}^{(k)}$, 
        which leads to a contradiction with $\forall j \text{ such that $(a,j) \in \cA_a^{(k)}$ }, \hat \Delta_{a,j} = \hat \Delta_{\max}$, since $\hat \Delta_{\max} > \frac{3}{4} \Delta_{\max}^{(k)}$ by Lemma~\ref{lem:Delta-hat-max}. 
    \end{enumerate}
\end{proof}

\begin{lemma}
    Let $k \leq KL-m$. 
    Suppose after $k-1$ phases, the algorithm has not terminated, and phase $k$ is $\eps$-sound. Let subtree $s \neq a$ be $\eps/2$-good and surviving after $k-1$ phases. 
  If all leaves satisfy that $\abs{\hmu - \mu} < \frac{\Delta_{(KL+1-k)}}{8}$, then $(s,1)$ will not be eliminated in phase $k$. 
  \label{lem:very-good-subtree-early-phases}
\end{lemma}

\begin{proof}
By Lemma~\ref{lem:max-gap-non-decreasing}, all leaves satisfy that $\abs{\hmu - \mu} < \frac{\Delta_{(KL+1-k)}}{8} \leq \frac{\Delta_{\max}^{(k)}}{8}$. 

    In the proof of Lemma~\ref{lem:optimal-leaf-in-suboptimal-subtree-epsilon}, in cases 2 and 3, we've shown that $\hDelta_{s,1} < \frac{\Delta_{\max}^{(k)}}{4}$. We will now upper bound $\hDelta_{s,1}$ in case 1 where $\hat \Delta_{s,1}$ is given by $\hmu_{\hat{a}, \hat{1}(\hat{a})} - \hmu_{s, \hat{1}(s)}$. 

    \begin{align*}
        \hmu_{\hat{a}, \hat{1}(\hat{a})} - \hmu_{s, \hat{1}(s)}
        &\leq
        \hmu_{\hat{a},1} - \hmu_{s, \hat{1}(s)}
        \\&< 
        \mu_{\hat{a},1} + \frac{\Delta_{\max}^{(k)}}{8} - \mu_{s, \hat{1}(s)} + \frac{\Delta_{\max}^{(k)}}{8}
        \\&\leq 
        \mu_{a,1} - \mu_{s, \hat{1}(s)} + \frac{\Delta_{\max}^{(k)}}{4}
        \\&\leq 
        \mu_{a,1} - \mu_{s, 1} + \frac{\Delta_{\max}^{(k)}}{4}
        \\&\leq 
        \fr \eps 2 + \frac{\Delta_{\max}^{(k)}}{4}
        \tag{$s$ is $\eps/2$-good}
        \\&< 
        \frac{\Delta_{(KL+1-k)}}{2} + \frac{\Delta_{\max}^{(k)}}{4}
        \tag{Lemma~\ref{lem:critical-gap-range}: $\eps < \Dt^* \leq \Delta_{(KL+1-k)}$}
        \\&\leq
        \frac{3\Delta_{\max}^{(k)}}{4} 
        \tag{Lemma~\ref{lem:max-gap-non-decreasing}: $\Delta_{(KL+1-k)} \leq \Delta_{\max}^{(k)}$}
    \end{align*}
    By Lemma~\ref{lem:Delta-hat-max}, $\hat \Delta_{\max} > \frac{3}{4} \Delta_{\max}^{(k)}$, hence $\hat \Delta_{s,1} < \hat \Delta_{\max}$, and $(s,1)$ will not be eliminated
\end{proof}

\begin{lemma}
    Recall the definition of $\Dt^*$ and $m$ in Definition~\ref{def:eps-complexity}, we have, for any $k \in [KL-m]$, \[
    \eps < \Dt^* \leq \Delta_{(KL+1-k)}.
    \]
    \label{lem:critical-gap-range}
\end{lemma}

\begin{proof}
    The first inequality, $\eps < \Dt^*$, follows directly the definition of $\Dt^*$. 

    For the second inequality, 
    \begin{align*}
        \Delta_{(KL+1-k)} 
        &\geq \Delta_{(KL+1-(KL-m))}
        \tag{$\Dt_{(\ell)}$ is non-decreasing in $\ell$}
        \\&= \Delta_{(m+1)}
        \\&= \Dt^*
        \tag{Definition of $m$}
    \end{align*}
\end{proof}

\subsection{Analysis of a generic phase $k$ in late regime ($k > KL-m$)}

\begin{lemma}
    Let $k > KL-m$. Suppose after $k-1$ phases, the algorithm has not terminated, and phase $k$ is $\eps$-sound. Suppose $t$ is an active nonempty $\eps$-bad subtree at the beginning of phase $k$. If all leaves satisfy that $\abs{\hmu - \mu} < \frac{\Delta^*}{8}$, then, $(t,1)$ is not eliminated by the single leaf elimination rule by the end of phase $k$. 
    \label{lem:bad-subtree-late-phases}
    % \yinan{The statement holds for $\eps/2$-bad subtree. We only need the statement for $\eps$-bad subtree for the later inductive proof. }
\end{lemma}
\begin{proof}
First, by item~\ref{item:soundness-5} in the soundness definition (Definition~\ref{def:soundness}), $(t,1)$ is surviving at the beginning of phase $k$. 
    Next, we note that $t$ cannot be the empirically best subtree. This is because, by the item~\ref{item:soundness-4} of soundness, ``either all $\eps/2$-good subtrees are nonempty, or all $\eps$-bad subtrees are empty'', the existence of $t$ implies that ``all $\eps$-bad subtrees are empty'' cannot be true, hence it must be that ``all $\eps/2$-good subtrees are nonempty''. 
    % Specifically, subtree $a$ is not empty, and, for any suriving leaf $(a,j)$, 
    % \begin{align*}
    %     \hmu_{a, j} - \hmu_{t, \hat{1}(t)}
    %     \geq \mu_{a, j} - \frac{\Delta^*}{8} - \hmu_{t, 1}
    %     \geq \mu_{a, 1} - \mu_{t, 1} - \frac{\Delta^*}{4}
    %     \geq \Dt^* - \frac{\Delta^*}{4}
    %     = \frac{3\Delta^*}{4}
    % \end{align*}
    Specifically, let $s$ be an $\eps/2$-good subtree, and for any surviving leaf $(s,j)$, 
    \begin{align*}
        \hmu_{s, j} - \hmu_{t, \hat{1}(t)}
        &\geq \mu_{s, j} - \frac{\Delta^*}{8} - \hmu_{t, 1}
        \\&\geq \mu_{s, 1} - \mu_{t, 1} - \frac{\Delta^*}{4}
        \\&= (\mu_{a, 1}- \mu_{t, 1}) - (\mu_{a, 1}- \mu_{s, 1} ) -\frac{\Delta^*}{4}
        \\&\geq \Dt^* - \eps/2 - \frac{\Delta^*}{4}
        \tag{$\mu_{a, 1}- \mu_{t, 1} \geq \Dt^*, ~\mu_{a, 1}- \mu_{s, 1} \leq \eps/2$}
        \\&> \frac{\Delta^*}{4}
        \tag{Lemma~\ref{lem:critical-gap-range}: $\Dt^* > \eps$}
    \end{align*}
    Hence, $t$ is not the empirically best subtree, and $\hat \Delta_{t,1}$ is given by $\hmu_{\hat{a}, \hat{1}(\hat{a})} - \hmu_{t, \hat{1}(t)}$.

    By the same argument as in case 1 in Lemma~\ref{lem:optimal-leaf-in-suboptimal-subtree-phase-1-improved}'s proof, we have that $(t,1)$ is not eliminated by the single leaf elimination rule. 
\end{proof}

\begin{lemma}
    Let $k > KL-m$. Suppose after $k-1$ phases, the algorithm has not terminated, and phase $k$ is $\eps$-sound. Suppose $t$ is an active nonempty $\eps$-bad subtree and $s$ is an active nonempty $\eps/2$-good subtree, at the beginning of phase $k$. If all leaves satisfy that $\abs{\hmu - \mu} < \frac{\Delta^*}{8}$, then, $s$ does not become empty in phase $k$.  
    \label{lem:very-good-subtree-late-phases}
\end{lemma}
\begin{proof}
    We note that it suffices to show that subtree $s$ will not be eliminated by the subtree elimination rule since the single-leaf elimination rule for a single-leaf subtree coincides with the subtree elimination rule, which further implies that we only need to consider the case where $s$ is not the empirically best. 

    By the reasoning in the previous lemma's proof, $(t,1)$ is surviving at the beginning of phase $k$. $t$ is not the empirically best subtree, and $\hat \Delta_{t,1}$ is given by $\hmu_{\hat{a}, \hat{1}(\hat{a})} - \hmu_{t, \hat{1}(t)}$. 

    As a result, to show $s$ does not become empty in phase $k$, it suffices to show $\hmu_{\hat{a}, \hat{1}(\hat{a})} - \hmu_{t, \hat{1}(t)} > \hmu_{\hat{a}, \hat{1}(\hat{a})} - \hmu_{s, \hat{1}(s)}$, that is $\hmu_{t, \hat{1}(t)} < \hmu_{s, \hat{1}(s)}$. Indeed,
    \begin{align*}
        \hmu_{s, \hat{1}(s)} - \hmu_{t, \hat{1}(t)}
        &\geq \mu_{s, \hat{1}(s)} - \frac{\Delta^*}{8} - \hmu_{t, 1} 
        \\&\geq \mu_{s, 1} - \mu_{t, 1} -\frac{\Delta^*}{4}
        \\&= (\mu_{a, 1}- \mu_{t, 1}) - (\mu_{a, 1}- \mu_{s, 1} ) -\frac{\Delta^*}{4}
        \\&\geq \Dt^* - \eps/2 - \frac{\Delta^*}{4}
        \tag{$\mu_{a, 1}- \mu_{t, 1} \geq \Dt^*, ~\mu_{a, 1}- \mu_{s, 1} \leq \eps/2$}
        \\&> \frac{\Delta^*}{4}
        \tag{Lemma~\ref{lem:critical-gap-range}: $\Dt^* > \eps$}
    \end{align*}

    % \begin{align*}
    %     \hat \Delta_{t,1} = 
    %     \hmu_{\hat{a}, \hat{1}(\hat{a})} - \hmu_{t, \hat{1}(t)} 
    %     &\geq \hmu_{a, \hat{1}(a)} - \hmu_{t, \hat{1}(t)}
    %     \\&> \mu_{a, \hat{1}(a)} -\frac{\Delta^*}{8} - \hmu_{t, \hat{1}(t)}
    %     \\&\geq \mu_{a, 1} -\frac{\Delta^*}{8} - \hmu_{t, 1}
    %     \\&> \mu_{a, 1} - \mu_{t, 1} -\frac{\Delta^*}{4}
    %     \\&\geq \frac{3\Delta^*}{4}
    %     \tag{$\mu_{a, 1} - \mu_{t, 1} \geq \Dt^*$}
    % \end{align*}

    % On the other hand, 
    % \begin{align*}
    %     \hat \Delta_{s,1} =
    %     \hmu_{\hat{a}, \hat{1}(\hat{a})} - \hmu_{s, \hat{1}(s)}
    %     &\leq \hmu_{\hat{a}, 1} - \hmu_{s, \hat{1}(s)}
    %     \\&< \mu_{\hat{a}, 1} + \frac{\Delta^*}{8} - \mu_{s, \hat{1}(s)} + \frac{\Delta^*}{8}
    %     \\&\leq \mu_{a, 1} - \mu_{s, 1} + \frac{\Delta^*}{4}
    %     \\&\leq \fr \eps 2 + \frac{\Delta^*}{4}
    %     \tag{$s$ is $\eps/2$-good}
    %     \\&< \frac{3\Delta^*}{4}
    %     \tag{$\Dt^* > \eps$}
    % \end{align*}
    Therefore, $s$ will not be eliminated by the subtree elimination rule, and $s$ will not become empty in phase $k$. 
\end{proof}

\subsection{Proof of Theorem~\ref{thm:main-upper-bound}}
\begin{lemma}
    A sufficient event for the SR algorithm to succeed in finding an $\eps$-good subtree is (without the subscript, $\hmu,  \mu$ refer to all leaves)
    \[
    A_{1} := 
    \begin{cases}
        \forall k \in [KL-m], \abs{\hat{\mu}^{(k)} -  \mu }< \frac{\Delta_{(KL+1-k)}}{8} \\
         \forall k \in [KL-m+1, KL-1], \abs{\hat{\mu}^{(k)} -  \mu }< \frac{\Delta_{(m+1)}}{8} 
    \end{cases}
    \]
    \label{lem:sufficient-condition-eps-subtree}
\end{lemma}

\begin{proof}
    Note that not all phases are executed. 
    We first show by induction that under event $A_1$, for any phase $k$,  phase $k$ is sound. 
    \begin{itemize}
        \item $k \leq KL-m$. 
        \begin{itemize}
            \item Base case: if $k=1$, phase $1$ is apparently sound since nothing has been operated on any leaves. 
            \item Induction case:
            Assume phase $k-1$ is sound. 
            \begin{enumerate}
                \item If phase $k-1$ is not executed, then phase $k$ is automatically sound since no operation is done in phase $k-1$.
                \item If phase $k-1$ is executed, then by Lemma~\ref{lem:optimal-leaf-in-suboptimal-subtree-epsilon}, Lemma~\ref{lem:subtree-a-not-empty-epsilon} and Lemma~\ref{lem:very-good-subtree-early-phases}, at phase $k-1$, the optimal(min) leaf in any suboptimal subtree is not eliminated alone, the optimal subtree does not become empty, and $\forall s \neq 1$ and $\eps/2$-good, $(s,1)$ is not eliminated. It follows that phase $k$ is sound, completing the induction for $k \leq KL-m$.  
            \end{enumerate}
        \end{itemize}
        \item $k > KL-m$. 
            \begin{itemize}
            \item Base case: $k=KL-m+1$. 
            \begin{enumerate}
                \item If phase $KL-m$ is not executed, then by the induction on phases $k \leq KL-m$, phase $KL-m$ is sound. 
                Note that if a tree satisfies the soundness condition (Definition~\ref{def:soundness}) for the earlier phases $k \leq KL-m$, then it also satisfies the soundness condition for the later phases $k > KL-m +1$. 
                It follows that phase $KL-m+1$ is automatically sound since no operation is done in phase $KL-m$.
                \item If phase $KL-m$ is executed, then by Lemma~\ref{lem:optimal-leaf-in-suboptimal-subtree-epsilon}, Lemma~\ref{lem:subtree-a-not-empty-epsilon} and Lemma~\ref{lem:very-good-subtree-early-phases}, at phase $KL-m$, the optimal leaf in any suboptimal subtree is not eliminated alone, the optimal subtree does not become empty, and $\forall s \neq 1$ and $\eps/2$-good, $(s,1)$ is not eliminated. It follows that phase $KL-m+1$ is sound.  
            \end{enumerate}
            
            \item Induction case:
            Assume phase $k-1 \in [KL-m+1, KL-1]$ is sound. 
            \begin{enumerate}
                \item If phase $k-1$ is not executed, then phase $k$ is automatically sound since no operation is done in phase $k-1$.
                \item If phase $k-1$ is executed, then by Lemma~\ref{lem:bad-subtree-late-phases} and Lemma~\ref{lem:very-good-subtree-late-phases}, at phase $k-1$, if there exists nonempty $\eps$-bad subtree, then the optimal leaf in any $\eps$-bad subtree is not eliminated alone, and any $\eps/2$-good subtree does not become empty. It follows that phase $k+1$ is sound, completing the induction for $k > KL-m$. 
            \end{enumerate}
        \end{itemize}
    \end{itemize}
    Therefore, for any phase $k$, phase $k$ is sound.  
 Suppose the algorithm terminates after exactly $k'$ phases, so the termination situation satisfies the soundness condition in phase $k'+1$. 
 \begin{itemize}
     \item If $k'+1 \leq KL-m$, then by the first item of the soundness definition for phase $k \leq KL-m$, the best subtree is not empty. In conjunction with that the algorithm terminates when there is only one surviving subtree, we conclude that the algorithm succeeds in finding an $\eps$-good subtree. (Actually, it outputs the best subtree. )
     \item If $k'+1 > KL-m$, then by the first item of the soundness definition for phase $k > KL-m$, ``either all $\eps/2$-good subtrees are nonempty, or all $\eps$-bad subtrees are empty''. 
     \begin{itemize}
        \item If the number of $\eps/2$-good subtrees is 1, then either all $\eps$-bad subtrees are empty, or the $\eps/2$-good subtree is nonempty. In both cases, the only surviving subtree is $\eps$-good. 
        \item Else if the number of $\eps/2$-good subtrees is $>1$, then since ``all $\eps/2$-good subtrees are nonempty'' cannot be true, it must be that all $\eps$-bad subtrees are empty, which implies that the only surviving subtree is $\eps$-good. 
    \end{itemize}
 \end{itemize}
    
\end{proof}

\begin{theorem}[Restatement of Theorem~\ref{thm:main-upper-bound}]
    Let $K,L$ satisfy $KL\ge 2$, and suppose $T>KL$. Run
Algorithm~\ref{alg:sr-for-mcts} with the phase schedule
\[
    n_k
    =
    \left\lceil
    \frac{(T-KL)/\overline{\log}(KL)}{KL+1-k}
    \right\rceil,
    \qquad
    \overline{\log}(KL)
    :=
    \frac12+\sum_{r=2}^{KL}\frac1r .
\]
Then, for every $\varepsilon\ge 0$ such that not all subtrees are
$\varepsilon$-good, the failure probability of Successive Rejects for
$(K,L)$-MCTS satisfies
\[
    \mathbb{P}\!\left(\widehat{i}_T \notin G_\varepsilon\right)
    \le
    2K^2L^2
    \exp\!\left(
        -\frac{T-KL}
        {128\,\overline{\log}(KL)\,H_2(\varepsilon)}
    \right).
\]
If all subtrees are $\varepsilon$-good, then
$\mathbb{P}(\widehat{i}_T \notin G_\varepsilon)=0$.
\end{theorem}

% \begin{theorem}
%     The failure probability of the algorithm Successive Rejects for $(K,L)$-MCTS in finding an $\eps$-good subtree is at most \[
%     2K^2L^2
%     \exp\!\left(
%         -\frac{T-KL}
%         {128\,\overline{\log}(KL)\,H_2(\varepsilon)}
%     \right).
%     \] 
%     % \[
%     % 2K^2 L^2 \exp(- \frac{1}{2} \frac{T / \ln (KL)}{\max_{i \geq m+1} i \Delta_{(i)}^{-2}} \frac{1}{64})
%     % \]
% \end{theorem}

\begin{proof}
If all subtrees are $\varepsilon$-good, then $G_\varepsilon=[K]$, and hence
\[
    \mathbb{P}\!\left(\widehat{i}_T \notin G_\varepsilon\right)=0.
\]
Thus the theorem is trivial in this case. In the rest of the proof, we assume
that not all subtrees are $\varepsilon$-good, so that
$G_\varepsilon \neq [K]$ and the critical gap $\Delta^\star$ and the index $m$
in Definition~\ref{def:eps-complexity} are well defined.

We first show that our budget allocation $n_k$ is feasible, i.e., the total number of samples does not exceed the budget $T$. 
\paragraph{Budget feasibility.}
Let $N:=KL$ and define
\[
    \overline{\log}(N)
    :=
    \frac12+\sum_{r=2}^{N}\frac1r .
\]
We use the phase schedule
\[
    n_k
    =
    \left\lceil
    \frac{(T-N)/\overline{\log}(N)}{N+1-k}
    \right\rceil,
    \qquad k=1,\ldots,N-1 .
\]
The worst case for the budget is when exactly one leaf is eliminated in each
phase. In that case, phase $k$ has $N+1-k$ active leaves, and the total number of
pulls is
\[
    B
    =
    N n_1
    +
    \sum_{k=2}^{N-1}
    (N+1-k)(n_k-n_{k-1}).
\]
By summation by parts,
\[
    B
    =
    \sum_{k=1}^{N-2} n_k + 2n_{N-1}.
\]
Let
\[
    a:=\frac{T-N}{\overline{\log}(N)}.
\]
Since $\lceil x\rceil \le x+1$,
\[
\begin{aligned}
    B
    &\le
    \sum_{k=1}^{N-2}
    \left(
        \frac{a}{N+1-k}+1
    \right)
    +
    2\left(
        \frac{a}{2}+1
    \right) \\
    &=
    a
    \left(
        \sum_{k=1}^{N-2}\frac{1}{N+1-k}
        +
        1
    \right)
    +
    N \\
    &=
    a
    \left(
        \sum_{r=3}^{N}\frac1r + 1
    \right)
    +
    N \\
    &=
    a
    \left(
        \frac12+\sum_{r=2}^{N}\frac1r
    \right)
    +
    N \\
    &=
    a\,\overline{\log}(N)+N \\
    &= T .
\end{aligned}
\]
Thus the total number of samples is at most $T$ in the single-elimination case.

If a phase eliminates more than one leaf, the algorithm skips some intermediate
single-elimination phases. This cannot increase the budget. Indeed, if the phase
index jumps from $a$ to $b>a$, then the actual additional cost is
\[
    (N+1-b)(n_b-n_a),
\]
whereas eliminating the same leaves one at a time would cost
\[
    \sum_{\ell=a+1}^{b}
    (N+1-\ell)(n_\ell-n_{\ell-1})
    \ge
    (N+1-b)\sum_{\ell=a+1}^{b}(n_\ell-n_{\ell-1})
    =
    (N+1-b)(n_b-n_a).
\]
Hence subtree eliminations cannot increase the total budget, and the algorithm
uses at most $T$ samples.

\paragraph{One-phase concentration.}
By the definition of $n_k$,
\[
    n_k
    \ge
    \frac{a}{N+1-k}.
\]
Therefore, by the sub-Gaussian concentration inequality,
\[
    \mathbb{P}\!\left(
        |\widehat{\mu}^{(k)}_i-\mu_i^{(k)}|
        \ge
        \frac{\Delta}{8}
    \right)
    \le
    2\exp\!\left(
        -\frac{n_k\Delta^2}{128}
    \right)
    \le
    2\exp\!\left(
        -\frac{a\Delta^2}{128(N+1-k)}
    \right).
\]
\paragraph{Union bound over phases.}
% Recall the event $A_1$ defined in Lemma~\ref{lem:sufficient-condition-eps-subtree}. 

% Lemma~\ref{lem:sufficient-condition-eps-subtree} shows that the event of failing to find an $\eps$-good subtree implies $\bar{A}_1$ (the complement of event $A_1$). So the probability of failing to find an $\eps$-good subtree is upper bound by $\bar{A}_1$. 

% The probability of $\bar{A}_1$ is upper bounded as: 

% Recall the event $A_1$ defined in
% Lemma~\ref{lem:sufficient-condition-eps-subtree}. By that lemma, the event that
% the algorithm fails to return an $\varepsilon$-good subtree is contained in
% $\overline A_1$. Hence it remains to bound $\mathbb{P}(\overline A_1)$.

Recall the event $A_1$ defined in Lemma~\ref{lem:sufficient-condition-eps-subtree}.
By that lemma, $A_1$ is sufficient for success, i.e.,
\[
    A_1 \subseteq \{\widehat i_T \in G_\varepsilon\}.
\]
Therefore,
\[
    \{\widehat i_T \notin G_\varepsilon\}
    \subseteq \overline A_1.
\]
It remains to bound $\mathbb P(\overline A_1)$.
Using the one-phase concentration bound above and taking a union bound over
leaves and phases,

\begin{align*}
\mathbb{P}(\overline A_1)
&\le
\sum_{k=1}^{N-m}
\sum_{i=1}^{N}
\mathbb{P}\!\left(
    |\widehat{\mu}^{(k)}_i-\mu_i^{(k)}|
    \ge
    \frac{\Delta_{(N-k+1)}}{8}
\right) \\
&\quad+
\sum_{k=N-m+1}^{N-1}
\sum_{i=1}^{N}
\mathbb{P}\!\left(
    |\widehat{\mu}^{(k)}_i-\mu_i^{(k)}|
    \ge
    \frac{\Delta_{(m+1)}}{8}
\right) 
\tag{Union bound} 
\\
&\le
\sum_{k=1}^{N-m}
2N
\exp\!\left(
    -\frac{a\,\Delta_{(N-k+1)}^2}{128(N-k+1)}
\right) \\
&\quad+
\sum_{k=N-m+1}^{N-1}
2N
\exp\!\left(
    -\frac{a\,\Delta_{(m+1)}^2}{128(m+1)}
\right) 
\tag{Hoeffding's inequality}
\end{align*}
By the definition
\[
    H_2(\varepsilon)
    :=
    \max_{r\ge m+1} r\Delta_{(r)}^{-2},
\]
each exponent above is at most
\[
    -\frac{a}{128H_2(\varepsilon)}.
\]
Since there are at most $N$ phases and at most $N$ leaves in each phase,
\[
\begin{aligned}
\mathbb{P}(\overline A_1)
&\le
2N^2
\exp\!\left(
    -\frac{a}{128H_2(\varepsilon)}
\right) \\
&=
2K^2L^2
\exp\!\left(
    -\frac{T-KL}
    {128\,\overline{\log}(KL)\,H_2(\varepsilon)}
\right).
\end{aligned}
\]
% This proves the claimed bound.

This proves the stated error bound with complexity
$H_2(\varepsilon)=\max_{r\ge m+1} r\Delta_{(r)}^{-2}$.

    %     \begin{align*}
    %     \PP(\bar{A}_1) 
    %     \leq& \sum_{k=1}^{KL-m} 
    %     \sbr{
    %     \sum_{i=1}^{KL} \PP\del{\abs{\hat{\mu}_i^{(k)} -  \mu_i^{(k)} }\geq \frac{\Delta_{(K-k+1)}}{8}} 
    %     } 
    %     \\&+ \sum_{k=KL-m+1}^{K-1} 
    %     \sbr{
    %     \sum_{i=1}^{KL} \PP\del{\abs{\hat{\mu}_i^{(k)} -  \mu_i^{(k)} }\geq \frac{\Delta_{(m+1)}}{8}} 
    %     }
    %     \tag{Union bound}
    %     \\
    %     \leq& \sum_{k=1}^{KL-m}
    %     \sbr{
    %     KL \cd 2 \exp(- \frac{1}{2} \frac{T / \ln (KL)}{KL-k+1} \frac{\Delta_{(KL-k+1)}^2}{64}) 
    %     } 
    %     \\&+ \sum_{k=KL-m+1}^{KL-1} 
    %     \sbr{KL \cd 2 \exp(- \frac{1}{2} \frac{T / \ln (KL)}{KL-k+1} \frac{\Delta_{(m+1)}^2}{64}) 
    %     } 
    %     \tag{Hoeffding's inequality}
    %     \\
    %     \leq& \sum_{k=1}^{KL-m}
    %     \sbr{
    %     2KL \exp(- \frac{1}{2} \frac{T / \ln (KL)}{K-k+1} \frac{\Delta_{(K-k+1)}^2}{64})
    %     } 
    %     \\&+ \sum_{k=KL-m+1}^{KL-1} 
    %     \sbr{ 2KL  \exp(- \frac{1}{2} \frac{T / \ln (KL)}{ m+1 } \frac{\Delta_{(m+1)} ^2}{64}) 
    %     } \\
    %     \leq&
    %     2 K^2L^2 \exp(- \frac{1}{2} \frac{T / \ln (KL)}{\max_{i \geq m+1} i \Delta_{(i)}^{-2}} \frac{1}{64})
    % \end{align*}
    % which implies a sample complexity $\blue{H_2(\eps)} := \max_{i \geq m+1} i \Delta_{(i)}^{-2}$. 
\end{proof}

\section{Proof of Lemma~\ref{lem:H2-H1}}
\label{sec:proof-lemma-H2-H1}
\begin{proof}
    Let $i$ be such that $i \geq m+1$. For any $j \leq i$, $(\Delta_{(j)} \vee \eps) \leq \Delta_{(i)}$, since $\Delta_{(j)} \leq \Delta_{(i)}$ and $ \eps \leq \Delta_{(i)}$. Thus, for all $j \leq i$, $( \Delta_{(j)} \vee \eps) ^{-2} \geq \Delta_{(i)}^{-2}$. Summing up, we have, 
    \begin{align*}
        i \Delta_{(i)}^{-2} \leq \sum_{j = 1}^i ( \Delta_{(j)} \vee \eps) ^{-2} \leq \sum_{j = 1}^{KL} ( \Delta_{(j)} \vee \eps) ^{-2}
    \end{align*}
    Since this holds for any $i \geq m+1$, we can take the maximum over $i \geq m+1$: 
    \begin{align*}
        \max_{i \geq m+1} i \Delta_{(i)}^{-2} \leq \sum_{j = 1}^{KL} ( \Delta_{(j)} \vee \eps) ^{-2} 
    \end{align*}
\end{proof}

\section{Proof of Theorem~\ref{thm:main-lower-bound}}

\subsection{Warmup: A fixed-budget lower bound for BAI with Gaussian arms}\label{sec:lb-bai-gaussian}

We present a version of a fixed-budget lower bound for best arm identification with flipping construction using Gaussian arms with unit variance.
The construction follows the spirit of~\citet{carpentier2016tight}, while the proof below is new and self-contained. \yinan{mention our proof achieves much better constant on $H$?}

\paragraph{Instance family.}
Fix $K\ge 2$ and parameters $(d_i)_{i=2}^K$ with $d_i>0$. Define $K$ instances $\{\nu_i\}_{i=1}^K$
over $K$ arms, where each arm yields i.i.d.\ rewards distributed as $\cN(\mu_k,1)$.

Base instance $\nu_1$:
\[
\mu_1=\tfrac12,\qquad \mu_i=\tfrac12-d_i\quad (i=2,\ldots,K).
\]
For each $i\in\{2,\ldots,K\}$, alternative instance $\nu_i$:
\[
\mu_i=\tfrac12+d_i,\qquad \mu_1=\tfrac12,\qquad
\mu_k=\tfrac12-d_k\quad (k\in\{2,\ldots,K\}\setminus\{i\}).
\]
Thus $\nu_1$ has unique best arm $1$, whereas $\nu_i$ has unique best arm $i$.

% Let
% \begin{equation}\label{eq:H-gauss-def}
% H \;:=\;\sum_{i=2}^K \frac{1}{d_i^2}.
% \end{equation}
As in the standard BAI setting, define the instance-dependent complexity
\[
H(\nu)\;:=\;\sum_{k\neq k^\star(\nu)} \frac{1}{\Delta_k(\nu)^2},
\qquad
\Delta_k(\nu):=\mu_{k^\star(\nu)}-\mu_k,
\]
where $k^\star(\nu)$ is the unique best arm.
Then for every $i\in[K]$,
\begin{equation}\label{eq:Hgauss-uniform}
H(\nu_i)\;\le\;H(\nu_1),
\end{equation}
since for $i=1$ the gaps are $\Delta_k(\nu_1)=d_k$ ($k\ge 2$), and for $i\ge 2$ the gaps satisfy
$\Delta_1(\nu_i)=d_i$ and $\Delta_k(\nu_i)=d_i+d_k\ge d_k$ for $k\neq 1,i$.

Define the bounded-complexity class $\mathcal{P}_H:=\{\nu:\ H(\nu)\le H\}$.

\begin{theorem}[Fixed-budget lower bound for Gaussian BAI]\label{thm:lb-bai-gaussian}
Consider the instance $\nu_1$ as described above and write $H := H(\nu_1)$. 
Let $\mathcal{A}$ be any (possibly adaptive) fixed-budget best-arm identification algorithm that
collects exactly $T$ samples and outputs a recommendation $J_T\in[K]$.
There exists an instance $\nu\in\mathcal{P}_{H}$ such that
\[
\mathbb{P}_{\nu}\bigl(J_T\neq k^\star(\nu)\bigr)
\;\ge\;
\frac14 \exp\!\left(-\frac{2T}{H}\right). 
\]
% where $H$ is given by~\eqref{eq:H-gauss-def}.
\end{theorem}

\begin{proof}
Fix the family $\{\nu_i\}_{i=1}^K$ above. By Eq.~\eqref{eq:Hgauss-uniform}, each $\nu_i\in\mathcal{P}_{H}$.

Let $\mathbb{P}_{\nu}$ denote the distribution of the full interaction transcript
generated by running $\mathcal{A}$ on instance $\nu$.
Let $N_k$ be the (random) number of pulls of arm $k$ up to time $T$, and write $\mathbb{E}_1$ for
expectation under $\nu_1$.

\paragraph{Step 1: KL decomposition.}
For any $i\in\{2,\ldots,K\}$, the instances $\nu_1$ and $\nu_i$ differ only in the mean of arm $i$,
namely $\tfrac12-d_i$ versus $\tfrac12+d_i$.
A standard divergence decomposition (chain rule for KL under adaptive sampling) yields
\begin{equation}\label{eq:KL-decomp-gauss}
\mathrm{KL}\!\left(\mathbb{P}_{\nu_1},\mathbb{P}_{\nu_i}\right)
=
\mathbb{E}_1[N_i]\cdot \mathrm{KL}\!\left(\cN(\tfrac12-d_i,1),\cN(\tfrac12+d_i,1)\right).
\end{equation}
For unit-variance Gaussians,
\[
\mathrm{KL}\!\left(\cN(m,1),\cN(m',1)\right)=\frac{(m-m')^2}{2}.
\]
Here $(m-m')^2=(2d_i)^2=4d_i^2$, hence
\begin{equation}\label{eq:KL-gauss-arm}
\mathrm{KL}\!\left(\cN(\tfrac12-d_i,1),\cN(\tfrac12+d_i,1)\right)=2d_i^2.
\end{equation}
Combining~\eqref{eq:KL-decomp-gauss} and~\eqref{eq:KL-gauss-arm} gives
\begin{equation}\label{eq:KL-upper-gauss}
\mathrm{KL}\!\left(\mathbb{P}_{\nu_1},\mathbb{P}_{\nu_i}\right)
=2\,\mathbb{E}_1[N_i]\,d_i^2.
\end{equation}

\paragraph{Step 2: Two-point testing inequality.}
We use the Bretagnolle--Huber inequality: for any distributions $P,Q$ and any event $E$,
\begin{equation}\label{eq:BH-gauss}
P(E^c)+Q(E)\;\ge\;\frac12 \exp\!\bigl(-\mathrm{KL}(P,Q)\bigr).
\end{equation}
Apply~\eqref{eq:BH-gauss} with $P=\mathbb{P}_{\nu_1}$, $Q=\mathbb{P}_{\nu_i}$, and $E=\{J_T=1\}$.
Since arm $1$ is optimal under $\nu_1$ and suboptimal under $\nu_i$,
\[
\mathbb{P}_{\nu_1}(E^c)=\mathbb{P}_{\nu_1}(J_T\neq 1)=\mathbb{P}_{\nu_1}(\mathrm{error}),
\qquad
\mathbb{P}_{\nu_i}(E)=\mathbb{P}_{\nu_i}(J_T=1)\le \mathbb{P}_{\nu_i}(\mathrm{error}).
\]
Therefore,
\begin{equation*}
\mathbb{P}_{\nu_1}(\mathrm{error})+\mathbb{P}_{\nu_i}(\mathrm{error})
\;\ge\;
\frac12 \exp\!\left(-\mathrm{KL}\!\left(\mathbb{P}_{\nu_1},\mathbb{P}_{\nu_i}\right)\right).
\end{equation*}

\paragraph{Step 3: Choose an index $i$ with small $\mathbb{E}_1[N_i]d_i^2$.}
We claim there exists $i\in\{2,\ldots,K\}$ such that
\begin{equation}\label{eq:pigeon-gauss}
\mathbb{E}_1[N_i]\,d_i^2 \;\le\; \frac{T}{H}.
\end{equation}
Indeed, if~\eqref{eq:pigeon-gauss} failed for all $i\ge 2$, then
\[
\sum_{i=2}^K \mathbb{E}_1[N_i]
\;>\;
\sum_{i=2}^K \frac{T}{H}\cdot \frac{1}{d_i^2}
\;=\;\frac{T}{H}\sum_{i=2}^K \frac{1}{d_i^2}
\;=\;\frac{T}{H}\cdot H
\;=\;T,
\]
contradicting $\sum_{i=2}^K \mathbb{E}_1[N_i]\le \sum_{k=1}^K \mathbb{E}_1[N_k]=T$.

Fix such an $i$. Then by~\eqref{eq:KL-upper-gauss} and~\eqref{eq:pigeon-gauss},
\[
\mathrm{KL}\!\left(\mathbb{P}_{\nu_1},\mathbb{P}_{\nu_i}\right)
=2\,\mathbb{E}_1[N_i]\,d_i^2
\;\le\;\frac{2T}{H}.
\]
Plugging into~\eqref{eq:sum-error-lb-gauss} yields
\[
\mathbb{P}_{\nu_1}(\mathrm{error})+\mathbb{P}_{\nu_i}(\mathrm{error})
\;\ge\;
\frac12 \exp\!\left(-\frac{2T}{H}\right).
\]
Hence at least one of the two terms is at least half of the right-hand side:
\[
\max\Bigl\{\mathbb{P}_{\nu_1}(\mathrm{error}),\ \mathbb{P}_{\nu_i}(\mathrm{error})\Bigr\}
\;\ge\;
\frac14 \exp\!\left(-\frac{2T}{H}\right).
\]
\end{proof}

\subsection{Proof of Theorem~\ref{thm:main-lower-bound}}
\label{sec:appendix-lower-bound-proof}
\begin{theorem}[Lower bound for maximin trees, restatement of Theorem~\ref{thm:main-lower-bound}]
Let $\nu$ be a maximin-tree instance with a unique optimal subtree (we call it the base instance), and assume its reward distributions on all leaves are Gaussian with variance 1. Consider the instance class $ \mathcal{P}_{4 H(\nu)}$.
For any fixed-budget algorithm $\pi$ that collects at most $T$ samples in total, there exists an instance
$\nu' \in \mathcal{P}_{4 H(\nu)}$ such that
\[
\mathbb{P}_{\nu'}\!\left(\mathrm{error}\right)
\;\ge\;
\fr 14 \exp\!\left(- 2 \frac{T}{\Hlb(\nu)}\right)
\]
% \[
% \mathbb{P}_{\nu'}\!\left(\mathrm{error}\right)
% \;\ge\;
% c_1 \exp\!\left(-c_2 \frac{T}{\Hlb(\nu)}\right),
% \]
where $\mathbb{P}_{\nu'}\!\left(\mathrm{error}\right)$ is the event that $\pi$ outputs an incorrect optimal subtree on instance $\nu'$. 
% , and $c_1,c_2>0$
% are universal constants (independent of $K,L,T$ and $\nu$).
\end{theorem}

\begin{proof}
We first construct an instance family. 
\paragraph{Instance family.}
For the base instance $\nu$, since its reward distributions on all leaves are Gaussian with variance 1, 
$\nu$ can be specified by its mean matrix $\mu=(\mu_{a,b})_{a\in[K],\,b\in[L]}$.

Define the index set of ``critical'' leaves
\[
I \;=\; \{(i,1): i\neq 1\}\ \cup\ \{(1,j): j\neq 1\}.
\]
For each $(i,j)\in I$, we construct an alternative instance
$\nu^{i,j}$ whose reward distributions are still Gaussian with variance $1$, but whose means differ
from $\nu$ at a subset of leaves.

For any $(i,1)$ with $i\neq 1$, define $\nu^{i,1}$ by the mean matrix
$\mu^{i,1}=(\mu^{i,1}_{a,b})_{a,b}$:
\[
\mu^{i,1}_{a,b}
=
\begin{cases}
\mu_{a,b} + 2\Delta_{i,1}, & a=i,\\
\mu_{a,b}, & \text{otherwise}.
\end{cases}
\]
Equivalently, $\nu^{i,1}$ is identical to $\nu$ except that leaves $\{(i,j): i \in [L]\}$ are shifted upward by
$2\Delta_{i,1}$.

For any $(1,j)$ with $j\neq 1$, define $\nu^{1,j}$ by
\[
\mu^{1,j}_{a,b}
=
\begin{cases}
\mu_{a,b} - 2\Delta_{1,j}, & (a,b)=(1,j),\\
\mu_{a,b}, & \text{otherwise}.
\end{cases}
\]
Equivalently, $\nu^{1,j}$ is identical to $\nu$ except that leaf $(1,j)$ is shifted downward by
$2\Delta_{1,j}$.

\smallskip
\noindent
% Collect these alternatives and the base instance $\nu$ as
% \[
% \mathcal{V}(\nu) \;:=\; \{\nu^{i,j} : (i,j)\in I\} \cup \{\nu\}.
% \]
Collect these alternatives as
\[
\mathcal{V}(\nu) \;:=\; \{\nu^{i,j} : (i,j)\in I\}.
\]
It is easy to see that the optimal subtree in any $\nu' \in \mathcal{V}(\nu)$ is not subtree 1. Specifically, for any $(i,1)$ with $i\neq 1$, the optimal subtree in $\nu^{i,1}$ is subtree $i$, and the second optimal subtree in $\nu^{i,1}$ is subtree $1$; for any $(1,j)$ with $j\neq 1$, the optimal subtree in $\nu^{1,j}$ is subtree 2. One can see that for each $(i,j)\in I$, $\Dt^{i,j}_{a,b} \geq \fr 12 \Dt_{a,b}$, for all $(a,b) \in [K] \times [L]$. Thus, for each $(i,j)\in I$, $\nu^{i,j} \in \mathcal{P}_{4 H(\nu)}$. 

For any $\nu'$, let $\mathbb{P}_{\nu'}$ denote the distribution of the full interaction transcript
generated by running $\pi$ on instance $\nu$.
Let $N_{i,j}$ be the (random) number of pulls of leaf $(i,j)$ up to time $T$, and write $\mathbb{E}_0$ for
expectation under $\mathbb{P}_{\nu}$ (here $\nu$ is the base instance).

\paragraph{Step 1: KL decomposition.}
A standard divergence decomposition (chain rule for KL under adaptive sampling) yields, for any $\nu' \in \mathcal{V}(\nu)$, 
\begin{equation*}
\mathrm{KL}\!\left(\mathbb{P}_{\nu},\mathbb{P}_{\nu'}\right)
=
\sum_{i,j} \EE_0 [N_{i,j}]\cdot \mathrm{KL}\!\left(\cN(\mu_{i,j},1),\cN(\mu'_{i,j},1)\right).
\end{equation*}
For unit-variance Gaussians,
\[
\mathrm{KL}\!\left(\cN(m,1),\cN(m',1)\right)=\frac{(m-m')^2}{2}.
\]
For any $i\in\{2,\ldots,K\}$, the instances $\nu^{i,1}$ and $\nu$ differ only in the mean of leaves in subtree $i$. Hence, 
\begin{align}
    \mathrm{KL}\!\left(\mathbb{P}_{\nu},\mathbb{P}_{\nu^{i,1}}\right)
=&
\sum_{j} \EE_0 [N_{i,j}]\cdot \mathrm{KL}\!\left(\cN(\mu_{i,j},1),\cN(\mu^{i,1}_{i,j},1)\right) \nonumber
\\=&
\sum_{j} \EE_0 [N_{i,j}]\cdot (2\Delta_{i,1}^2) 
.
\label{eqn:KL-for-alternative-1}
\end{align}

For any $j\in\{2,\ldots,L\}$, the instances $\nu^{1,j}$ and $\nu$ differ only in the mean of leaf $(1,j)$. Hence, 
\begin{align}
    \mathrm{KL}\!\left(\mathbb{P}_{\nu},\mathbb{P}_{\nu^{1,j}}\right)
=&
\EE_0 [N_{1,j}]\cdot \mathrm{KL}\!\left(\cN(\mu_{1,j},1),\cN(\mu^{1,j}_{1,j},1)\right) \nonumber
\\=&
\EE_0 [N_{1,j}]\cdot (2\Delta_{1,j}^2)
. \label{eqn:KL-for-alternative-2}
\end{align}

% \begin{align*}
%     KL(\PP_{\nu}, \PP_{\nu'}) 
%     \lsim \sum_{i,j} \EE_0 [N_{i,j}] (\mu_{i,j} - \mu_{i,j}')^2 
% \end{align*}
% \begin{equation}
% % \label{eq:KL-decomp-gauss-tree}
% \mathrm{KL}\!\left(\mathbb{P}_{\nu_1},\mathbb{P}_{\nu_i}\right)
% =
% \mathbb{E}_1[N_i]\cdot \mathrm{KL}\!\left(\cN(\tfrac12-d_i,1),\cN(\tfrac12+d_i,1)\right).
% \end{equation}

% Here $(m-m')^2=(2d_i)^2=4d_i^2$, hence
% \begin{equation}\label{eq:KL-gauss-arm}
% \mathrm{KL}\!\left(\cN(\tfrac12-d_i,1),\cN(\tfrac12+d_i,1)\right)=2d_i^2.
% \end{equation}
% Combining~\eqref{eq:KL-decomp-gauss} and~\eqref{eq:KL-gauss-arm} gives
% \begin{equation}\label{eq:KL-upper-gauss}
% \mathrm{KL}\!\left(\mathbb{P}_{\nu_1},\mathbb{P}_{\nu_i}\right)
% =2\,\mathbb{E}_1[N_i]\,d_i^2.
% \end{equation}

\paragraph{Step 2: Two-point testing inequality.}
We use the Bretagnolle--Huber inequality: for any distributions $P,Q$ and any event $E$,
\begin{equation}\label{eq:BH-gauss-tree}
P(E^c)+Q(E)\;\ge\;\frac12 \exp\!\bigl(-\mathrm{KL}(P,Q)\bigr).
\end{equation}
Let $(i,j) \in I$ and
apply~\eqref{eq:BH-gauss-tree} with $P=\mathbb{P}_{\nu_1}$, $Q=\mathbb{P}_{\nu^{i,j}}$, and $E=\{ \text{output subtree }1\}$.
Since arm $1$ is optimal under $\nu_1$ and suboptimal under $\nu^{i,j}$,
\[
\mathbb{P}_{\nu^{i,j}}(E^c)=\mathbb{P}_{\nu_1}(J_T\neq 1)=\mathbb{P}_{\nu_1}(\mathrm{error}),
\qquad
\mathbb{P}_{\nu^{i,j}}(E)=\mathbb{P}_{\nu^{i,j}}( \text{output subtree }1)\le \mathbb{P}_{\nu^{i,j}}(\mathrm{error}).
\]
Thus,
\begin{equation*}
\mathbb{P}_{\nu_1}(\mathrm{error})+\mathbb{P}_{\nu^{i,j}}(\mathrm{error})
\;\ge\;
\frac12 \exp\!\left(-\mathrm{KL}\!\left(\mathbb{P}_{\nu_1},\mathbb{P}_{\nu^{i,j}}\right)\right),
\end{equation*}
and 
\begin{equation}\label{eq:sum-error-lb-gauss}
\max\Bigl\{\mathbb{P}_{\nu_1}(\mathrm{error}),\ \mathbb{P}_{\nu^{i,j}}(\mathrm{error})\Bigr\}
\;\ge\;
\frac12\mathbb{P}_{\nu_1}(\mathrm{error})+\mathbb{P}_{\nu^{i,j}}(\mathrm{error})
\;\ge\;
\frac14 \exp\!\left(-\mathrm{KL}\!\left(\mathbb{P}_{\nu_1},\mathbb{P}_{\nu^{i,j}}\right)\right). 
\end{equation}
\paragraph{Step 3: Proof by contradiction.}
We assume, for contradiction, that the algorithm has a small error on all instances:
$$ \forall \nu' \in \mathcal{P}_{4 H(\nu)}, \quad \PP_{\nu'}(\text{error}) < \fr 14 \exp\!\left(- 2 \frac{T}{\Hlb(\nu)}\right)$$

This implies that for all $(i,j) \in I$, 
\begin{align*}
    \max\Bigl\{\mathbb{P}_{\nu_1}(\mathrm{error}),\ \mathbb{P}_{\nu^{i,j}}(\mathrm{error})\Bigr\} 
    < \fr 14 \exp\!\left(- 2 \frac{T}{\Hlb(\nu)}\right)
\end{align*}
In conjunction with Equation~\eqref{eq:sum-error-lb-gauss}, we have, for all $(i,j) \in I$, 
\[
\mathrm{KL}\!\left(\mathbb{P}_{\nu_1},\mathbb{P}_{\nu^{i,j}}\right)
> \frac{2 T}{\Hlb(\nu)}
\]
By Equations~\eqref{eqn:KL-for-alternative-1} and~\eqref{eqn:KL-for-alternative-2}, that is, for all $(i,1): i \neq 1$, 
    $$
    \sum_{j \in [L]} \EE_0 [N_{i,j}] \Delta_{i,1}^2 > \frac{T}{\Hlb(\nu)}
    $$
    and also, for all $(1,j): j \neq 1$, $$
    \EE_0 [N_{1,j}] \Delta_{1,j}^2 > \frac{T}{\Hlb(\nu)}
    $$
However, 
\begin{align*}
    T &\geq \sum_{i \neq 1} \sum_{j \in [L]} \EE_0 [N_{i,j}] + \sum_{j \neq 1} \EE_0 [N_{i,j}] 
    \\&> 
\sum_{i \neq 1} \frac{T}{\Hlb(\nu)} \fr{1}{\Delta_{i,1}^2} + \sum_{j \neq 1} \frac{T}{\Hlb(\nu)} \fr{1}{\Delta_{1,j}^2} 
\\&= 
\frac{T}{\Hlb(\nu)}  \del{ \sum_{i \neq 1} \fr{1}{\Delta_{i,1}^2} + \sum_{j \neq 1} \fr{1}{\Delta_{1,j}^2} } 
\\&= T
\end{align*}
leading to a contradiction. Therefore, it cannot be true that 
$$ \forall \nu' \in \mathcal{P}_{4 H(\nu)}, \quad \PP_{\nu'}(\text{error}) < \fr 14 \exp\!\left(- 2 \frac{T}{\Hlb(\nu)}\right),$$
completing the proof. 

\end{proof}

\section{A negative result for permutation-style lower bounds}
\label{app:negative-permutation-lb}

The lower bound in Section~\ref{sec:lower-bound} is governed by
$H_{\mathrm{lb}}(\nu)$, which shows a sparsity pattern compared to the upper-bound complexity
$H_2(0)$. A natural question is whether one can prove a
stronger $H_2$-type lower bound by adapting the permutation construction used in
fixed-budget best-arm identification. This appendix shows that this route fails
for max--min trees, even in the smallest nontrivial $(2,2)$ case.

The obstruction is that, once the grouping of leaves into subtrees is known,
permutation uncertainty alone does not capture the full statistical difficulty
of max--min identification. 
% Some leaves may be expensive for a safe algorithm to
% handle, but irrelevant to the most economical alternative that flips the identity
% of the optimal subtree.

\paragraph{Lower bounds in fixed-budget best-arm identification.}
Fixed-budget best-arm identification (BAI) can be viewed as a multi-hypothesis testing problem: each bandit instance $\nu=(\nu_1,\dots,\nu_K)$ induces a unique optimal arm
$i^\star(\nu)\in[K]$, and an algorithm $\pi$ adaptively collects $T$ samples and outputs an estimate $\hat{i}\in[K]$. The misidentification event is
$\mathsf{err}(\pi,\nu):=\{\hat{i}\neq i^\star(\nu)\}$ with probability
$p_T(\pi,\nu):=\PP_\nu^\pi(\mathsf{err})$.
A central consequence of this testing viewpoint is that meaningful lower bounds cannot be pointwise in the sense of bounding $p_T(\pi,\nu)$ from below for a fixed instance $\nu$ uniformly over algorithms: such a statement is impossible, since the degenerate algorithm that always outputs $i^\star(\nu)$ achieves $p_T(\pi,\nu)=0$ on that particular instance.
Therefore, fixed-budget lower bounds are formulated as local minimax guarantees over alternative instances that induce different optimal arms, 
and take an adversarial quantifier structure of the form
\[
\forall \pi,\ \exists \nu'\text{ ``close'' to } \nu, \text{ s.t. }
p_T(\pi,\nu') \text{ is large}. 
\]
where ``closeness'' can take different definitions, depending on the specific theorem statements and constructions. \yinan{Mention closeness can be interpreted as reasonable algorithms, connecting with regret minimization}

Thus, fixed-budget BAI lower bounds certify that no algorithm can be simultaneously accurate on $\nu$ and on a suitably chosen alternative $\nu'$ that induces a different correct decision, ruling out ``lucky guess'' strategies.

Following this spirit, one of the earliest results on fixed-budget BAI lower bounds presented in~\citet{audibert10best} adopts a permutation construction, and thus an alternative instance $\nu'$ being close to $\nu$ means $\nu'_i = \nu_{\sigma(i)}$, for all $i \in [K]$, where $\sigma: [K] \rightarrow [K]$ is a permutation. The lower bound (Theorem 4) in~\citet{audibert10best} can be interpreted as follows: for any instance $\nu$, and any algorithm $\pi$, even if the knowledge (just except for the indices) of the instance is available to the algorithm, on at least one of the close instance $\nu'$ (a permutation of $\nu$), $\pi$ has a misidentification probability at least $\exp(-\tTh(T/H_2(\nu)))$. 

One natural question we asked was, can a similar lower bound be shown for max-min subtree identification (with $\eps = 0$), i.e., \emph{is it possible to show that, even if the knowledge (just except for the indices) of the instance $\nu$ is available to the algorithm, any algorithm still has a high probability of error of $\exp(-\tTh(T/H_2(\nu)))$?}

Subsequently, we show that, perhaps surprisingly, the above-described lower bound for max-min trees cannot be shown. This indicates that the learning difficulty in max-min identification is not merely a consequence of ``label uncertainty'' (as in unstructured $K$-armed BAI), but is instead deeply rooted in the internal maximin structure. In particular, unlike the unstructured setting where all arms play symmetric roles under permutations, max-min trees inherently break this symmetry, and not every leaf contributes ``equally'' to the max-min identification difficulty.

\paragraph{Subtree-preserving permutations.}
Fix $K,L\ge 1$. A subtree-preserving permutation is a tuple
$\sigma \;=\;\bigl(\sigma_0,\sigma_1,\ldots,\sigma_K\bigr)$, 
% \[
% \sigma \;=\;\bigl(\sigma_0,\sigma_1,\ldots,\sigma_K\bigr),
% \]
where $\sigma_0:[K]\rightarrow[K]$ is a permutation of subtrees and, for each $i\in[K]$,
$\sigma_i:[L]\rightarrow[L]$ is a permutation of the leaves within subtree $i$.
Given an instance $\nu$ with
reward distributions $\{(i, j, \nu_{i,j})\}_{i\in[K],\,j\in[L]}$, we define the subtree-preserving permuted instance
 $\nu^\sigma$ to be the instance with reward distributions
 % $\{(i, j, \nu_{\sigma_0(i),\sigma_i(j)})\}_{i\in[K],\,j\in[L]}$. 
\begin{equation}\label{eq:tree-permutation}
\{(i, j, \nu_{\sigma_0(i),\sigma_i(j)})\}_{i\in[K],\,j\in[L]}
\end{equation}
Equivalently, $\nu^\sigma$ is obtained by first 
inside each subtree $i$, permutating its leaves by $\sigma_i$, and then, permuting subtrees by $\sigma_0$.

\paragraph{A $(K=2,L=2)$ counterexample to $H_2$-type necessity under known grouping.}
% \kj{$(2,2)$ may not be clear immediately. how about we say $(K=2,L=2)$? }
We now formalize a simple impossibility statement showing that when the grouping of leaf means into
subtrees is known, the worst-case hardness captured by an $H_2$ (or $H_1$)-type upper bound for trees (Definition~\ref{def:eps-complexity}) is not necessary, and thus an $H_2$ (or $H_1$)-type lower bound is impossible to be shown with known grouping, 
even in the smallest nontrivial case.

\begin{proposition}[Known grouping can make $(2,2)$ maximin identification easier]
\label{prop:known-grouping-easy-22}
Fix $\rho \in (0,1)$ and consider the base $(2,2)$ Gaussian maximin-tree
instance $\bar\nu_\rho$ with unit variances and grouped leaf means
\[
    \bigl\{\{\bar\mu_{1,1},\bar\mu_{1,2}\},
           \{\bar\mu_{2,1},\bar\mu_{2,2}\}\bigr\}
    =
    \bigl\{\{1,\rho\},\{\rho^2,0\}\bigr\}.
\]
Let $\mathcal I_\rho^{\mathrm{grp}}$ be the class of instances obtained from
$\bar\nu_\rho$ by an arbitrary subtree-preserving permutation, i.e., by possibly
swapping the two subtrees and by possibly permuting the two leaves inside each
subtree. Suppose that the algorithm is given this grouped base instance
$\bar\nu_\rho$, but not the realized permutation.

Then there exists a single fixed-budget identification procedure $\pi$ such that,
uniformly over all $\nu \in \mathcal I_\rho^{\mathrm{grp}}$,
\[
    \mathbb P_\nu\!\left(\widehat i_T \neq i^\star(\nu)\right)
    \le
    C \exp\!\left(
        - c \frac{T}{H_1^{\mathrm{flatten}}(\rho)}
    \right),
\]
where $c,C>0$ are universal constants and
\[
    H_1^{\mathrm{flatten}}(\rho)
    :=
    (1-\rho)^{-2}
    +
    (1-\rho^2)^{-2}
    +
    1 .
\]
Moreover, let $H_1^{\mathrm{tree}}(\rho)$ denote the tree complexity
$H_1(0)$ from Definition~\ref{def:eps-complexity}, evaluated on this grouped
instance family. Then
\[
    H_1^{\mathrm{flatten}}(\rho)
    =
    o\!\left(H_1^{\mathrm{tree}}(\rho)\right)
    \qquad \text{as } \rho \rightarrow 0^+ .
\]
\end{proposition}

\begin{proof}
For every instance $\nu \in \mathcal I_\rho^{\mathrm{grp}}$, one subtree has
leaf means $\{1,\rho\}$ and the other subtree has leaf means $\{\rho^2,0\}$.
Therefore the subtree containing the leaf with mean $1$ is the unique optimal
subtree, since its min-value is $\rho$, whereas the other subtree has min-value
$0$.

Consider the following fixed-budget procedure. Ignore the tree structure during
sampling and run a standard fixed-budget best-arm identification procedure on
the four leaves, treating them as four arms. After the budget is exhausted, let
$\widehat \ell_T$ be the recommended best leaf, and output the subtree containing
$\widehat \ell_T$. This is a single algorithm; it does not depend on which
subtree-preserving permutation was realized.

If the procedure correctly identifies the unique leaf with mean $1$, then it
outputs the optimal subtree. Hence
\[
    \left\{\widehat i_T \neq i^\star(\nu)\right\}
    \subseteq
    \left\{\widehat \ell_T \neq \ell^\star(\nu)\right\},
\]
where $\ell^\star(\nu)$ denotes the unique leaf with mean $1$. The ordinary
four-arm best-arm identification problem has gaps
\[
    1-\rho, \qquad 1-\rho^2, \qquad 1 .
\]
Thus, by the standard fixed-budget best-arm identification guarantee
\citep{audibert10best}, there exist universal constants $c,C>0$ such that
\[
    \mathbb P_\nu\!\left(\widehat \ell_T \neq \ell^\star(\nu)\right)
    \le
    C \exp\!\left(
        - c \frac{T}{H_1^{\mathrm{flatten}}(\rho)}
    \right),
\]
where
\[
    H_1^{\mathrm{flatten}}(\rho)
    =
    (1-\rho)^{-2}
    +
    (1-\rho^2)^{-2}
    +
    1 .
\]
This proves the claimed upper bound.

It remains to compare this flattened complexity with the maximin-tree
complexity. In the notation of Definition~\ref{def:eps-complexity}, after
ordering leaves within subtrees and ordering subtrees by min-value, the same
instance has
\[
    \mu_{1,1}=\rho, \qquad
    \mu_{1,2}=1, \qquad
    \mu_{2,1}=0, \qquad
    \mu_{2,2}=\rho^2 .
\]
Therefore the leafwise maximin gaps are
\[
    \Delta_{1,1} = \mu_{1,1}-\mu_{2,1} = \rho,
    \qquad
    \Delta_{1,2} = \mu_{1,2}-\mu_{2,1} = 1,
\]
and
\[
    \Delta_{2,1}
    =
    \max\{\mu_{1,1}-\mu_{2,1},\,\mu_{2,1}-\mu_{2,1}\}
    =
    \rho,
\]
\[
    \Delta_{2,2}
    =
    \max\{\mu_{1,1}-\mu_{2,1},\,\mu_{2,2}-\mu_{2,1}\}
    =
    \max\{\rho,\rho^2\}
    =
    \rho .
\]
Consequently,
\[
    H_1^{\mathrm{tree}}(\rho)
    =
    \sum_{i,j} \Delta_{i,j}^{-2}
    =
    \rho^{-2} + 1 + \rho^{-2} + \rho^{-2}
    =
    1 + 3\rho^{-2}.
\]
On the other hand,
\[
    H_1^{\mathrm{flatten}}(\rho)
    =
    (1-\rho)^{-2}
    +
    (1-\rho^2)^{-2}
    +
    1
    \longrightarrow 3
    \qquad \text{as } \rho \rightarrow 0^+ .
\]
Hence
\[
    \frac{H_1^{\mathrm{flatten}}(\rho)}
         {H_1^{\mathrm{tree}}(\rho)}
    \longrightarrow 0
    \qquad \text{as } \rho \rightarrow 0^+,
\]
which proves
$H_1^{\mathrm{flatten}}(\rho)
=
o(H_1^{\mathrm{tree}}(\rho))$.
\end{proof}

\section{Perspectives: fixed-budget optimality and critical leaves}
\label{sec:appendix-perspectives}

Our upper and lower bounds leave a gap when $\eps=0$ that is not yet fully characterized. 

The fixed-confidence counterpart of maxmin action identification is now relatively
well understood. The asymptotically optimal allocation is usually
described through the characteristic time
\[
    T^\star(\mu)^{-1}
    =
    \sup_{w \in \Sigma}
    \inf_{\lambda \in \Alt(\mu)}
    \sum_{a} w_a d(\mu_a,\lambda_a),
\]
where $\Alt(\mu)$ denotes the set of alternative instances with a different
correct answer. For maxmin game trees, \citet{kaufmann2021mixture} show in
Section~5.2.2 that the generic GLR stopping rule, combined with Tracking, gives
an asymptotically optimal fixed-confidence procedure for depth-two game trees
under the required regularity assumptions on the oracle weights. In that setting,
the problem has rank $L+1$, and the oracle weights are obtained from the
variational problem above; they are described as computable by disciplined convex
optimization tools, but only numerically computable in general.

Our results address a different criterion. Fixed-confidence guarantees concern a
data-dependent stopping time and a prescribed error probability $\delta$, whereas
fixed-budget guarantees concern the error probability after a deterministic
budget $T$. Therefore, the fixed-confidence optimality result does not directly
provide a fixed-budget error exponent

Nevertheless, the fixed-confidence analysis provides useful intuition about why
maximin identification is structurally different from ordinary best-arm
identification. In Section~6 of \citet{garivier2016maximin}, the authors analyze
the simplest nontrivial $(2,2)$ case and show that the optimal allocation can
depend delicately on the relative positions of the leaf means. Translating their
notation to ours, let
\[
    \mu_1 := \mu_{1,1}, \qquad
    \mu_2 := \mu_{1,2}, \qquad
    \mu_3 := \mu_{2,1}, \qquad
    \mu_4 := \mu_{2,2}.
\]
In the regime $\mu_4 > \mu_2$, their fixed-confidence oracle allocation satisfies
$w_4^\star(\mu)=0$. Hence an asymptotically optimal strategy may sample leaf
$(2,2)$ only a vanishing fraction of the time, even though this leaf is part of
the tree. This illustrates a phenomenon that is central to our lower-bound
discussion: some leaves may be irrelevant to the closest way of confusing the
identity of the optimal subtree.

A closely related phenomenon appears in our fixed-budget lower bound. For a base
instance $\nu$, the proof constructs alternative instances that change the
identity of the optimal subtree while incurring small information cost. Such a
closest confusing alternative need not perturb every leaf. In the $(2,2)$
example above, the alternative may leave leaf $(2,2)$ unchanged, so the
corresponding lower-bound exponent does not involve the gap associated with that
leaf. From an information-theoretic point of view, this leaf is not needed to
create the most economical confusion.

This does not immediately imply a matching upper bound. In an upper-bound proof,
the algorithm must operate without knowing in advance which leaves are
structurally critical for the current instance. In particular, it must avoid a
failure mode specific to max--min trees: if the true minimizing leaf of a
suboptimal subtree is removed while the subtree remains active, then the
surviving minimum of that subtree is artificially increased, and later phases may
compare against a distorted max--min value. Therefore a sound fixed-budget
algorithm may need to spend samples to distinguish whether a leaf is irrelevant
or whether it is the minimizer that certifies the subtree's value. This is the
reason our elimination rule is designed to be tree-safe. 
% rather than simply to
% track the leaves appearing in a particular closest alternative.

Recent work on fixed-budget best-arm identification suggests that this
phenomenon should not be surprising. Even in ordinary unstructured BAI, sharp
fixed-budget optimality is more delicate than fixed-confidence optimality.
\citet{komiyama2022minimax} study minimax normalized error exponents and show
that one natural optimal rate, denoted $R^{\mathrm{go}}$, is tied to an oracle
allocation depending on the final empirical distribution; they emphasize that
achieving this rate by an actual adaptive algorithm is nontrivial. Their delayed
optimal tracking construction attains a related asymptotic rate, but is primarily
theoretical because it requires optimizing over high-dimensional allocation
functions. \citet{degenne2023existence} further shows that some fixed-budget
identification problems do not admit a single complexity, in the sense of a lower
bound on the error exponent that is attained by one algorithm on all instances;
this failure already occurs in simple Bernoulli best-arm identification with two
arms.

Taken together, these observations suggest that closing the remaining gap for
fixed-budget max--min identification may require a sharper alignment between two
objects: on the algorithmic side, sampling rules that adaptively discover which
leaves are critical without corrupting the max--min structure; and on the
information-theoretic side, lower bounds based on the truly closest confusing
alternatives. Our results provide one step in this direction by giving an
explicit $\varepsilon$-agnostic fixed-budget upper bound and a complementary
lower-bound analysis that isolates a subset of structurally critical leaves.

\section{Experiments}
\label{sec:experiments}

% It confirms that our theory correctly predicts the log error probability

% We evaluate the proposed SR-MCTS algorithm on synthetic depth-2 max--min trees. Each leaf
% $(i,j)$ generates independent Gaussian samples
% \[
%     X_{i,j}\sim \mathcal N(\mu_{i,j},\sigma^2),
% \]
% with common variance $\sigma^2$ and instance-dependent means. The value of subtree $i$ is
% $v_i=\min_{j\in[L]}\mu_{i,j}$, and the learner outputs a subtree $\widehat i_T$ after a fixed
% budget $T$. We report the empirical probability of recommending a non-$\varepsilon$-good subtree,
% \[
%     \Pr(\widehat i_T\notin G_\varepsilon),
%     \qquad
%     G_\varepsilon=\{i\in[K]:v_i\ge v^\star-\varepsilon\},
%     \qquad
%     v^\star=\max_i v_i .
% \]
% For $\varepsilon$-good experiments, $\varepsilon$ is used only for evaluation; the algorithms do not
% receive $\varepsilon$ as input.

In this section, we evaluate the performance of SR-MCTS against three baselines: uniform sampling, Bottom-up SAR, and SARCompare, using a maxmin tree setting. The results show that SR-MCTS consistently outperforms these methods in both best-subtree identification and $\varepsilon$-good subtree identification. We further examine the sample allocation patterns of each method and show that SR-MCTS exhibits the desired adaptive allocation behavior. Finally, we provide empirical evidence supporting our theoretical bound. All simulations were run on a standard CPU machine; no GPU was used.

\textfloatsep=.5em

\paragraph{Baselines.}
% We compare against uniform sampling, a Bottom-up SAR baseline, and SAR+Compare.
The Bottom-up SAR baseline applies the multi-bandit Successive Accepts and Rejects
algorithm for multi-bandit best arm identification in~\citet{bubeck13multiple} to the negated rewards within
subtrees, thereby identifying one candidate minimum leaf per subtree. SAR+Compare first
uses this bottom-up stage with budget fraction $\alpha$, and then uses Successive Rejects~\citep{audibert10best}
to compare the selected candidate minima across subtrees. We set $\alpha = 0.8$ for all experiments.

% We compare against uniform sampling, bottom-up SAR, and SAR+Compare. Bottom-up
% SAR first identifies one candidate minimizing leaf in each subtree and then chooses the subtree
% with the largest estimated candidate value. SAR+Compare uses a fraction $\alpha$ of the budget
% for this bottom-up stage and the remaining budget to compare the resulting candidate minima. We
% use $\alpha = 0.8$. \yinan{To be confirmed}

% Throughout, ties are broken by a fixed deterministic rule. For $M\ge 2$, define
% \[
%     \overline{\log}(M):=\frac12+\sum_{r=2}^M\frac1r .
% \]

\begin{algorithm}[H]
\caption{\textsc{MultiBanditSAR} Subroutine}
\label{alg:multibandit-sar-subroutine}
\begin{algorithmic}[1]
\REQUIRE A collection of bandit problems, budget $B$
\STATE Apply the Successive Accepts and Rejects algorithm of~\citet{bubeck13multiple}.
\STATE In each bandit problem, return one estimated best arm and its empirical mean.
\end{algorithmic}
\end{algorithm}

\begin{algorithm}[H]
\caption{\textsc{SR} Subroutine}
\label{alg:sr-subroutine}
\begin{algorithmic}[1]
\REQUIRE A set of arms, budget $B$
\STATE Apply Successive Rejects to the arms using budget $B$.
\STATE Return the final surviving arm and its empirical mean.
\end{algorithmic}
\end{algorithm}

\begin{algorithm}[H]
\caption{Uniform Sampling Baseline}
\label{alg:uniform-baseline}
\begin{algorithmic}[1]
\REQUIRE Budget $T$, depth-2 tree with $K$ subtrees and $L$ leaves per subtree
\STATE Sample all $KL$ leaves as evenly as possible using budget $T$.
\STATE Compute the empirical min-value of each subtree:
\[
    \widehat v_i = \min_{j\in[L]} \widehat\mu_{i,j}.
\]
\STATE \textbf{return}
\[
    \widehat i_T \in \arg\max_{i\in[K]} \widehat v_i .
\]
\end{algorithmic}
\end{algorithm}

\begin{algorithm}[H]
\caption{Bottom-up SAR Baseline}
\label{alg:bottom-up-sar}
\begin{algorithmic}[1]
\REQUIRE Budget $T$, depth-2 tree with $K$ subtrees and $L$ leaves per subtree
\STATE Treat each subtree as one bandit problem with arms given by its leaves.
\STATE Apply \textsc{MultiBanditSAR} with budget $T$ to the negated rewards $-X_{i,j}$.
\STATE Let $\widehat j(i)$ be the selected leaf in subtree $i$.
\STATE Estimate each subtree by the empirical mean of its selected candidate minimum:
\[
    \widehat m_i = \widehat\mu_{i,\widehat j(i)} .
\]
\STATE \textbf{return}
\[
    \widehat i_T \in \arg\max_{i\in[K]} \widehat m_i .
\]
\end{algorithmic}
\end{algorithm}

\begin{algorithm}[H]
\caption{SAR+Compare Baseline}
\label{alg:sar-compare}
\begin{algorithmic}[1]
\REQUIRE Budget $T$, split parameter $\alpha\in(0,1)$, depth-2 tree with $K$ subtrees and $L$ leaves
\STATE $T_{\mathrm{min}}\leftarrow \lfloor \alpha T\rfloor$ and
$T_{\mathrm{cmp}}\leftarrow T-T_{\mathrm{min}}$.
\STATE Apply \textsc{MultiBanditSAR} with budget $T_{\mathrm{min}}$ to the negated rewards $-X_{i,j}$.
\STATE Let $\widehat j(i)$ be the selected candidate minimum in subtree $i$.
\STATE Apply \textsc{SR} with budget $T_{\mathrm{cmp}}$ to the candidate leaves
\[
    \{(i,\widehat j(i)):i\in[K]\},
\]
using the original rewards.
\STATE \textbf{return} the subtree containing the surviving candidate leaf.
\end{algorithmic}
\end{algorithm}

\paragraph{Construction of $(K,L)$-Maxmin Trees.} Since the tree value is determined by the minimum leaf value within each subtree, we directly control the minimum values of the subtrees. Specifically, we set these minimum values to be evenly spaced between $0$ and $-0.18$. For the remaining leaf nodes, we consider two variants. In a structured version, for each subtree $i$, we sample or assign its leaf values from the interval between its minimum value and that value plus $0.04$. In a random version, we sample the remaining leaf values uniformly from the interval between the subtree minimum value and $1$. For all constructions, we set $K = 10$ and $L = 10$. Figure~\ref{fig:tree_means} visualizes the mean values of all leaf nodes for both instances. We use both the structured and random instances in Experiments 1 and 2, and use only the structured instance in Experiments 3 and 4. The default budget is 2000. For the first experiment, we vary the budget in the range of 2000 to 10000. For all experiments, we run $1000$ independent trials and report standard errors whenever applicable. 

\begin{figure}[H]
    \begin{center}
      \includegraphics[width=0.6\linewidth]{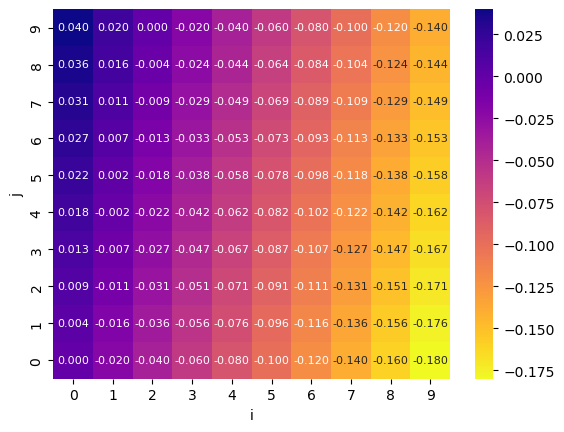}
      \includegraphics[width=0.6\linewidth]{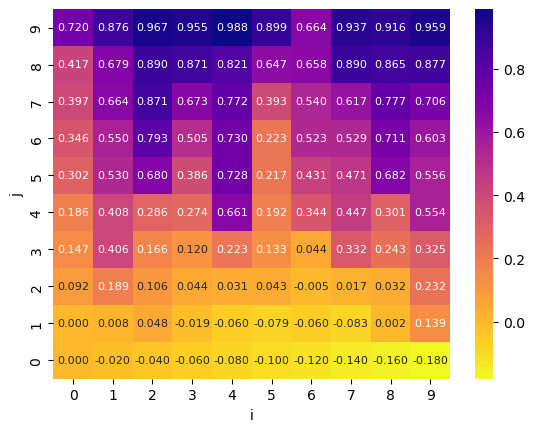}
    \end{center}
  \caption{Visualization of the mean values in a max-min tree. We use the convention $(i,j)$ to denote a leaf node, where $i$ indexes the subtree and $j$ indexes the leaf within that subtree. Top: structured instance. Bottom: random instance.}
  \label{fig:tree_means}   
\end{figure}

\paragraph{Experiment 1: Best-Subtree Identification.} The first experiment evaluates best-subtree identification, where the goal is to correctly identify the subtree with the largest minimum value. Figure~\ref{fig:probmis_vs_budget} illustrates the probability of misidentification as a function of the sampling budget $T$. As expected, the probability of misidentification decreases for all methods as more samples become available. Notably, SR-MCTS consistently statistically outperforms the competing methods across the entire range of budgets. These results demonstrate that SR-MCTS is highly sample-efficient, identifying the optimal subtree more reliably with fewer samples.

\begin{figure}[H]
    \begin{center}
      \includegraphics[width=0.6\linewidth]{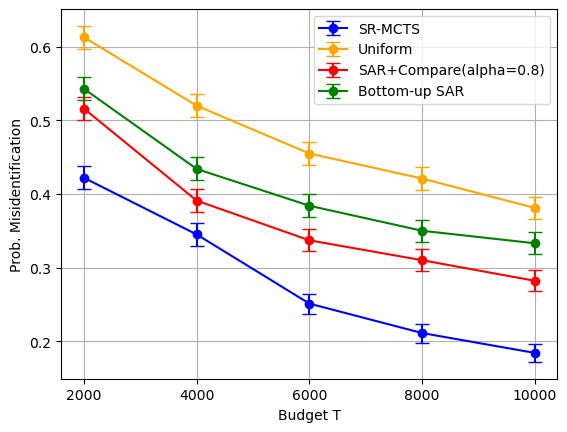}
      \includegraphics[width=0.6\linewidth]{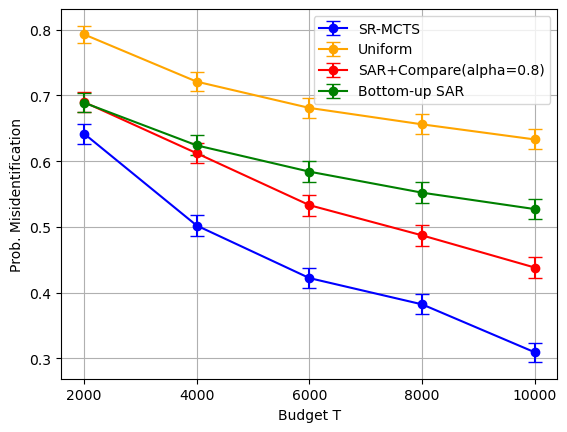}
    \end{center}
  \caption{Probability of misidentification for all methods as the budget $T$ increases. Top: structured instance. Bottom: random instance.}
  \label{fig:probmis_vs_budget}
\end{figure}

\paragraph{Experiment 2: $\varepsilon$-Good Subtree Identification.} The second experiment considers $\varepsilon$-good subtree identification, where the goal is to identify any subtree whose minimum value is within $\varepsilon$ of the optimal subtree. Recall from the preliminaries that subtree $i$ is $\varepsilon$-good if $\mu_{i,1} \ge \mu_{1,1} - \varepsilon$. We note that all algorithms are $\varepsilon$-agnostic: the value of $\varepsilon$ is not provided as an input. Nevertheless, the algorithms can still be evaluated under this relaxed objective by checking whether the returned subtree is $\varepsilon$-good. As $\varepsilon$ increases, the problem becomes easier because more subtrees satisfy the $\varepsilon$-good condition. Consequently, we expect the probability of misidentification to decrease for larger values of $\varepsilon$. In our max-min tree construction, when $\varepsilon = 0.08$, approximately half of the subtrees are $\varepsilon$-good, making the identification problem substantially easier. This trend is observed for all methods in Figure~\ref{fig:probmis_vs_eps}. Across all values of $\varepsilon$, SR-MCTS consistently achieves a lower probability of misidentification, indicating that it is more effective than the competing methods at identifying near-optimal subtrees.

\begin{figure}[H]
    \begin{center}
      \includegraphics[width=0.7\linewidth]{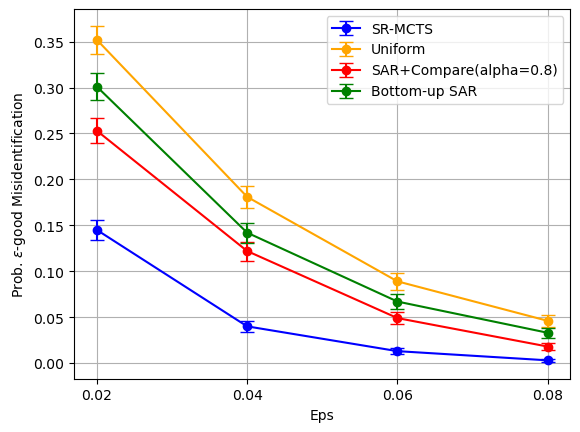}
      \includegraphics[width=0.7\linewidth]{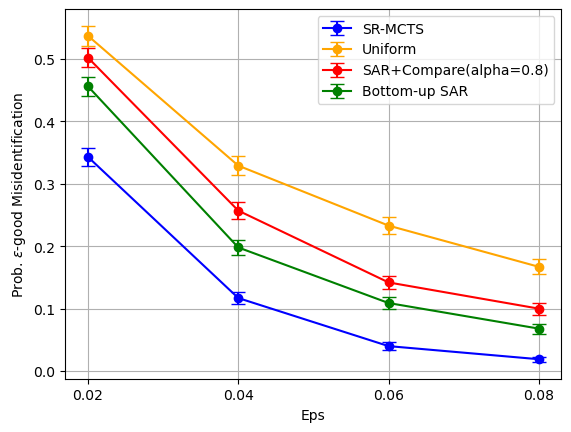}
    \end{center}
  \caption{Probability of $\eps$-good misidentification for all methods as $\eps$ increases. Top: structured instance. Bottom: random instance.}
  \label{fig:probmis_vs_eps}
\end{figure}

\paragraph{Experiment 3: Sample Allocation Patterns.} The third experiment examines how samples are allocated across leaf nodes by SR-MCTS, SAR+Compare, and Bottom-up SAR as shown Figure~\ref{fig:nsamples_srmcts},~\ref{fig:nsamples_sarcompare}, and ~\ref{fig:nsamples_bottomupsar} respectively. The heatmaps show the average number of samples assigned to each leaf node, where $i$ indexes the subtree and $j$ indexes the leaf within that subtree. SR-MCTS exhibits a noticeably different allocation pattern compared with SAR+Compare and Bottom-up SAR. Both SAR+Compare and Bottom-up SAR allocate samples more uniformly across subtrees, with slightly higher sampling effort on the minimum-valued leaves within each subtree. In contrast, SR-MCTS concentrates more samples on critical regions of the tree, especially on leaves that are most informative for distinguishing promising subtrees from suboptimal ones. This adaptive behavior is desirable in the max-min tree setting. Once a subtree is clearly suboptimal, SR-MCTS can reduce sampling effort on that subtree and reallocate samples to more competitive regions. For example, bad subtrees whose maximum leaf value is already below the minimum value of the optimal subtree are likely to be eliminated early. As a result, SR-MCTS spends fewer samples on clearly inferior subtrees and focuses more on the leaf nodes that are most relevant for identifying the best subtree.

\begin{figure}[H]
    \begin{center}
      \includegraphics[width=0.5\linewidth]{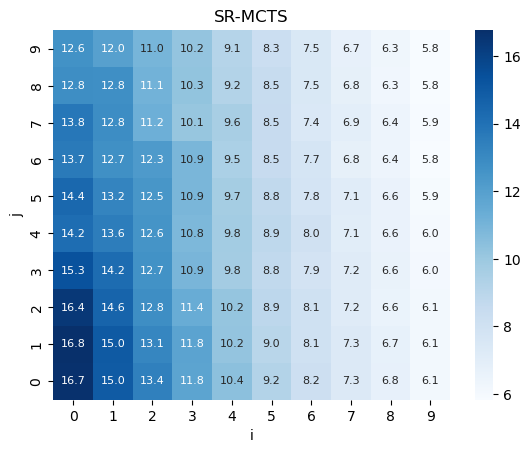}
    \end{center}
  \caption{Heatmap of SR-MCTS Sampling Behavior}
  \label{fig:nsamples_srmcts}
\end{figure}

\begin{figure}[H]
    \begin{center}
      \includegraphics[width=0.5\linewidth]{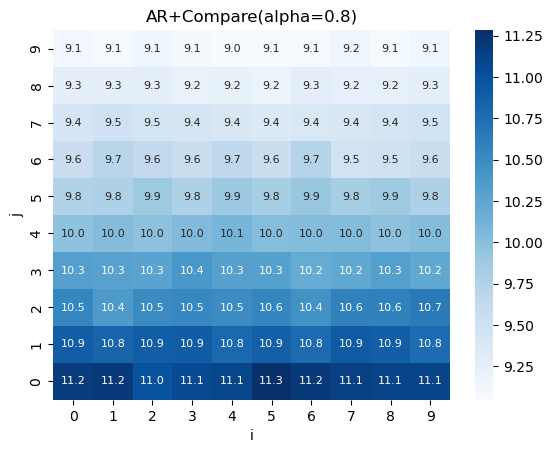}
    \end{center}
  \caption{Heatmap of SAR+Compare Sampling Behavior }
  \label{fig:nsamples_sarcompare}
\end{figure}

\begin{figure}[H]
    \begin{center}
      \includegraphics[width=0.5\linewidth]{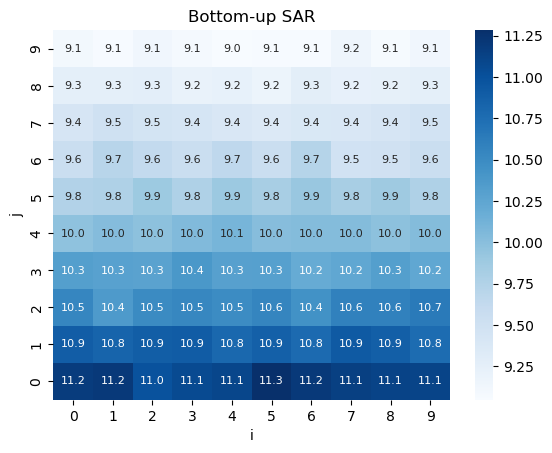}
    \end{center}
  \caption{Heatmap of Bottom-up SAR Sampling Behavior}
  \label{fig:nsamples_bottomupsar}
\end{figure}

\paragraph{Experiment 4: Empirical Validation of Theoretical Bounds.} Lastly, we empirically examine whether the dependence on $H_2(\varepsilon)$ in our upper-bound analysis, defined in Eq.~\ref{eq:H-defs}, is reflected in practice. Figure~\ref{fig:h2_theoretical} plots the log probability of misidentification against the normalized theoretical scaling term
$\frac{-T + KL}{\log(KL)H_2(\varepsilon)}$, ignoring constant factors. The approximately linear trend observed in the plot is consistent with our theoretical analysis and supports the interpretation of $H_2(\varepsilon)$ as a key problem-dependent complexity term governing the difficulty of max-min tree identification.

\begin{figure}[H]
    \begin{center}
      \includegraphics[width=0.7\linewidth]{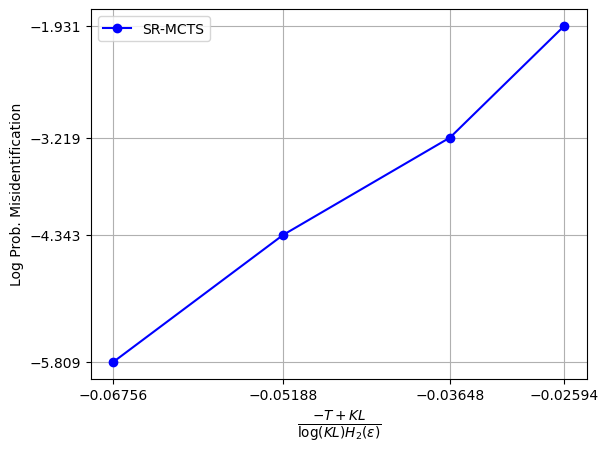}
    \end{center}
  \caption{Log probability of misidentification for SR-MCTS as a function of the theoretical scaling term $\frac{-T+KL}{\log(KL) H_2(\varepsilon)}$}
  \label{fig:h2_theoretical}
\end{figure}

% Figure 
% In the first experiment, we demonstrate 

% \begin{figure}[H]
%   \begin{subfigure}{0.5\linewidth}
%     \begin{center}
%       \includegraphics[width=1.0\linewidth]{figures/probmis_vs_budget.png}
%       % \caption{training}
%     \end{center}
%   \end{subfigure}
%   \begin{subfigure}{0.5\linewidth}
%     \begin{center}
%       \includegraphics[width=1.0\linewidth]{figures/propmis_vs_eps.png}
%       % \caption{}
%     \end{center}
%   \end{subfigure}
%   \caption{.}
%   \label{fig:uniform_mse_vs_num_actions}    
% \end{figure}

%%%%%%%%%%%%%%%%%%%%%%%%%%%%%%%%%%%%%%%%%%%%%%%%%%%%%%%%%%%%

% \appendix

% \section{Technical appendices and supplementary material}
% Technical appendices with additional results, figures, graphs, and proofs may be submitted with the paper submission before the full submission deadline (see above). You can upload a ZIP file for videos or code, but do not upload a separate PDF file for the appendix. There is no page limit for the technical appendices. 

% Note: Think of the appendix as ``optional reading'' for reviewers. The paper must be able to stand alone without the appendix; for example, adding critical experiments that support the main claims to an appendix is inappropriate. 

%%%%%%%%%%%%%%%%%%%%%%%%%%%%%%%%%%%%%%%%%%%%%%%%%%%%%%%%%%%%

% \newpage
% \input{checklist.tex}

\end{document}